\documentclass{article}

\PassOptionsToPackage{numbers, compress}{natbib}
\usepackage[preprint]{neurips_2026}

\usepackage[utf8]{inputenc} % allow utf-8 input
\usepackage[T1]{fontenc}    % use 8-bit T1 fonts
\usepackage{hyperref}       % hyperlinks
\usepackage{url}            % simple URL typesetting
\usepackage{booktabs}       % professional-quality tables
\usepackage{amsfonts}       % blackboard math symbols
\usepackage{nicefrac}       % compact symbols for 1/2, etc.
\usepackage{microtype}      % microtypography
\usepackage{xcolor}         % colors

\usepackage{amsmath}
\usepackage{amssymb}
\usepackage{mathtools}
\usepackage{amsthm}
\usepackage{bm} 
\usepackage{multirow}
\usepackage{subcaption}
\usepackage{wrapfig}
\usepackage{amsfonts}       % 用于支持特殊的数学字体 (例如 \mathbb)
\usepackage{svg}
\usepackage{rotating}
\usepackage{algorithm}
\usepackage{algorithmic}  % 或者 algorithmicx 和 algpseudocode
\usepackage{enumitem}

\usepackage[capitalize,noabbrev]{cleveref}

\theoremstyle{plain}
\newtheorem{theorem}{Theorem}[section]
\newtheorem{proposition}[theorem]{Proposition}

\newtheorem{corollary}[theorem]{Corollary}
\theoremstyle{definition}
\newtheorem{definition}[theorem]{Definition}

\theoremstyle{remark}
\newtheorem{remark}[theorem]{Remark}

\usepackage[dvipsnames, svgnames, table]{xcolor}
\usepackage[textsize=tiny]{todonotes}

\definecolor{deepiris}{HTML}{5D3FD3}

\newcommand{\greybg}[1]{{\cellcolor[HTML]{D0CECE}\textbf{#1}}}

\usepackage{CJKutf8} % 引入 CJKutf8

\title{SPACE: Unifying Symmetric and Asymmetric Routing Problems for Generalist Neural Solver}

\author{%
  Rongsheng Chen$^{1,2,3}$ \quad Changliang Zhou$^{1,3}$ \quad Canhong Yu$^{4}$ \quad Yuanyao Chen$^{1,3}$\\ \textbf{Yu Zhou$^{4}$ \quad Zhuo Chen$^{2}$ \quad Zhenkun Wang$^{1,3}$}\\
  $^{1}$ School of Automation and Intelligent Manufacturing,\\ Southern University of Science and Technology, Shenzhen, China\\
  $^{2}$ Pengcheng Laboratory, Shenzhen, China\\
  $^{3}$ Guangdong Provincial Key Laboratory of Fully Actuated System Control Theory and Technology,\\
  Southern University of Science and Technology, Shenzhen, China\\
  $^{4}$ College of Computer Science and Software Engineering, Shenzhen University, Shenzhen, China\\
  \texttt{\{chenrs2025,zhoucl2022\}@mail.sustech.edu.cn, 2410103054@mails.szu.edu.cn,} \\
  \texttt{12433020@mail.sustech.edu.cn, zhouyu\_1022@126.com, chenzhuo.zoom@gmail.com,} \\
  \texttt{wangzk3@sustech.edu.cn}
}

\begin{document}
\begin{CJK*}{UTF8}{gbsn} 

\maketitle

\begin{abstract}
  %   The abstract paragraph should be indented \nicefrac{1}{2}~inch (3~picas) on both the left- and right-hand margins. Use 10~point type, with a vertical spacing (leading) of 11~points. The word \textbf{Abstract} must be centered, bold, and in point size 12. Two line spaces precede the abstract. The abstract must be limited to one paragraph.
Generalist neural routing solvers have shown great potential in solving diverse vehicle routing problems (VRPs) with a unified model. However, existing solvers are typically limited to symmetric settings or degrade in performance when switching to asymmetric settings due to input inconsistencies or inherent structural differences, substantially limiting their practicality in real-world scenarios that encompass both scenarios. To address this limitation, we define the spatial position of each node based on the relative distances to a specific set of pivots and further propose a Spatial Pivot-Aligned Coordinate-free Embedding (SPACE) framework that unifies node representation and solution generation across symmetric and asymmetric VRPs. Specifically, we construct a bidirectional Fréchet representation using a novel furthest pivot sampling strategy to enable invariant node representations across distinct problem settings. Furthermore, we introduce a weight-decomposed adaptive decoding mechanism that decouples geometric perception from problem representations, mitigating the overfitting of constraint decisions to a specific geometry setting. Extensive experiments on 110 VRP variants, comprising 55 symmetric problems and their asymmetric counterparts, demonstrate that SPACE achieves promising zero-shot generalization in both symmetric and asymmetric VRPs.
\end{abstract}
\section{Introduction}
\label{sec:intro}
The Vehicle Routing Problem (VRP) is an essential class of combinatorial optimization problems (COPs) and has significant practical importance in domains such as logistics and supply chain management~\citep{sar2023systematic}. As an NP-hard problem, solving diverse VRPs efficiently is challenging. While exact solvers guarantee optimality, their prohibitive computational costs render them impractical for real-world scenarios. In contrast, traditional heuristics~\citep{LKH3,HGS} have achieved remarkable performance within reasonable time limits over the past few decades. However, these methods rely heavily on extensive domain expertise to handcraft specialized rules for specific problems. Given the rapidly growing diversity of VRP variants in practical applications, manual adjustment for each case has become increasingly impractical.

Recently, neural combinatorial optimization (NCO) methods have attracted significant attention for their potential to reduce reliance on expert knowledge while maintaining competitive solution quality~\citep{bengio2021machine,ba2026survey}. By automatically learning implicit heuristic rules from data, numerous specialist solvers achieve promising performance on specific problems like the Traveling Salesman Problem (TSP)~\citep{joshi2019efficient,li2024fastT2T} and the Capacitated VRP (CVRP)~\citep{kwon2020pomo,zhou2024icam}. However, they typically require per-problem retraining to accommodate the specific constraints and attributes of distinct VRP variants. More critically, most methods focus on symmetric settings and heavily rely on coordinate-based encoding schemes, rendering them inapplicable to asymmetric settings where node coordinates are unavailable. These limitations impose substantial training overhead and hinder practical deployment on new problems, especially in asymmetric scenarios. 

To address the challenge of cross-problem generalization, research focus has increasingly shifted toward generalist neural routing solvers capable of handling diverse VRP variants. Existing efforts can be broadly categorized into three paradigms: 1) constraint combination; 2) adapter-based fine-tuning; and 3) the emerging unified data representation. The first category treats VRP variants as combinations of predefined constraints and trains a unified model across these combinations to enable knowledge sharing among seen constraints~\citep{liu2024mtpomo,zhou2024mvmoe}. The second approach builds a shared backbone for all problems while incorporating problem-specific input/output adapters, thereby reducing retraining costs~\citep{lin2024cross,drakulic2025goal}. The last paradigm replaces problem enumeration with data unification, significantly broadening the problem coverage to over 100 VRP variants~\citep{zhou2025urs}. A detailed review of related work on specialist and generalist neural solvers is provided in Appendix \ref{append:related_work}.

\begin{wrapfigure}[18]{r}{0.65\textwidth}
\vspace{-18pt}
    \centering{\includegraphics[width=0.95\linewidth]{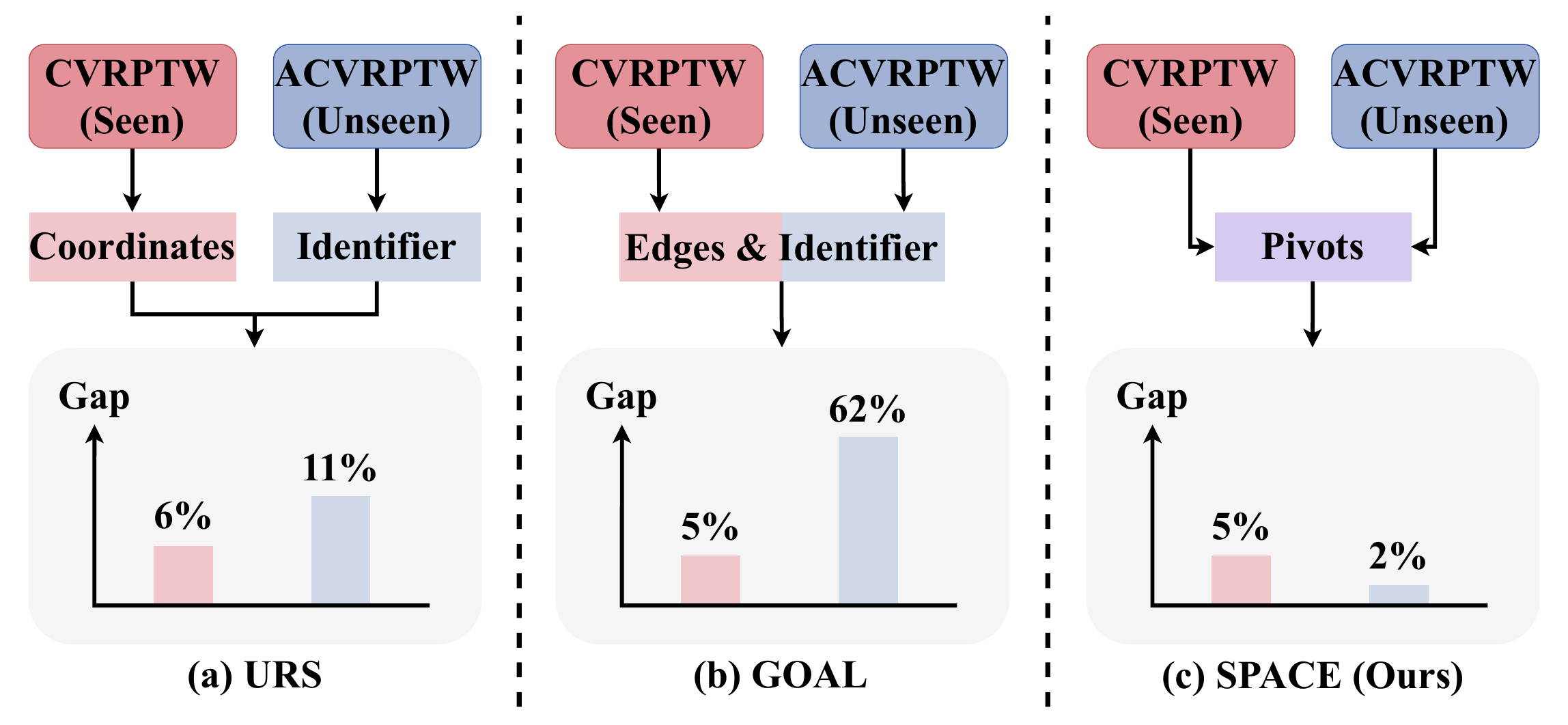}}
    \caption{\textbf{Effects of node representations on generalization gaps across different generalist solvers.} Both (a) URS~\citep{zhou2025urs} and (b) GOAL~\citep{drakulic2025goal} suffer from performance degradation caused by input inconsistencies or inherent structural differences between symmetric and asymmetric settings. (c) Our SPACE employs an invariant node representation to unify two problem settings, ensuring robust zero-shot generalization.\protect\footnotemark}
    \label{fig:motivation}
\end{wrapfigure}
\footnotetext{Given that GOAL already captures both symmetric and asymmetric features during training, we adopt an adapter sharing strategy to minimize retraining costs.}

Despite these advancements, existing methods still suffer from limited applicability, as most generalist solvers remain limited in symmetric VRPs. However, real-world scenarios encompass both symmetric and asymmetric settings. For instance, while drone route planning typically operates in symmetric settings~\citep{attenni2023drone}, urban express delivery often exhibits significant asymmetry~\citep{son2025rrnco}. This duality necessitates a unified solver that seamlessly handles both settings to effectively address complex real-world challenges. As illustrated in Figure \ref{fig:motivation}, although URS~\citep{zhou2025urs} and GOAL~\citep{drakulic2025goal} have attempted to address this challenge, they exhibit notable drawbacks. URS relies on distinct positional representations for these two settings, hindering the sharing of learned structural knowledge and compromising its zero-shot generalization on asymmetric variants. Conversely, while the adapter-based GOAL employs a unified distance-only encoding, the inherent structural difference results in performance degradation when switching between distinct settings. In addition, GOAL lacks zero-shot generalization capabilities for unseen problem variants. 

In this paper, grounded in Bourgain's Theorem~\cite{bourgain1985lipschitz}, we construct an invariant node representation for unifying symmetric and asymmetric VRPs. By defining the spatial position of each node based on its distances to a specific set of \textbf{pivots}, we propose a \underline{\textbf{S}}patial \underline{\textbf{P}}ivot-\underline{\textbf{A}}ligned \underline{\textbf{C}}oordinate-free \underline{\textbf{E}}mbedding (SPACE) framework, which significantly improves cross-problem zero-shot generalization for NCO methods, especially in asymmetric VRPs. Our contributions can be summarized as follows:

\begin{itemize}
    \item We construct a bidirectional Fréchet representation using a novel furthest pivot sampling strategy. It enables a unified node representation across symmetric and asymmetric VRPs, thereby delivering superior results across both symmetric and asymmetric variants. 
    \item We introduce a weight-decomposed adaptive decoding mechanism that decouples geometric perception from problem representations, mitigating the overfitting of constraint decisions to a specific geometry setting.
   \item Extensive experiments on \textbf{110} VRP variants, comprising 55 symmetric problems and their asymmetric counterparts, demonstrate that SPACE achieves promising zero-shot generalization in both symmetric and asymmetric VRPs.
\end{itemize}

\section{Background and Motivation}
\label{sec:preliminary}
In this section, we first formulate VRPs and then review the common pipeline of autoregressive constructive solvers. Finally, based on our preliminary experiment and theoretical analysis, we identify a potential limitation in current encoding and introduce a targeted technique to address it.

\subsection{Vehicle Routing Problems}
Mathematically, the VRP is formulated as an optimization task within a finite action space. Formally, an instance is a tuple $\mathcal{G} = (\mathcal{V}, \bm{D})$, where $\mathcal{V}=\{v_i\}_{i=0}^n$ denotes the set of $n+1$ nodes unless the variant has no depot (e.g., TSP). In VRP, each node $v_i \in \mathcal{V}$ indicates node coordinates $\{x_{i},y_{i}\}$ when available and problem-specific attributes (e.g., demands in CVRP). The edge set of the problem is characterized by $\bm{D}=\{d(v_i, v_j)| \ \forall \  i,j \in 0, 1, \dots, n\}$. To accommodate the realistic asymmetric setting, i.e., $d(v_i, v_j) \neq d(v_j, v_i)$, $\bm{D}$ is characterized as a \textbf{quasimetric space}\footnote{A quasimetric relaxes the symmetry constraint of a standard metric, allowing for asymmetric costs where $d(v_i, v_j) \neq d(v_j, v_i)$. Please refer to Appendix \ref{app:metric_defs} for more details.} \cite{wilson1931quasi}, which satisfies non-negativity, identity of indiscernibles, and the triangle inequality $d(v_i, v_j) \le d(v_i, v_k) + d(v_k, v_j)$.

A feasible solution is a set of routes that satisfies all operational constraints, where each route originates and terminates at the depot (or starting node)~\citep{toth2002vehicle}. We represent the solution as a sequence $\bm{\pi} = (\pi_1, \pi_2, \ldots, \pi_T)$ (i.e., a finite sequence), which is a permutation of $T$ nodes. The objective is to find a permutation $\pi^* \in \bm{\pi}_{\text{feasible}}$ that minimizes the cost function $\mathcal{J}$, i.e., $\pi^* = \arg\min_{\pi \in \bm{\pi}_{\text{feasible}}} \mathcal{J}(\pi | \mathcal{G})$. The function $\mathcal{J}$ formulates diverse routing objectives. In most VRP variants, it can be represented as the cumulative distance $\sum_{i=1}^{T-1} d(\pi_i, \pi_{i+1}) + d(\pi_T, \pi_{1})$.

In this paper, we target a comprehensive coverage of symmetric and asymmetric settings. Consistent with \citet{zhou2025urs}, each variant incorporates at least one of nine constraints to reflect diverse real-world scenarios: (1) Capacity (C); (2) Open Route (O); (3) Backhaul (B); (4) Backhaul and Priority (BP); (5) Duration Limit (L); (6) Time Windows (TW); (7) Multi-Depot (MD); (8) Prize Collecting (PC); (9) Pickup and Delivery (PD). Detailed definitions are provided in Appendix \ref{append:setup_constraints}.

\subsection{Autoregressive Constructive Neural Solver}
Autoregressive constructive solvers have become the dominant paradigm for solving VRPs due to their significant reduction in reliance on hand-crafted heuristics. Most constructive solvers adopt an encoder-decoder architecture parametrized by $\bm{\theta}=\{\bm{\theta}_{enc},\bm{\theta}_{dec}\}$~\citep{luo2023lehd,kwon2020pomo}. Without loss of generality, we delineate the pipeline using the representative AM~\citep{kool2019attention}. Given an instance $\mathcal{G}$, the encoder $\bm{\theta}_{enc}$ transforms raw node features into advanced node embeddings $H^{(L)}=\{\mathbf{h}_i^{(L)}\}_{i=0}^n$ via linear projection and subsequent $L$ attention layers. At decoding step $t$, the decoder $\bm{\theta}_{dec}$ sequentially appends a node to the partial solution $\pi_{1:t-1}=(\pi_{1},\pi_{2},\dots,\pi_{t-1})$. This process continues until a complete solution is constructed. To ensure feasibility, a masking function $\mathcal{M}_t$ prevents invalid selections by assigning $-\infty$ logits to nodes that are either visited or violate problem-specific constraints (e.g., capacity).

\subsection{Motivation and Key Idea}
\label{sec:motivation}
\paragraph{Rethinking Symmetric and Asymmetric Encoding}

The nature of the distance matrix $\bm{D}$ yields a disparity within the encoding process. In symmetric instances, node coordinates provide explicit inductive biases that can fully reflect the problem setting, with $\bm{D}$ serving merely as auxiliary biases to facilitate decision-making in some works~\citep{zhou2024icam,huang2025reld}. Conversely, for asymmetric problems, the lack of inherent node coordinates means their geometric features are entirely determined by $\bm{D}$~\citep{kwon2021matnet,drakulic2023bq}. This distinction creates a critical representation gap, hindering the generalization of neural solvers that rely on symmetric spatial priors to asymmetric scenarios.

Intuitively, since the distance matrix $\bm{D}$ is shared across both symmetric and asymmetric problems, a straightforward unified strategy is to rely solely on distance-based encoding. To validate the limitations of this strategy, we conduct an experiment using URS~\citep{zhou2025urs}, which is a representative neural solver capable of zero-shot generalization across a wide range of VRP variants. Specifically, we train a variant of URS that removes the explicit node coordinates used to distinguish between symmetric and asymmetric inputs, denoted URS-Dist. This compels the model to achieve representation unification via random identifiers and intrinsic distance biases $\bm{D}$, while keeping all other configurations unchanged to ensure a fair comparison. 

\begin{wraptable}[20]{r}{0.45\textwidth}
\vspace{-18pt}
\begin{center}
  \centering
  \caption{Gap comparison between different representation strategies.}
  \resizebox{\linewidth}{!}{
    \begin{tabular}{ll|c|c}
    \toprule[0.5mm]
    \multicolumn{2}{c|}{Problem} & URS   & URS-Dist \\
    \midrule
    \multirow{3}[2]{*}{CVRPTW} & Sym.  & \greybg{6.13\%} & 6.28\% \\
          & Asym. & 11.28\% & \greybg{9.34\%} \\
          & Avg.  & 8.71\% & \greybg{7.81\%} \\
    \midrule
    \multirow{3}[2]{*}{OCVRPTW} & Sym.  & \greybg{5.07\%} & 5.25\% \\
          & Asym. & 17.47\% & \greybg{14.24\%} \\
          & Avg.  & 11.27\% & \greybg{9.74\%} \\
    \midrule
    \multirow{3}[2]{*}{OCVRPBP} & Sym.  & \greybg{14.30\%} & 16.57\% \\
          & Asym. & 28.27\% & \greybg{22.23\%} \\
          & Avg.  & 21.29\% & \greybg{19.40\%} \\
    \midrule
    \multirow{3}[2]{*}{PDTSP} & Sym.  & \greybg{4.98\%} & 6.09\% \\
          & Asym. & 6.21\% & \greybg{-2.31\%} \\
          & Avg.  & 5.60\% & \greybg{1.89\%} \\
    \midrule
    \multirow{3}[2]{*}{PDCVRP} & Sym.  & \greybg{-1.47\%} & 1.43\% \\
          & Asym. & 7.03\% & \greybg{3.55\%} \\
          & Avg.  & 2.78\% & \greybg{2.49\%} \\
    \bottomrule[0.5mm]
    \end{tabular}%
    }
  \label{tab:urs_dist_summary}%
\end{center}
\end{wraptable}

As illustrated in Table \ref{tab:urs_dist_summary}, URS-Dist outperforms URS overall and in asymmetric VRPs, which validates the necessity of a unified representation. However, it exhibits performance degradation in symmetric variants. We attribute this to the removal of explicit coordinates, which deprives the attention~\citep{vaswani2017attention} of essential spatial symmetry priors. Since $\bm{D}$ is used solely as a scalar attention bias, the model is forced to learn two distinct geometric spaces using the same attention without any input priors, making it challenging for the model to perform well in both symmetric and asymmetric settings. These empirical findings highlight the critical need for an advanced coordinate-free encoding scheme that provides relative spatial positions and distinguishable symmetric/asymmetric priors, thereby achieving robust zero-shot generalization across symmetric and asymmetric VRPs.

\paragraph{The Necessity of Bi-Lipschitz Embedding}
To unify symmetric and asymmetric instances, neural solvers require a mapping $\Phi: \mathcal{V} \to \mathbb{R}^k$ that transforms discrete nodes into continuous feature vectors without discarding the geometric priors encoded by the distance matrix $\bm{D}$. This requirement naturally motivates a Lipschitz-type stability criterion: nearby nodes in the underlying routing geometry should remain nearby in the embedding space, while excessive contraction should be avoided whenever the selected reference points sufficiently separate the nodes\cite{bourgain1985lipschitz}. However, because distances induced by normed vector spaces are symmetric, whereas asymmetric costs satisfy $d(v_i,v_j)\neq d(v_j,v_i)$, a direct bi-Lipschitz condition with respect to the raw quasimetric $d$ is mathematically inappropriate. We therefore formulate the stability target with respect to a symmetrized reference metric, and retain asymmetric information through directional reference-distance coordinates.

Formally, for a symmetric reference metric $D$, an embedding $\Phi$ is termed $L$-Lipschitz if distance expansion in the embedding space is bounded by a factor of $L$. If it is also non-contractive up to a constant, then $\Phi$ is a \textit{bi-Lipschitz embedding}: there exist constants $\alpha, \beta > 0$ such that for all node pairs $v_i, v_j \in \mathcal{V}$,
\begin{equation}
\label{eq:bilipschitz}
    \alpha \cdot D(v_i, v_j) \le \| \Phi(v_i) - \Phi(v_j) \|_p \le \beta \cdot D(v_i, v_j).
\end{equation}
The embedding is quantified by its distortion, defined as $\mathcal{D}(\Phi) = \beta / \alpha$ \cite{matousek2013lectures}. If an isometric embedding achieves $\mathcal{D}(\Phi) = 1$, it represents perfect geometric preservation.

Low distortion is desirable because existing neural routing solvers~\citep{kool2019attention,kwon2020pomo} leverage attention mechanisms~\citep{vaswani2017attention} whose vector similarities are expected to correlate with routing costs. Excessive distortion $\mathcal{D}(\Phi)$ causes the vector-space geometry to diverge from the true costs, thereby making cost-dependent feasibility harder to infer from embeddings alone\footnote{We provide a formal derivation of how distortion impacts geometric constraints in Proposition \ref{prop:constraint_preservation} (Appendix \ref{app:geometric_logic}).}. For the BFR used in SPACE, we only claim a provable non-expansive stability bound for arbitrary pivot sets; a lower-bound or bounded-distortion statement requires an additional pivot-separation condition, stated explicitly in Proposition \ref{prop:pivot_separation_lower_bound}. We further show in Proposition \ref{prop:coverage_induced_separation} that a smaller pivot covering radius improves this separation behavior.

\paragraph{The Feasibility of Coordinate-free Encoding}
Obtaining a low-distortion embedding necessitates establishing whether the underlying geometry can be constructed solely from distance information. To this end, Bourgain’s Embedding Theorem~\cite{bourgain1985lipschitz} provides a fundamental guarantee: any finite metric space can be embedded into a Euclidean space with bounded distortion by defining coordinates based on distances to specific subsets of nodes (pivots). The theorem ensures that there exists an embedding where distances are contracted by at most a logarithmic factor while maintaining the upper bound of the original metric (see Theorem \ref{thm:bourgain_main} in Appendix \ref{app:bourgain_proof} for the formal statement and bounds).

The embedding theorem relies on defining coordinates in terms of distances to a specific subset of $\mathcal{V}$. Formally, each coordinate $\phi_\tau(v)$ can be defined as the distance from node $v$ to a subset $A_\tau \subset \mathcal{V}$, denoted as $d(v, A_\tau) = \min_{a \in A_\tau} d(v, a)$. Bourgain demonstrates that by aggregating a series of random subsets $\{A_\tau\}$ of varying cardinalities, the resulting embedding $\Phi(v)$ theoretically satisfies stated distortion bounds.

This theoretical result shows that explicit coordinates are not strictly necessary to capture problem settings. While the original theorem addresses symmetric metrics, it establishes the feasibility of a coordinate-free encoding paradigm. This motivates our design to construct data representations directly from the distance matrix $\bm{D}$, and extends it to handle asymmetric VRPs.

\begin{figure*}[t]
    \centering
    \includegraphics[width=1\linewidth]{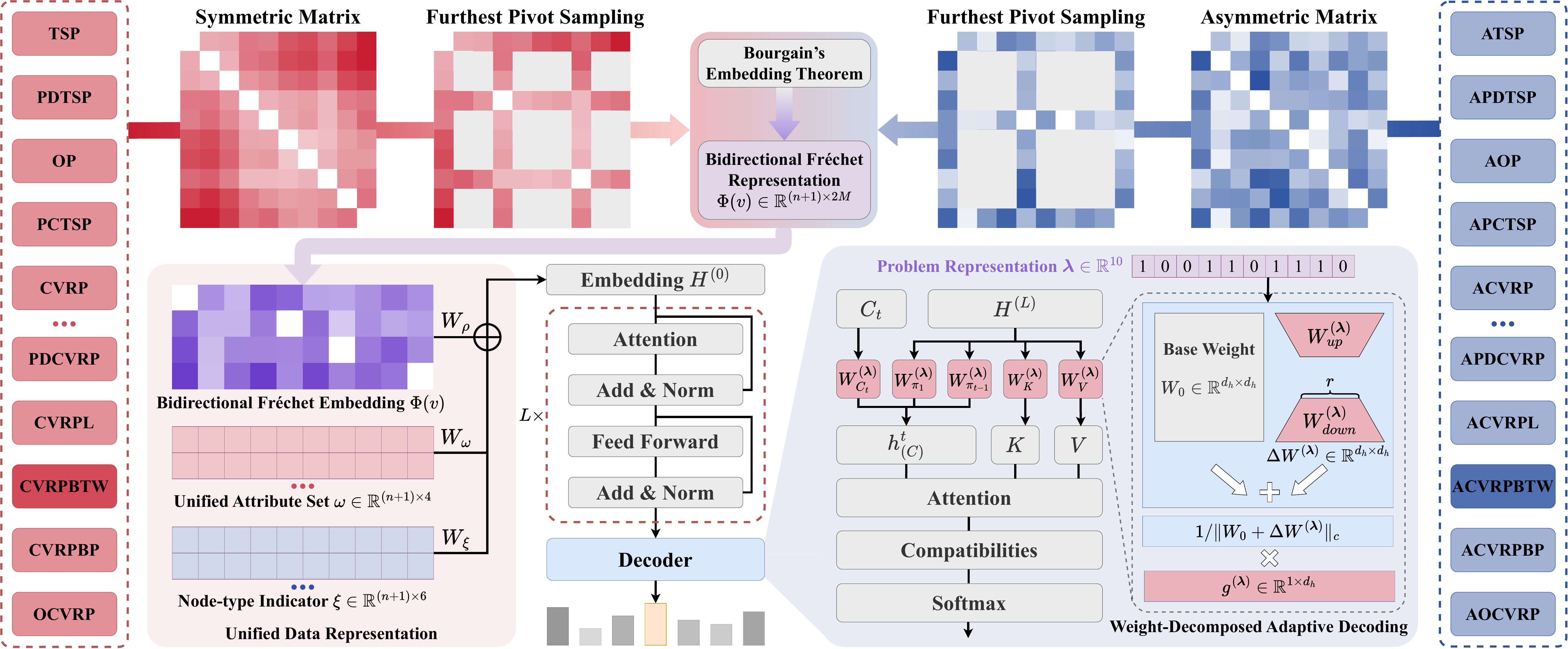}
    \caption{\textbf{Overview of our SPACE framework, illustrated on the (A)CVRPBTW instances.} By leveraging a unified bidirectional Fréchet representation coupled with a weight-decomposed adaptive decoding mechanism, SPACE achieves robust cross-problem zero-shot generalization across both symmetric and asymmetric VRPs, particularly remarkable in asymmetric scenarios.}
    \label{fig:space_overview}
% \vspace{-10pt}
\end{figure*}

\section{Methodology} 
\label{sec:method}

Building on the feasibility analysis in Section~\ref{sec:motivation}, we propose a novel SPACE framework for generalist neural routing solvers, as shown in \cref{fig:space_overview}. Since URS~\citep{zhou2025urs} exhibits zero-shot generalization capability on both symmetric and asymmetric VRP variants, we adopt it as our basic model. Our SPACE includes two key components. Their implementations are detailed below.

\subsection{Bidirectional Fréchet Representation}

Building upon the embedding concept established by Bourgain's Theorem, we introduce a computationally efficient construction to map the VRP instance into a unified vector space. While the theorem guarantees structural preservation for metric spaces, we adapt this distance-to-reference-set perspective to finite quasimetric space $(\mathcal{V}, d)$ by explicitly encoding directionality. We adopt the framework of \textit{Fréchet embeddings}, which represent the spatial position of a node via its distances to a specific set of reference nodes, termed \textit{pivots}.

Let $\mathcal{P}_M = \{p_m\}_{m=1}^M \subset \mathcal{V}$ denote a selected subset of $M$ pivots. In symmetric Euclidean spaces, a standard embedding $\phi(v) = [d(v, p_m)]_{m=1}^M$ can encode relative positions when the pivots separate the nodes. However, realistic routing problems frequently exhibit significant asymmetry ($d(v_i, v_j) \neq d(v_j, v_i)$). A unidirectional embedding can miss this non-commutative topology. Thus, we propose the \textit{Bidirectional Fréchet representation}, which explicitly encodes both the outgoing and incoming distances relative to the pivot set. Formally, for any node $v \in \mathcal{V}$, the embedding $\Phi: \mathcal{V} \to \mathbb{R}^{2M}$ is defined as the concatenation of the outgoing and incoming distance vectors:

\begin{equation}
\label{eq:directed_frechet}
    \Phi(v) = \frac{1}{\sqrt{2M}} \bigoplus_{m=1}^M \left[ \ d(v, p_m), \ d(p_m, v) \ \right],
\end{equation}
where $\oplus$ denotes vector concatenation and the factor $\frac{1}{\sqrt{2M}}$ normalizes the energy of the embedding. This construction inherently satisfies the Lipschitz condition required for geometric stability. Since the Euclidean norm is symmetric, strictly bounding the embedding distance by an asymmetric quasimetric $d(v_i, v_j)$ is mathematically ill-posed. Instead, we establish stability with respect to the \textit{symmetrized metric} $D_{\mathrm{sym}}(v_i, v_j) = \max(d(v_i, v_j), d(v_j, v_i))$. For any pivot $p_m$, the quasimetric triangle inequality in the two directions gives $|d(v_i,p_m)-d(v_j,p_m)|\le D_{\mathrm{sym}}(v_i,v_j)$ and $|d(p_m,v_i)-d(p_m,v_j)|\le D_{\mathrm{sym}}(v_i,v_j)$. Consequently, the Euclidean distance in the embedding space is upper-bounded as:
\begin{equation}
\label{eq:lipschitz_bound}
    \| \Phi(v_i) - \Phi(v_j) \|_2 \le \sqrt{\sum_{m=1}^{2M} \frac{1}{2M} D_{\mathrm{sym}}(v_i, v_j)^2} = D_{\mathrm{sym}}(v_i, v_j).
\end{equation}
This ensures the mapping is \textit{1-Lipschitz} regarding the symmetrized topology, guaranteeing that nodes close in the original graph (in both directions) remain close in the embedding space. Although the embedding distance $\|\cdot\|_2$ is symmetric, the concatenation of incoming and outgoing features in Eq. \ref{eq:directed_frechet} provides directional information from which the neural decoder can learn asymmetric costs. We provide the rigorous proof, the conditional non-contraction statement, the coverage-induced separation bound, and the necessity of this bidirectional design in Proposition \ref{prop:lipschitz_continuity}, Proposition \ref{prop:pivot_separation_lower_bound}, Proposition \ref{prop:coverage_induced_separation}, and Corollary \ref{cor:bidirectionality} (Appendix \ref{app:directed_frechet_proof}).

\paragraph{Geometry-Preserving Pivot Selection}
\label{subsec:landmark_selection}

The representational capacity of the Bidirectional Fréchet Representation $\Phi$ depends critically on the spatial distribution of the selected pivot set $\mathcal{P}$. Randomization frequently clusters pivots, generating highly correlated distance features that collapse the representation's effective dimensionality and obscure fine-grained geometric details. Maximizing the information content of the representation requires maximizing the spatial diversity of $\mathcal{P}$.

We therefore employ \textit{Furthest Pivot Sampling} (FPS) to construct the pivot set. Initialized with the depot node $v_0$, the algorithm iteratively selects the node $p_{m}$ that maximizes the minimum distance to the set of previously chosen pivots $\mathcal{P}_{m-1}$:
\begin{equation}
\label{eq:fps_greedy}
    p_m = \arg\max_{v \in \mathcal{V}} \min_{p \in \mathcal{P}_{m-1}} d_{\mathrm{fps}}(v, p).
\end{equation}
Here, we utilize the metric $d_{\mathrm{fps}}(v_i, v_j) = \frac{1}{2}(d(v_i, v_j) + d(v_j, v_i))$ to ensure balanced coverage across the underlying topology regardless of edge directionality. This iterative maximization should be interpreted as the standard furthest-first traversal for the metric $d_{\mathrm{fps}}$, rather than as an exact solution to the $M$-center objective. Under $d_{\mathrm{fps}}$, the greedy pivot set achieves the usual 2-approximation to the optimal covering radius, with the guarantee stated in Proposition~\ref{prop:fps_metric_coverage}. The link between pivot coverage and BFR separation is formalized in Proposition~\ref{prop:coverage_induced_separation}. Geometrically, the procedure pushes pivots toward under-covered regions of the data topology, reducing redundant distance features and helping distinct topological regions map to distinct regions in the vector space. Consequently, the representation provides the neural model with a high-resolution triangulation of the underlying geometry. More details are provided in Appendix~\ref{app:fps_details}.

\subsection{Weight-Decomposed Adaptive Decoding}
\label{subsec:dora_adaptation}

While BFR unifies the input representation, neural solvers must still learn constraint-handling logic that generalizes across diverse problem settings. A recent study~\citep{zhou2025urs} constructs a multi-hot problem representation $\bm{\lambda}$, based on input features, and employs a problem-conditioned parameter generation mechanism. While this approach mitigates the learning burden, it suffers from a key limitation. Within $\bm{\lambda}$, the spatial structures (i.e., symmetric/asymmetric) and various attributes (e.g., demand) are deeply coupled. When the same attribute appears across two spatial structures or multiple attributes are recombined into unseen variants, the previously hypernetwork-based technique risks overfitting constraint decisions to a specific metric setting, thereby delivering poor results on another setting (e.g., from symmetric to asymmetric scenarios).

Inspired by Weight-Decomposed Low-Rank Adaptation (DoRA)~\citep{liu2024dora}, we propose weight-decomposed adaptive decoding (WDAD) to enforce a decoupling between geometric perception and problem representations, thereby improving zero-shot generalization across asymmetric and symmetric VRPs. To align with our unified representation framework, we exclude explicit identifiers for symmetric and asymmetric properties in $\bm{\lambda}$(see Appendix~\ref{append:multi_hot} for details). Each active entry of $\bm{\lambda}$ corresponds to a routing attribute, and WDAD composes task-specific decoder parameters from these active attributes instead of allocating separate parameters to every VRP variant. For each adaptive projection matrix in the decoder, we decompose the effective weight into a shared base weight $W_0$ and a representation-conditioned update $\Delta W^{(\bm{\lambda})}$:
\begin{equation}
    W_{\mathrm{eff}}^{(\bm{\lambda})}
    = W_0 + \Delta W^{(\bm{\lambda})}, \qquad
    \Delta W^{(\bm{\lambda})}
    = \frac{1}{Z(\bm{\lambda})}\sum_{j:\lambda_j=1}\delta_j,
    \quad
    Z(\bm{\lambda})=\max(1,\|\bm{\lambda}\|_1).
\end{equation}
where each $\delta_j$ is an attribute-specific adapter induced by the $j$-th active constraint. To preserve the decoupling between update direction and magnitude, each adapter is parameterized by a DoRA-style low-rank decomposition with $H$ independent heads for every attribute:
\begin{equation}
    \delta_j =
    \sum_{\eta=1}^{H}
    g_{\eta}^{(j)} \odot
    \frac{
        W_{up,\eta}^{(j)}W_{down,\eta}^{(j)}
    }{
        \|W_{up,\eta}^{(j)}W_{down,\eta}^{(j)}\|_{\mathrm{row}} + \epsilon
    },
\end{equation}
where $W_{up,\eta}^{(j)} \in \mathbb{R}^{d_{\mathrm{out}} \times r}$ and $W_{down,\eta}^{(j)} \in \mathbb{R}^{r \times d_{\mathrm{in}}}$ determine the update direction, while $g_{\eta}^{(j)} \in \mathbb{R}^{d_{\mathrm{out}}}$ controls the row-wise magnitude of each output channel. Here, $\|\cdot\|_{\mathrm{row}}$ denotes the $\ell_2$ norm along the input dimension of each output row. The normalization by $Z(\bm{\lambda})$ averages nonempty active-attribute updates and falls back to the shared base decoder when no attribute is active, preventing magnitude explosion when multiple constraints coexist and allowing the decoder to compose known constraint primitives into unseen combinations. WDAR allows the shared base weight $W_0$ to capture general routing knowledge learned from BFR, whereas $\Delta W^{(\bm{\lambda})}$ injects only the task-specific decision logic required by the active attribute.  More implementation details are provided in Appendix~\ref{app:dora_details}.

\section{Experiments}
\label{sec:experiments}

In this section, we evaluate SPACE against both classical and learning-based solvers across 110 VRPs and adopt the 100-node setting as our primary evaluation criterion. All experiments are conducted on a single NVIDIA GeForce RTX 4090 GPU (24GB of memory).

\subsection{Experimental Setup}

\paragraph{Problem Setting} 
We adopt the unified data representation established by \citet{zhou2025urs} and select 12 VRPs in training to ensure comprehensive feature coverage, they are: (1) ATSP~\citep{kwon2021matnet};(2) TSP~\citep{kool2019attention}; (3) CVRP~\citep{kool2019attention}; (4) ACVRP~\citep{kwon2021matnet,kool2019attention}; (5) Orienteering Problem (OP)~\citep{kool2019attention}; (6) PCTSP~\citep{kool2019attention}; (7) PDTSP~\citep{li2021pdtsp}; (8) CVRPTW~\citep{zhou2024mvmoe}; (9) OCVRP~\citep{zhou2024mvmoe}; (10) CVRPB~\citep{zhou2024mvmoe}, (11) OCVRPTW~\citep{zhou2024mvmoe}, and (12) ACVRPBTW~\citep{zhou2025urs}. We provide a detailed discussion of problem selection in Appendix \ref{append:problem_selection}. To demonstrate the broad applicability of SPACE, we evaluate the model on 55 symmetric problems and their asymmetric counterparts. Notably, we focus on a model's overall performance across two settings, rather than limiting it to a single scenario.  Each test dataset consists of 1,000 instances. 

\paragraph{Model \& Training Setting}
\label{sec:exp_model_setting}

Regarding the proposed BFR, we set the number of pivots to $M=8$. For the WDAD, we configure the rank to $r=32$ and the number of heads to $3$. To ensure a fair comparison, all remaining hyperparameters and training protocols strictly adhere to \citep{zhou2025urs} configuration. SPACE is trained by the REINFORCE~\citep{williams1992reinforce} algorithm, and the data is generated on the fly. Further details on model architecture and training are provided in Appendix \ref{append:model_training}.

\paragraph{Baseline}

We compare SPACE with the following generalist neural routing solvers: MTPOMO~\citep{liu2024mtpomo}, MVMoE/4E~\citep{zhou2024mvmoe}, ReLD-MTL~\citep{huang2025reld}, RouteFinder~\citep{berto2024routefinder}, CaDA~\citep{li2025cada}\footnote{Due to differing training variants and constraints, we evaluate URS against RouteFinder~\citep{berto2024routefinder} and CaDA~\citep{li2025cada} on all shared seen and unseen VRPs under consistent problem settings (see Appendix \ref{append:compare_rf_cada} for comparable results).}, GOAL-MTL~\citep{drakulic2025goal}, and URS~\citep{zhou2025urs}.  The oracle solvers of all evaluated problems are provided in Appendix \ref{append:multi_hot}.

\paragraph{Metrics \& Inference}
We report the optimality gap (Gap) and inference time (Time) for each method. The optimality gap quantifies the discrepancy between the solutions generated by neural solvers and those obtained by oracle solvers. Note that we report the average gap for symmetric and asymmetric scenarios. For most NCO baselines, we execute the source code provided by the authors using default settings. Consistent with~\citep{zhou2025urs}, we report the best result with $\times8$ and $\times128$ instance augmentation for symmetric and asymmetric instances, respectively. The mark ($-$) indicates that the method does not support the corresponding setting.

\begin{table*}[t]
  \centering
  \caption{Performance comparison on the 12 VRP variants used in SPACE training.}
  \resizebox{\textwidth}{!}{
    \begin{tabular}{l|cc|cc|cc|cc|cc|cc}
    \toprule[0.5mm]
    \multirow{2}[2]{*}{Method} & \multicolumn{2}{c|}{CVRP100} & \multicolumn{2}{c|}{ACVRP100} & \multicolumn{2}{c|}{CVRPTW100} & \multicolumn{2}{c|}{CVRPB100} & \multicolumn{2}{c|}{OCVRP100} & \multicolumn{2}{c}{OCVRPTW100} \\
          & Gap   & Time  & Gap   & Time  & Gap   & Time  & Gap   & Time  & Gap   & Time  & Gap   & Time \\
    \midrule
    Oracle & 0.00\% & 9.1m  & 0.00\% & 6.6m  & 0.00\% & 19.6m & 0.00\% & 20.8m & 0.00\% & 5.3m  & 0.00\% & 20.8m \\
    \midrule
    MVMoE & 1.65\% & 12s   & \multicolumn{2}{c|}{$-$} & 4.90\% & 15s   & 1.28\% & 11s   & 3.14\% & 13s   & 3.85\% & 15s \\
    MTPOMO & 1.85\% & 6s    & \multicolumn{2}{c|}{$-$} & 5.31\% & 8s    & 1.67\% & 5s    & 3.46\% & 6s    & 4.41\% & 7s \\
    ReLD-MTL & \greybg{1.42\%} & 8s    & \multicolumn{2}{c|}{$-$} & \greybg{4.56\%} &  10s  & \greybg{0.90\%} & 6s    & \greybg{2.32\%} & 7s    & \greybg{3.10\%} & 9s \\
    GOAL-MTL & 4.22\% & 48s   & \multicolumn{2}{c|}{$-$} & 4.66\% & 42s   & \multicolumn{2}{c|}{$-$} & \multicolumn{2}{c|}{$-$} & \multicolumn{2}{c}{$-$} \\
    \midrule
    URS   & 1.81\% & 6s    & 3.06\% & 1.4m  & 6.13\% & 8s    & 1.46\% & 6s    & 3.24\% & 6s    & 5.07\% & 8s \\
    SPACE & 1.83\% & 7s    & \greybg{2.67\%} & 1.8m  & 5.36\% & 9s    & 1.45\% & 6s    & 3.31\% & 7s    & 4.20\% & 8s \\
    \midrule
    \midrule
    \multirow{2}[2]{*}{Method} & \multicolumn{2}{c|}{TSP100} & \multicolumn{2}{c|}{ATSP100} & \multicolumn{2}{c|}{OP100} & \multicolumn{2}{c|}{PCTSP100} & \multicolumn{2}{c|}{PDTSP100} & \multicolumn{2}{c}{ACVRPBTW} \\
          & Gap   & Time  & Gap   & Time  & Gap   & Time  & Gap   & Time  & Gap   & Time  & Gap   & Time \\
    \midrule
    Oracle & 0.00\% & 6m    & 0.00\% & 2m    & 0.00\% & 1.5m  & 0.00\% & 1.2h  & 0.00\% & 9.8m  & 0.00\% & 3.5h \\
    \midrule
    GOAL-MTL & \multicolumn{2}{c|}{$-$} & \greybg{1.77\%} & 1m    & 1.20\% & 38s   & \multicolumn{2}{c|}{$-$} & \multicolumn{2}{c|}{$-$} & \multicolumn{2}{c}{$-$} \\
    \midrule
    URS   & \greybg{0.57\%} & 6s    & 2.26\% & 1.1m  & 0.45\% & 4s    & 1.06\% & 5s    & 4.98\% & 4s    & 9.95\% & 2.2m \\
    SPACE & 0.64\% & 5s    & 2.89\% & 1.4m  & \greybg{0.44\%} & 5s    & \greybg{0.88\%} & 5s    & \greybg{4.71\%} & 5s    & \greybg{-0.52\%} & 2.1m \\
    \bottomrule[0.5mm]
    \end{tabular}%
    }
  \label{tab:seen_performance}%
\end{table*}%

\subsection{Performance Evaluation}

\begin{table*}[t]
  \centering
  \caption{Zero-shot generalization on 24 unseen VRPs.  }
  \resizebox{\textwidth}{!}{
    \begin{tabular}{l|ccc|ccc|ccc}
    \toprule[0.5mm]
          & Symmetric & Asymmetric & Average & Symmetric & Asymmetric & Average & Symmetric & Asymmetric & Average \\
    Method & Gap (Time) & Gap (Time) & Gap   & Gap (Time) & Gap (Time) & Gap   & Gap (Time) & Gap (Time) & Gap \\
    \midrule
    \midrule
    Problem & \multicolumn{3}{c|}{CVRPBPLTW} & \multicolumn{3}{c|}{OCVRPLTW} & \multicolumn{3}{c}{CVRPBLTW} \\
    \midrule
    Oracle & 0.00\% (6.6m) & 0.00\% (6.6m) & 0.00\% & 0.00\% (3.5h) & 0.00\% (3.5h) & 0.00\% & 0.00\% (3.5h) & 0.00\% (3.5h) & 0.00\% \\
    MVMoE & 16.66\% (16s) & $-$ & $-$ & 3.90\% (15s) & $-$ & $-$ & 7.33\% (16s) & $-$ & $-$ \\
    MTPOMO & 17.15\% (10s) & $-$ & $-$ & 4.37\% (8s) & $-$ & $-$ & 7.75\% (8s) & $-$ & $-$ \\
    ReLD-MTL & 15.95\% (12s) & $-$ & $-$ & 3.16\% (9s) & $-$ & $-$ & 6.94\% (9s) & $-$ & $-$ \\
    URS   & 18.33\% (10s) & 13.82\% (2.8m) & 16.08\% & 5.12\% (9s) & 16.35\% (2.8m) & 10.74\% & 9.18\% (8s) & 10.18\% (2.4m) & 9.68\% \\
    SPACE & 17.54\% (11s) & 4.16\% (2.9m) & \greybg{10.85\%} & 4.26\% (9s) & 6.99\% (2.6m) & \greybg{5.63\%} & 8.29\% (8s) & -0.37\% (2.3m) & \greybg{3.96\%} \\
    \midrule
    \midrule
    Problem & \multicolumn{3}{c|}{OCVRPBTW} & \multicolumn{3}{c|}{OCVRPBPTW} & \multicolumn{3}{c}{MDOCVRPLTW} \\
    \midrule
    Oracle & 0.00\% (3.5h) & 0.00\% (6.6m) & 0.00\% & 0.00\% (6.6m) & 0.00\% (6.6m) & 0.00\% & 0.00\% (6.6m) & 0.00\% (6.6m) & 0.00\% \\
    MVMoE & 9.95\% (15s) & $-$ & $-$ & 9.42\% (15s) & $-$ & $-$ & 35.82\% (2.1m) & $-$ & $-$ \\
    MTPOMO & 10.45\% (6s) & $-$ & $-$ & 9.94\% (9s) & $-$ & $-$ & 30.08\% (1.5m) & $-$ & $-$ \\
    ReLD-MTL & 9.29\% (8s) & $-$ & $-$ & 8.43\% (11s) & $-$ & $-$ & 24.65\% (1.8m) & $-$ & $-$ \\
    URS   & 13.77\% (7s) & 14.45\% (2.3m) & 14.11\% & 10.05\% (8s) & 17.45\% (2.5m) & 13.75\% & 21.88\% (1.1m) & 24.30\% (20m) & 23.09\% \\
    SPACE & 11.60\% (7s) & 5.48\% (2.2m) & \greybg{8.54\%} & 9.62\% (9s) & 9.65\% (2.7m) & \greybg{9.64\%} & 18.16\% (1.1m) & 15.35\% (20m) & \greybg{16.75\%} \\
    \midrule
    \midrule
    Problem & \multicolumn{3}{c|}{CVRPLTW} & \multicolumn{3}{c|}{OCVRPB} & \multicolumn{3}{c}{MDCVRPLTW} \\
    \midrule
    Oracle & 0.00\% (3.5h) & 0.00\% (3.5h) & 0.00\% & 0.00\% (3.5h) & 0.00\% (3.5h) & 0.00\% & 0.00\% (6.6m) & 0.00\% (6.6m) & 0.00\% \\
    MVMoE & 1.47\% (16s) & $-$ & $-$ & 7.09\% (12s) & $-$ & $-$ & 41.51\% (2.1m) & $-$ & $-$ \\
    MTPOMO & 1.92\% (9s) & $-$ & $-$ & 7.34\% (5s) & $-$ & $-$ & 34.21\% (1.5m) & $-$ & $-$ \\
    ReLD-MTL & 1.17\% (11s) & $-$ & $-$ & 5.36\% (6s) & $-$ & $-$ & 26.44\% (1.8m) & $-$ & $-$ \\
    URS   & 2.67\% (9s) & 12.22\% (2.7m) & 7.45\% & 9.35\% (6s) & 18.05\% (2.5m) & 13.70\% & 32.92\% (1.3m) & 15.26\% (22m) & 24.09\% \\
    SPACE & 1.82\% (9s) & 2.26\% (2.6m) & \greybg{2.04\%} & 7.33\% (6s) & -1.34\% (1.7m) & \greybg{2.99\%} & 27.29\% (1.3m) & 4.64\% (21m) & \greybg{15.97\%} \\
    \midrule
    \midrule
    Problem & \multicolumn{3}{c|}{MDCVRPTW} & \multicolumn{3}{c|}{CVRPBPTW} & \multicolumn{3}{c}{MDOCVRPB} \\
    \midrule
    Oracle & 0.00\% (6.6m) & 0.00\% (6.6m) & 0.00\% & 0.00\% (6.6m) & 0.00\% (6.6m) & 0.00\% & 0.00\% (6.6m) & 0.00\% (6.6m) & 0.00\% \\
    MVMoE & 40.61\% (2.1m) & $-$ & $-$ & 16.86\% (15s) & $-$ & $-$ & 29.77\% (1.5m) & $-$ & $-$ \\
    MTPOMO & 33.05\% (1.5m) & $-$ & $-$ & 17.22\% (9s) & $-$ & $-$ & 24.80\% (1.1m) & $-$ & $-$ \\
    ReLD-MTL & 25.48\% (1.8m) & $-$ & $-$ & 16.09\% (11s) & $-$ & $-$ & 17.66\% (1.2m) & $-$ & $-$ \\
    URS   & 31.51\% (1.3m) & 15.38\% (21m) & 23.45\% & 18.15\% (9s) & 14.52\% (2.6m) & 16.33\% & 13.46\% (47s) & 25.95\% (16m) & 19.70\% \\
    SPACE & 25.97\% (1.2m) & 4.73\% (19m) & \greybg{15.35\%} & 17.57\% (10s) & 4.73\% (2.7m) & \greybg{11.15\%} & 13.07\% (47s) & 11.03\% (12m) & \greybg{12.05\%} \\
    \bottomrule[0.5mm]
    \end{tabular}%
    }
  \label{tab:unseen_performance}%
\end{table*}%

\paragraph{Performance on Seen VRPs}

As shown in Table \ref{tab:seen_performance}, SPACE remains highly competitive across the 12 seen variants among all comparable solvers and achieves smaller optimality gaps than URS on most tasks. Following the URS philosophy, we reduce reliance on problem-specific constraint states and use a model-agnostic masking function that handles diverse constraints during decoding (see Appendix \ref{append:problem_state}). While this design choice may yield inferior results on small-scale seen VRPs compared to neural solvers that explicitly embed these states (e.g., ReLD-MTL), SPACE demonstrates a clear advantage on both unseen asymmetric variants (Table \ref{tab:unseen_performance}) and large-scale instances (Table \ref{tab:large_scale_realistic}).

\paragraph{Generalization on Unseen VRPs}

We evaluate the zero-shot generalization of SPACE across 98 unseen VRP variants. As shown in Table \ref{tab:unseen_performance} (see Appendix \ref{sec:detailed_experimental_results} for detailed results), SPACE delivers the lowest average gap on most variants. Crucially, it achieves superior results on asymmetric variants while remaining competitive on symmetric ones, demonstrating its strong zero-shot generalization on both symmetric and asymmetric settings.

\begin{table*}[t]
  \centering
  \caption{Scalability comparison between different solvers on complex VRP variants with 1000 nodes.}
  \resizebox{\textwidth}{!}{
    \begin{tabular}{l|ccc|ccc|ccc}
    \toprule[0.5mm]
          & Symmetric & Asymmetric & Average & Symmetric & Asymmetric & Average & Symmetric & Asymmetric & Average \\
    Method & Gap (Time) & Gap (Time) & Gap   & Gap (Time) & Gap (Time) & Gap   & Gap (Time) & Gap (Time) & Gap \\
    \midrule
    Problem & \multicolumn{3}{c|}{CVRPB1000} & \multicolumn{3}{c|}{OCVRP1000} & \multicolumn{3}{c}{CVRPTW1000} \\
    \midrule
    Oracle & 0.00\% (20m) & 0.00\% (20m) & 0.00\% & 0.00\% (20m) & 0.00\% (20m) & 0.00\% & 0.00\% (20m) & 0.00\% (20m) & 0.00\% \\
    MTPOMO & 25.77\% (1.3m) & $-$ & $-$ & 63.98\% (1.3m) & $-$ & $-$ & 37.57\% (1.7m) & $-$ & $-$ \\
    MVMoE & 170.03\% (1.0m) & $-$ & $-$ & 149.38\% (1.2m) & $-$ & $-$ & 46.55\% (1.7m) & $-$ & $-$ \\
    ReLD-MTL & 5.48\% (1.0m) & $-$ & $-$ & 24.75\% (1.2m) & $-$ & $-$ & 15.04\% (1.6m) & $-$ & $-$ \\
    URS   & -4.54\% (40s) & 4.37\% (37s) & -0.09\% & 14.84\% (45s) & 27.51\% (51s) & 21.18\% & 8.30\% (1.3m) & 4.20\% (1.1m) & 6.25\% \\
    SPACE & -4.16\% (40s) & -19.10\% (33s) & \greybg{-11.63\%} & 13.56\% (44s) & 10.28\% (38s) & \greybg{11.92\%} & 7.34\% (1.3m) & -2.92\% (1.2m) & \greybg{2.21\%} \\
    \bottomrule[0.5mm]
    \end{tabular}%
    }
  \label{tab:large_scale_realistic}%
\end{table*}%

\paragraph{Results on Benchmark Dataset}

\begin{wraptable}[9]{r}{0.55\textwidth}
\vspace{-19pt}
\begin{center}
  \centering
  \caption{Comparison on CVRPLIB Set-X and Set-XXL. }
  \resizebox{\linewidth}{!}{
    \begin{tabular}{l|c|c|c|c|c}
    \toprule[0.5mm]
    Method & 500 $< n \leq$ 1K & 1K $< n \leq$ 5K & 5K $< n \leq$ 7K & ALL   & Solved \# \\
    \midrule
    MTPOMO & 16.80\% & 77.52\% & OOM   & $-$ & 32/34 \\
    MVMoE/4E & 26.41\% & 234.60\% & OOM   & $-$ & 32/34 \\
    MVMoE/4E-L & 19.61\% & 154.68\% & OOM   & $-$ & 32/34 \\
    RF-MVMoE & 18.80\% & 108.79\% & OOM   & $-$ & 32/34 \\
    RF-TE & 12.30\% & 36.23\% & OOM   & $-$ & 32/34 \\
    CaDA  & 334.84\% & OOM   & OOM   & $-$ & 31/34 \\
    ReLD-MTL & 9.44\% & 23.80\% & OOM   & $-$ & 32/34 \\
    URS   & 8.68\% & 14.69\% & 12.09\% & 9.20\% & 34/34 \\
    SPACE & \greybg{7.86\%} & \greybg{12.53\%} & \greybg{11.76\%} & \greybg{8.34\%} & 34/34 \\
    \bottomrule[0.5mm]
    \end{tabular}%
    }
  \label{tab:benchmark_summary}%
\end{center}
\end{wraptable}
We evaluate SPACE on large-scale benchmark instances from CVRPLIB Set-X~\citep{uchoa2017cvrplib_setx} and Set-XXL~\citep{arnold2019cvrplib_xxl}.  Note that all models are trained on 100-scale instances. As summarized in Table \ref{tab:benchmark_summary}, SPACE still maintains its position as the best overall solver among all comparable neural solvers. These results further underscore the practical applicability of SPACE in real-world large-scale scenarios. Complete instance-level results are reported in \cref{tb:exp_benchmark_large_scale,tab:setxxl_detailed}.

\paragraph{Scalability on Large-scale Realistic Problems}
To further assess scalability, we evaluate SPACE on symmetric and asymmetric VRP variants comprising 1,000 nodes. As shown in Table \ref{tab:large_scale_realistic}, SPACE consistently outperforms all evaluated baselines and achieves the best overall performance on all three evaluated variants. These results demonstrate the robust scalability of SPACE in addressing large-scale complex VRPs.

\begin{wraptable}[14]{r}{0.5\textwidth}
\vspace{-18pt}
\begin{center}
  \centering
  \caption{Ablation summary of SPACE.}
  \resizebox{\linewidth}{!}{
    % Table generated by Excel2LaTeX from sheet 'summary'
    \begin{tabular}{ccccc|cc}
    \toprule[0.5mm]
    \multirow{2}[2]{*}{Pivot} & \multirow{2}[2]{*}{FPS} & \multirow{2}[2]{*}{Depot} & \multirow{2}[2]{*}{Bi.} & \multirow{2}[2]{*}{WDAD} & \multicolumn{2}{c}{Avg.gap} \\
          &       &       &       &       & Sym.  & Asym. \\
    \midrule
    $\times$ & $\times$ & $\times$ & $\times$ & 3     & 7.67\% & 2.38\% \\
    5     & $\checkmark$ & $\checkmark$ & $\checkmark$ & 3     & 5.19\% & 2.04\% \\
    10    & $\checkmark$ & $\checkmark$ & $\checkmark$ & 3     & 5.50\% & 2.35\% \\
    20    & $\checkmark$ & $\checkmark$ & $\checkmark$ & 3     & 5.83\% & 1.54\% \\
    30    & $\checkmark$ & $\checkmark$ & $\checkmark$ & 3     & 5.70\% & 2.55\% \\
    8     & $\times$ & $\checkmark$ & $\checkmark$ & 3     & 83.10\% & 95.72\% \\
    8     & $\checkmark$ & $\times$ & $\checkmark$ & 3     & 5.85\% & 2.06\% \\
    8     & $\checkmark$ & $\checkmark$ & $\times$ & 3     & 5.80\% & 2.64\% \\
    8     & $\checkmark$ & $\checkmark$ & $\checkmark$ & $\times$ & 7.06\% & 3.82\% \\
    8     & $\checkmark$ & $\checkmark$ & $\checkmark$ & 1     & 5.58\% & 2.23\% \\
    8     & $\checkmark$ & $\checkmark$ & $\checkmark$ & 3     & \greybg{5.11\%} & \greybg{1.27\%} \\
    \bottomrule[0.5mm]
    \end{tabular}%
    }   
  \label{tab:ablation_summary}%
\end{center}
\end{wraptable}

\section{Ablation Study}

We conduct ablation studies across 16 widely studied CVRP variants and their asymmetric counterparts to validate the contributions of BFR and WDAD. Table \ref{tab:ablation_summary} reports the average gaps on symmetric and asymmetric variants. Replacing the entire BFR with a distance-based unified representation (similar to GOAL~\citep{drakulic2025goal}) degrades the symmetric performance, while removing FPS results in significant degradation in both settings, indicating that stable pivot selection is necessary for using reference-distance features. The depot initialization and bidirectional coordinates (denoted as Bi.) also contribute to the final performance. For WDAD, replacing it with the original URS hypernetwork or reducing it to a single head both weaken performance, suggesting that decomposed multi-head updates are useful for transferring constraint logic across problem settings. Detailed per-variant results and the meaning of each ablation option are provided in \cref{tab:ablation_bfr_details,tab:ablation_pivot_details,tab:ablation_wdad_details} in Appendix~\ref{append:ablation}, respectively.

\section{Conclusion, Limitation, and Future Work}
\label{sec:conclusion}

In this work, we propose SPACE, a generalist neural routing framework for unifying symmetric and asymmetric VRPs. The core of SPACE is a bidirectional Fréchet representation that defines node positions in terms of incoming and outgoing distances to pivots selected via furthest pivot sampling, enabling a coordinate-free representation shared across distinct problem settings. Its generalization capability is further enhanced by weight-decomposed adaptive decoding, which separates geometry-invariant routing knowledge from constraint-specific decision logic. Extensive experiments on 110 VRP variants show that SPACE significantly improves performance in asymmetric settings while maintaining competitive performance on symmetric variants, and further scales to large benchmark instances without fine-tuning.

\paragraph{Limitation and Future Work}

A current limitation of SPACE is that it deliberately reduces reliance on problem-specific constraint states. We retain only essential decoder state signals and use a model-agnostic masking function to handle diverse constraints during decoding, which favors cross-problem transfer but can weaken constraint-aware performance on constraint-heavy variants. Improving the model's constraint awareness while preserving model agnosticism is a key direction for future work.

{\small

\bibliographystyle{unsrtnat} 
\bibliography{references}        % References
}

\clearpage
% APPENDIX
%%%%%%%%%%%%%%%%%%%%%%%%%%%%%%%%%%%%%%%%%%%%%%%%%%%%%%%%%%%%%%%%%%%%%%%%%%%%%%%
%%%%%%%%%%%%%%%%%%%%%%%%%%%%%%%%%%%%%%%%%%%%%%%%%%%%%%%%%%%%%%%%%%%%%%%%%%%%%%%
\appendix

\section{Related Work}
\label{append:related_work}
\subsection{Specialist Neural Solvers for Symmetric VRPs}
\label{sec:related_work_stl_s}
Symmetric VRPs constitute the primary focus of neural routing research. Since the seminal application of Ptr-Nets to the TSP~\citep{vinyals2015pointer,bello2016neural}, a succession of autoregressive constructive solvers has emerged for TSPs and CVRPs and already demonstrated promising results on small instances (e.g., 100-node TSPs)~\citep{nazari2018reinforcement,xin2020stepwise,kim2022symnco}. Notably, the Transformer-style Attention Model (AM)~\citep{kool2019attention} and its multi-trajectory variant, POMO~\citep{kwon2020pomo}, have become the prevailing architectures. However, these models struggle with large-scale generalization, spurring researchers to improve them by varying-scale training~\cite{zhou2024icam,zhou2023omni,huang2025reld}, search space reduction~\citep{fang2024invit,zhou2025L2R,chen2025projection}, incorporating distance information~\citep{gao2023elg,zhou2024icam,huang2025reld}, and alternative heavy decoder architectures~\citep{drakulic2023bq,luo2023lehd,luo2024Boosting}. Other paradigms mainly include problem decomposition~\citep{zheng2024udc,ye2023glop} and non-autoregressive heatmap generation augmented with additional searches (e.g., Monte Carlo tree search (MCTS)~\citep{fu2021attgcn_mcts,qiu2022dimes} and 2-opt~\citep{li2024T2T,sun2023difusco}). Specialist solvers have also emerged for variants like heterogeneous fleets~\cite{li2022hcvrp}, pickup and delivery problems~\citep{li2021pdtsp}, CVRP with Backhauls~\citep{wang2024cvrpb}, and min-max multiple VRPs~\citep{zheng2024dpn}. Among these, constructive solvers are most common but are typically designed for VRPs and typically fail in asymmetric settings.

\subsection{Specialist Neural Solvers for Asymmetric VRPs}
\label{sec:related_work_stl_a}
Asymmetric VRPs (AVRPs) have garnered increasing attention due to their alignment with real-world logistics. Since AVRPs lack explicit coordinate information, research primarily focuses on node representations and on incorporating distance matrices. MatNet~\citep{kwon2021matnet} reformulates ATSPs as a bipartite graph and introduces Mixed-Score Attention to incorporate distance information, establishing itself as the mainstream architecture. However, its reliance on fixed one-hot vectors hinders scalability. \citet{pan2025unico} address this limitation by introducing Pseudo One-hot Embeddings. ICAM~\citep{zhou2024icam} introduces $k$-nearest distance features and a lightweight distance-biased attention mechanism to enhance asymmetric spatial perception. Subsequently, \citet{son2025rrnco} further validate effectiveness on real-world scenarios by integrating neural adaptive bias and distance-based probabilistic sampling. Alternatively, the random identifier has been used for each node as its initial node feature in \citep{drakulic2023bq,huang2025reld}, while GLOP~\citep{ye2023glop} and UDC~\citep{zheng2024udc} leverage divide-and-conquer strategies to improve large-scale generalization. Despite these advancements, existing AVRP solvers generally employ highly specialized architectures to address the intrinsic asymmetry, which limits their applicability to symmetric problems and the development of a unified framework.

\subsection{Generalist Neural Routing Solvers}
\label{sec:related_work_mtl}
To address the challenge of cross-problem generalization, growing attention has been directed toward multi-task learning capable of handling diverse routing problems. A primary approach formulates VRP variants as combinations of predefined constraints and trains a shared neural model~\citep{liu2024mtpomo,zhou2024mvmoe,zheng2025mtl_kd,li2025cada}. While achieving promising results with up to 48 variants~\citep{berto2024routefinder}, these methods are inherently bounded by manually specified constraint sets. More critically, they struggle to transfer to asymmetric settings. Another mainstream paradigm is adapter-based fine-tuning~\citep{wang2025mtl_mab,lin2024cross}. Although GOAL~\citep{drakulic2025goal} performs well in both seen symmetric and asymmetric VRPs, this paradigm lacks the capability for zero-shot generalization to unseen problems. Most recently, URS~\citep{zhou2025urs} replaces problem enumeration with a unified data representation and successfully covers over 100 (A)VRP variants without any fine-tuning. However, URS still relies on distinct positional representations for symmetric and asymmetric instances. In contrast, this work constructs a neural solver with zero-shot capabilities, enabling unified encoding and maintaining high performance across a wide range of VRPs.

\clearpage

\section{Setups of VRP Variants}
\label{append:setup_constraints}

\subsection{Symmetry and Asymmetry} 

In symmetric VRP, the distance function is defined as a standard metric that satisfies non-negativity, identity, symmetry, and the triangle inequality. Under this formulation, the problem is naturally modeled as an undirected weighted graph embedded in a metric space, where the distance between any two nodes is reversible and directly determined by the underlying geometry. In contrast, asymmetric VRP relaxes the symmetry requirement (i.e, $d(v_i, v_j) \neq d(v_j, v_i)$), resulting in a fundamentally different distance structure. While non-negativity, identity of indiscernibles, and the triangle inequality are retained, the distance function strictly defines a quasimetric space. This change from a metric to a quasimetric formulation transforms the problem into a directed weighted graph with irreversible costs. Compared to the symmetric setting, it fundamentally modifies the structure of the solution space and the associated algorithmic properties.

\paragraph{Symmetry}
To construct symmetric instances, we randomly sample node coordinates from a 2D unit square, denoted as $\mathcal{X}_i \sim \mathcal{U}[0, 1]^2$. We then compute the distance $d(v_i, v_j)$ between nodes $i$ and $j$, which is generally the Euclidean distance. The symmetric instance naturally forms a metric space (i.e., $d(v_i, v_j) = d(v_j, v_i), \forall v_i, v_j \in \mathcal{V}$).

\paragraph{Asymmetry}
To generate asymmetric VRP instances, we adopt the construction method in \citet{kwon2021matnet}, which directly generates distance matrices of problem instances. We initialize an integer-based directed distance matrix $\bm{D}$, where we independently sample $d(v_i, v_j)^{(0)}$ ($i \neq j$) from the discrete interval $[0, 10^6)$ and set the diagonal elements $d(v_i, v_i) = 0$. Due to the initial distance matrix typically violating the triangle inequality, we iteratively apply the min-plus closure rule: $d(v_i, v_j)^{(m+1)} = \min_{l} \left( d(v_i, v_l)^{(m)} + d(v_l, v_j)^{(m)} \right)$  to enforce the quasimetric property. We compute this transitive closure in a min-plus semiring, terminating the process upon convergence. The resulting matrix $\bm{D}$ strictly satisfies the triangle inequality $d(v_i, v_j) \le d(v_i, v_k) + d(v_k, v_j)$ while generally maintaining asymmetry ($d(v_i, v_j) \neq d(v_j, v_i)$), thus forming an asymmetric metric VRP instance. Finally, we scale the costs as $\tilde{d}(v_i, v_j) = d(v_i, v_j) / \text{scaler}$ to ensure numerical stability for neural network training. Overall, the resulting asymmetric VRP instances are defined over a geometric space that is fundamentally different from the Euclidean metric space underlying symmetric VRP. While both settings satisfy the triangle inequality, the asymmetric distance structure induces a directed geometric space in which spatial relationships are inherently orientation-dependent, leading to a distinct notion of proximity and path optimality compared to the symmetric case.

\subsection{Diverse Constraints}
\paragraph{Capacity (C)}
For all PD-related variants except for PDTSP, we follow the configuration established by \citet{zhou2025urs} and set $C=20$. For other VRP variants, we adopt the conventional settings in \citet{kool2019attention,kwon2020pomo}, where $C=50$. Node demands are sampled from a discrete uniform distribution over the set $\{1, \dots, 9\}$.
\paragraph{Open Route (O)}
Under the Open Route configuration, vehicles are not mandated to return to the depot; thus, the final return leg is omitted from the total cost calculation. This relaxation propagates to other constraints: (i) for O-TW, the constraint $t_i \le T_{\text{end}}$ is relaxed for the final destination, where $T_{\text{end}}$ denotes the depot end time; (ii) for O-L, the route length limit is evaluated without considering the distance back to the depot.
\paragraph{Backhaul (B)}
Conventional CVRP limits operations to deliveries from a depot to linehaul nodes. The inclusion of backhaul constraints introduces backhaul nodes, requiring goods collection and return transport to the depot. Following the setting in \citet{liu2024mtpomo}, node demands are sampled from a discrete uniform distribution $U\{1, 2, \dots, 9\}$. We negate the demands of a randomly selected $20\%$ subset, thereby designating nodes with positive and negative values as linehaul and backhaul customers, respectively. Notably, no precedence constraints are imposed between these two node categories. Following \citet{zhou2024mvmoe}, we mandate that the first customer in each route be a linehaul node to ensure solution feasibility, unless the remaining unvisited nodes consist entirely of backhauls. 
\paragraph{Backhaul and Priority (BP)}
Building upon the B constraint, the BP variant incorporates additional precedence constraints regarding the service order: within each route, all linehaul nodes must be visited prior to any backhaul nodes. If a backhaul node is selected as the initial customer, the vehicle capacity is initialized to zero to reflect the absence of delivery load.
\paragraph{Duration Limit (L)}
For symmetric instances, node coordinates are sampled from $\mathcal{U}[0, 1]^2$, resulting in a maximum depot-to-customer distance of $\sqrt{2}$. The round-trip distance for a single-node route is thus capped at $2\sqrt{2} \approx 2.82$. To ensure feasibility across all instances, we set the duration limit to $L=3.0$, consistent with the setting in \citet{zhou2024mvmoe}. For asymmetric instances, we adopt the setting from \citet{zhou2025urs} and fix the duration limit at $L=0.6$.
\paragraph{Time Windows (TW)}
For symmetric instances, the depot time window $[e_0,l_0]$ is defined as $[0, 3]$ with service time $s_0=0$. For customer nodes, we adopt the generation scheme from \citet{zhou2024mvmoe} for time windows and set the service time $s_i$ to $0.2$. For asymmetric instances, we follow the \cite{zhou2025urs} configuration by setting $[e_0,l_0] = [0, 1]$, while maintaining all other parameters consistent with the symmetric case.
\paragraph{Multi-Depot (MD)}
Consistent with \citet{zhou2025urs}, we fix the depot count at $3$, sampling coordinates from $\mathcal{U}[0, 1]^2$ identically to customer nodes. Unlike relaxed formulations that permit open routing, we enforce a strict \textit{Return-to-Origin} constraint: every valid route must form a closed loop originating from and terminating at the same depot $v_d$. To unify multi-depot management within a single solution sequence, we model the transition between differing depots (e.g., $v_d \to v_e$) as a zero-cost virtual switch. This formulation allows the agent to sequentially optimize disjoint fleets while strictly excluding inter-depot travel distances from the objective function. For a detailed description of the selection strategy for multiple depots, please refer to Appendix~\ref{app:select_md}.

\paragraph{Prize Collecting (PC)}
Under the PC constraint setting, each customer node is associated with a prize and a penalty. While visiting a node accrues its prize, skipping it incurs a corresponding penalty. The objective is to minimize the total cost, defined as the sum of the solution length and the penalties for all unvisited nodes, subject to a minimum prize-collection requirement. Following the configuration in \citet{kool2019attention}, the prizes are sampled from a continuous uniform distribution $\mathcal{U}(0, 4/n)$. The penalties are drawn from $\mathcal{U}(0, 3k_n/n)$, where $n=100$ and $k_n=4$. These parameters are scaled relative to the problem size to maintain consistent task difficulty.

\paragraph{Pickup and Delivery (PD)}
In the PD setting, nodes are partitioned into pickup and delivery pairs. Two fundamental requirements must be satisfied: (1) Precedence: each delivery node must be visited only after its corresponding pickup node; and (2) Completion: the vehicle is prohibited from returning to the depot until all collected loads have been delivered, ensuring that the vehicle is empty at the end of each route. Following the generation scheme in \citet{li2021heterogeneous}, we designate node $0$ as the depot for an instance of size $n$. We define nodes $1$ to $n/2$ as pickup nodes and nodes $n/2+1$ to $n$ as their corresponding delivery nodes. Specifically, for each pickup node $i$, we assign node $i+n/2$ as its unique delivery partner. For PDCVRP and its variants, we sample the demand for each pickup node from a discrete uniform distribution $\mathcal{U}\{1, 9\}$, while setting the demand of the paired delivery node to its negative equivalent to ensure load balance.

\subsection{The Selection Strategy of Multiple Depots}
\label{app:select_md}

Unlike prior works that treat depot returns implicitly~\citep{berto2024routefinder}, we explicitly parse the autoregressive action sequence into depot-rooted route segments. Let $\mathcal{V}_{\mathrm{depot}}\subset\mathcal{V}$ denote the depot set. A nonempty route segment that starts from depot $v_d\in\mathcal{V}_{\mathrm{depot}}$ must terminate at the same depot, i.e., it has the form $v_d\to v_{i_1}\to\cdots\to v_{i_s}\to v_d$. Consecutive depot actions, such as $v_d\to v_e$, are interpreted only as a fleet-selection switch between two route segments. They are not physical travel arcs and do not modify the original distance function $d$.

\paragraph{Return-to-Origin Masking}
During decoding, we maintain an origin-depot state for each partial trajectory. When the current node is a customer, the masking function allows at most the origin depot among all depot nodes and assigns $-\infty$ logits to the other depots. This prevents invalid segments such as $v_d\to v_i\to v_e$ with $v_e\neq v_d$. For pickup-delivery variants, returning to the origin depot is further masked until all picked-up loads in the current segment have been delivered. When the current node is already a depot and unserved customers remain, the decoder may either start a customer-serving segment from the current depot or select another depot as a control action that switches the active fleet. Once all customers have been served, only the current depot remains available, which terminates the trajectory.

The multi-depot objective is computed from the original distance matrix after the sequence has been generated. For a closed route segment $v_d\to v_{i_1}\to\cdots\to v_{i_s}\to v_d$, its cost is
\begin{equation}
    d(v_d,v_{i_1})+\sum_{r=1}^{s-1}d(v_{i_r},v_{i_{r+1}})+d(v_{i_s},v_d).
\end{equation}
For open-route variants, the final return term is omitted. Consecutive depot-to-depot control actions contribute zero switch cost, but this zero cost is applied only at the sequence-evaluation level; the quasimetric distance matrix itself remains unchanged.

\paragraph{Global Lookahead Decoding}
A depot state is a branching point because the next useful route may start from any depot. Greedily decoding only from the currently active depot can therefore underuse distant fleets. To address this, when a trajectory is at a depot we evaluate every candidate depot $v_e\in\mathcal{V}_{\mathrm{depot}}$ before committing to the next action. For each $v_e$, we query the decoder with the encoded representation of $v_e$ and the outgoing distance vector from $v_e$ to all nodes, while masking all depot nodes so that the comparison is based on feasible customer choices. We then record the highest customer-selection probability under this hypothetical start. If the best candidate depot is the current depot, the decoder directly selects the corresponding customer; otherwise, it first selects the best candidate depot as a fleet-switch action. For trajectories that are not at a depot, the standard autoregressive decoder is used without this lookahead.

\clearpage
\section{Detailed Theoretical Proofs}

This section formalizes the geometric claims used in Section \ref{sec:preliminary}. We model VRP distances as finite quasimetrics, analyze stability under the symmetrized metric, and prove the key guarantees of Bidirectional Fr\'echet Representation: non-expansiveness, conditional non-contraction, coverage-induced separation, and the necessity of bidirectional coordinates.

\subsection{Distance Quasimetrics in VRPs}

\label{app:metric_defs}

\begin{definition}[Quasimetric Space]
\label{def:quasimetric}
The node set of a routing problem is defined as a pair $(\mathcal{V}, d)$, where $\mathcal{V}$ is a finite set and $d: \mathcal{V} \times \mathcal{V} \to \mathbb{R}_{\ge 0}$ satisfies:
\begin{enumerate}
    \item $d(v_i, v_j) \ge 0$ (Non-negativity)
    \item $d(v_i, v_i) = 0$ (Identity of indiscernibles; we assume $d(v_i, v_j) > 0$ if $i \neq j$)
    \item $d(v_i, v_k) \le d(v_i, v_j) + d(v_j, v_k)$ (Triangle Inequality)
\end{enumerate}
Unlike a standard metric, a quasimetric does not require symmetry, i.e., $d(v_i, v_j) \neq d(v_j, v_i)$ is permissible \cite{wilson1931quasi}.
\end{definition}

\begin{definition}[Symmetrized Metric for Stability Analysis]
\label{def:symmetrized_metric}
Given a quasimetric space $(\mathcal{V}, d)$, we define the symmetrized metric $D_{\mathrm{sym}}: \mathcal{V} \times \mathcal{V} \to \mathbb{R}_{\ge 0}$ as:
\begin{equation}
    D_{\mathrm{sym}}(v_i, v_j) = \max(d(v_i, v_j), d(v_j, v_i)).
\end{equation}
This is a valid metric: non-negativity, symmetry, and identity follow directly from the quasimetric assumptions, and the triangle inequality follows since
\begin{align}
    D_{\mathrm{sym}}(x,z)
    &= \max\{d(x,z),d(z,x)\} \nonumber\\
    &\le \max\{d(x,y)+d(y,z),d(z,y)+d(y,x)\} \nonumber\\
    &\le D_{\mathrm{sym}}(x,y)+D_{\mathrm{sym}}(y,z).
\end{align}
This metric represents the "worst-case" separation between two nodes and serves as the upper bound for Lipschitz continuity analysis in Section \ref{sec:method}.
\end{definition}

\begin{remark}[On Symmetrization]
While Bourgain's Theorem is originally stated for metric spaces (symmetric), theoretical guarantees for asymmetric spaces can be derived by considering the symmetrized metric $D_{\mathrm{sym}}(v_i, v_j) = \max(d(v_i, v_j), d(v_j, v_i))$ or $D_{\mathrm{sym}}(v_i, v_j) = d(v_i, v_j) + d(v_j, v_i)$. In our method, we explicitly handle asymmetry by constructing a \textit{directed} embedding that concatenates representations for both $d(v_i, \cdot)$ and $d(\cdot, v_i)$.
\end{remark}

\subsection{Preservation of Geometric Logic via Bi-Lipschitz Embeddings}
\label{app:geometric_logic}

Here we formally show why minimizing distortion with respect to a symmetric reference metric is useful for preserving routing geometry. In the BFR analysis below, this reference metric is $D_{\mathrm{sym}}$ rather than the raw asymmetric quasimetric $d$.

\begin{proposition}[Constraint Preservation]
\label{prop:constraint_preservation}
Let $D$ be a metric on $\mathcal{V}$ and let $\Phi: \mathcal{V} \to \ell_p$ be a bi-Lipschitz embedding with constants $\alpha,\beta>0$ and distortion $\mathcal{D}=\beta/\alpha$. Then relative metric neighborhoods are preserved up to factor $\mathcal{D}$: for any distinct $v_i,v_j,v_k\in\mathcal{V}$ with $D(v_i,v_k)>0$,
\begin{equation}
    \frac{\| \Phi(v_i)-\Phi(v_j)\|_p}
         {\| \Phi(v_i)-\Phi(v_k)\|_p}
    \le
    \mathcal{D}\,
    \frac{D(v_i,v_j)}{D(v_i,v_k)}.
\end{equation}
Moreover, any two-hop metric path obeys
\begin{equation}
    \| \Phi(v_i)-\Phi(v_k)\|_p
    \le
    \beta\bigl(D(v_i,v_j)+D(v_j,v_k)\bigr).
\end{equation}
\end{proposition}

\begin{proof}
By the definition of a bi-Lipschitz embedding,
\begin{equation}
    \alpha D(u,w)
    \le
    \| \Phi(u)-\Phi(w)\|_p
    \le
    \beta D(u,w)
    \quad \text{for all } u,w\in\mathcal{V}.
\end{equation}
Using the upper bound for $(v_i,v_j)$ and the lower bound for $(v_i,v_k)$ gives
\begin{equation}
    \frac{\| \Phi(v_i)-\Phi(v_j)\|_p}
         {\| \Phi(v_i)-\Phi(v_k)\|_p}
    \le
    \frac{\beta D(v_i,v_j)}{\alpha D(v_i,v_k)}
    =
    \mathcal{D}\frac{D(v_i,v_j)}{D(v_i,v_k)}.
\end{equation}
The two-hop bound follows from the upper bi-Lipschitz inequality and the triangle inequality of $D$:
\begin{equation}
    \| \Phi(v_i)-\Phi(v_k)\|_p
    \le
    \beta D(v_i,v_k)
    \le
    \beta\bigl(D(v_i,v_j)+D(v_j,v_k)\bigr).
\end{equation}
\end{proof}

\subsection{Bourgain's Embedding Theorem}
\label{app:bourgain_proof}

We use Bourgain's theorem as a known external result. The theorem justifies the general principle that distances to reference subsets can encode finite metric geometry with controlled distortion.

\begin{theorem}[\cite{bourgain1985lipschitz}]
\label{thm:bourgain_main}
Let $(\mathcal{V}, d)$ be any finite metric space with $|\mathcal{V}| = n$. There exists an embedding $\Phi: \mathcal{V} \to \ell_2^k$ with dimension $k = O(\log^2 n)$ such that the distortion is bounded by $\mathcal{D}(\Phi) \le C \log n$, where $C$ is a universal constant independent of the metric space structure.
\end{theorem}

After rescaling, the theorem ensures that for any pair of nodes $v_i, v_j \in \mathcal{V}$, the embedding $\Phi$ satisfies:
\begin{equation}
\label{eq:bourgain_bound}
    \frac{1}{C \log n} d(v_i, v_j) \le \| \Phi(v_i) - \Phi(v_j) \|_2 \le d(v_i, v_j).
\end{equation}
The upper bound implies that the embedding is 1-Lipschitz, ensuring that distances in the vector space never exaggerate the true metric costs. The lower bound guarantees that the contraction of distances is logarithmically bounded, thereby preserving the distinguishability of distinct nodes and the integrity of local neighborhoods.

One standard construction behind the theorem samples reference subsets at logarithmically many scales. Let $J=\lfloor \log_2 n \rfloor$. For each scale $j\in\{1,\ldots,J\}$ and repetition $t\in\{1,\ldots,T\}$ with $T=\Theta(\log n)$, sample a subset $A_{j,t}\subseteq\mathcal{V}$ by including each point independently with probability $2^{-j}$. For a nonempty subset $A$, define
\begin{equation}
    \phi_A(v)=d(v,A)=\min_{a\in A}d(v,a),
\end{equation}
and use the convention $\phi_{\emptyset}(v)=\operatorname{diam}(\mathcal{V})$ for all $v$. Thus an empty sampled subset contributes only a constant coordinate and does not affect pairwise embedding distances.
We then concatenate the subset-distance coordinates as
\begin{equation}
    f(v_i) = \frac{1}{\sqrt{T \cdot J}} \bigoplus_{j=1}^J \bigoplus_{t=1}^T \phi_{A_{j, t}}(v_i).
\end{equation}
For any fixed subset $A$ under this convention, the map $v\mapsto \phi_A(v)$ is 1-Lipschitz:
\begin{equation}
    |\phi_A(v_i)-\phi_A(v_j)|\le d(v_i,v_j),
\end{equation}
because the claim is trivial when $A=\emptyset$, and for $A\neq\emptyset$ the triangle inequality gives $d(v_i,A)\le d(v_i,v_j)+d(v_j,A)$ and, by symmetry of the metric, the reverse inequality with $v_i$ and $v_j$ exchanged. Hence the normalized concatenation is non-expansive:
\begin{equation}
    \|f(v_i)-f(v_j)\|_2^2
    =
    \frac{1}{TJ}\sum_{j=1}^{J}\sum_{t=1}^{T}
    |\phi_{A_{j,t}}(v_i)-\phi_{A_{j,t}}(v_j)|^2
    \le d(v_i,v_j)^2.
\end{equation}
The non-contraction lower bound in Eq.~\eqref{eq:bourgain_bound} is the nontrivial probabilistic separation part of Bourgain's theorem: with $T=\Theta(\log n)$ repetitions, the sampled subsets separate all pairs at the appropriate scale with high probability, yielding contraction at most $O(\log n)$ after a union bound over all node pairs. The deterministic BFR analyzed below has its own guarantee: unconditional non-expansiveness for arbitrary pivots and conditional non-contraction only under explicit pivot separation.

\subsection{Bidirectional Fréchet Representation Construction}
\label{app:directed_frechet_proof}

While Bourgain's Theorem guarantees the existence of a low-distortion embedding via random subsets, practical implementation requires a deterministic, computationally efficient construction. Here, we analyze the stability properties of the \textit{Bidirectional Fréchet Representation} employed in our method. For arbitrary fixed-size pivot sets, the guarantee below is non-expansiveness; non-contraction requires a separate pivot-separation condition.

\begin{definition}[Bidirectional Fréchet Representation]
Let $(\mathcal{V}, d)$ be a finite quasimetric space and $\mathcal{P} = \{p_1, \dots, p_M\} \subset \mathcal{V}$ be a fixed set of $M$ pivot nodes. We first define the unnormalized bidirectional pivot-coordinate vector
\begin{equation}
    \widetilde{\Phi}(v) =
    \bigoplus_{m=1}^M \left[ d(v, p_m), \, d(p_m, v) \right].
\end{equation}
The Bidirectional Fréchet Representation used by SPACE is the normalized embedding
\begin{equation}
    \Phi(v) = \frac{1}{\sqrt{2M}}\widetilde{\Phi}(v).
\end{equation}
Here $\oplus$ denotes vector concatenation. The embedding space is equipped with the standard Euclidean norm $\|\cdot\|_2$.
\end{definition}

We now prove that this specific construction satisfies a Lipschitz condition with respect to the symmetrized quasimetric geometry. This property shows that BFR does not artificially expand symmetrized geometric distances; feasibility constraints are still enforced by the decoder mask rather than by the embedding alone.

\begin{proposition}[1-Lipschitz Continuity in Quasimetric Spaces]
\label{prop:lipschitz_continuity}
The normalized Bidirectional Fréchet Embedding $\Phi(v) = \frac{1}{\sqrt{2M}} \bigoplus_{m} [d(v, p_m), d(p_m, v)]$ is \textit{1-Lipschitz} (non-expansive) with respect to the symmetrized distance metric $D_{\mathrm{sym}}(v_i, v_j) = \max(d(v_i, v_j), d(v_j, v_i))$. That is, for all $v_i, v_j \in \mathcal{V}$:
\begin{equation}
    \| \Phi(v_i) - \Phi(v_j) \|_2 \le D_{\mathrm{sym}}(v_i, v_j).
\end{equation}
\end{proposition}

\begin{proof}
Consider the contribution of a single pivot $p_m$ to the squared Euclidean distance in the embedding space. The component-wise difference vector for the $m$-th pivot is:
\begin{equation}
    \delta_m = \left[ d(v_i, p_m) - d(v_j, p_m), \quad d(p_m, v_i) - d(p_m, v_j) \right].
\end{equation}
We analyze the two components separately using the triangle inequality inherent to the quasimetric $d$:

\textbf{1. Outgoing Component:}
By the triangle inequality, $d(v_i, p_m) \le d(v_i, v_j) + d(v_j, p_m)$. Rearranging gives $d(v_i, p_m) - d(v_j, p_m) \le d(v_i, v_j)$.
Similarly, $d(v_j, p_m) \le d(v_j, v_i) + d(v_i, p_m)$ implies $d(v_j, p_m) - d(v_i, p_m) \le d(v_j, v_i)$.
Combining these, we bound the absolute difference:
\begin{equation}
    |d(v_i, p_m) - d(v_j, p_m)| \le \max(d(v_i, v_j), d(v_j, v_i)) = D_{\mathrm{sym}}(v_i, v_j).
\end{equation}

\textbf{2. Incoming Component:}
By the triangle inequality, $d(p_m, v_i) \le d(p_m, v_j) + d(v_j, v_i)$. Rearranging gives $d(p_m, v_i) - d(p_m, v_j) \le d(v_j, v_i)$.
Similarly, $d(p_m, v_j) \le d(p_m, v_i) + d(v_i, v_j)$ implies $d(p_m, v_j) - d(p_m, v_i) \le d(v_i, v_j)$.
Thus, the absolute difference is also bounded:
\begin{equation}
    |d(p_m, v_i) - d(p_m, v_j)| \le \max(d(v_j, v_i), d(v_i, v_j)) = D_{\mathrm{sym}}(v_i, v_j).
\end{equation}

Now, summing over all $M$ pivots and applying the normalization factor $\frac{1}{\sqrt{2M}}$:
\begin{align}
    \| \Phi(v_i) - \Phi(v_j) \|_2^2 &= \sum_{m=1}^M \left( \frac{1}{2M}|d(v_i, p_m) - d(v_j, p_m)|^2 + \frac{1}{2M}|d(p_m, v_i) - d(p_m, v_j)|^2 \right) \\
    &\le \frac{1}{2M} \sum_{m=1}^M \left( D_{\mathrm{sym}}(v_i, v_j)^2 + D_{\mathrm{sym}}(v_i, v_j)^2 \right) \\
    &= \frac{1}{2M} \cdot 2M \cdot D_{\mathrm{sym}}(v_i, v_j)^2 \\
    &= D_{\mathrm{sym}}(v_i, v_j)^2.
\end{align}
Taking the square root completes the proof:
\begin{equation}
    \| \Phi(v_i) - \Phi(v_j) \|_2 \le D_{\mathrm{sym}}(v_i, v_j).
\end{equation}
\end{proof}

\begin{proposition}[Conditional Non-Contraction under Pivot Separation]
\label{prop:pivot_separation_lower_bound}
For the normalized BFR in Proposition \ref{prop:lipschitz_continuity}, define the pivot separation score
\begin{equation}
    s_{\mathcal{P}}(v_i,v_j)
    =
    \max_{p\in\mathcal{P}}
    \max\left\{
    |d(v_i,p)-d(v_j,p)|,\,
    |d(p,v_i)-d(p,v_j)|
    \right\}.
\end{equation}
Then, for all $v_i,v_j\in\mathcal{V}$,
\begin{equation}
    \frac{s_{\mathcal{P}}(v_i,v_j)}{\sqrt{2M}}
    \le
    \|\Phi(v_i)-\Phi(v_j)\|_2
    \le
    D_{\mathrm{sym}}(v_i,v_j).
\end{equation}
Consequently, if the pivot set is $\eta$-separating with respect to $D_{\mathrm{sym}}$, i.e.,
\begin{equation}
    s_{\mathcal{P}}(v_i,v_j) \ge \eta\, D_{\mathrm{sym}}(v_i,v_j)
    \quad \text{for all } v_i\neq v_j,
\end{equation}
then BFR is bi-Lipschitz with respect to $D_{\mathrm{sym}}$ and has distortion at most $\sqrt{2M}/\eta$. Without this separation condition, an arbitrary pivot subset can collapse distinct nodes, so no bounded-distortion guarantee is claimed.
\end{proposition}

\begin{proof}
The upper bound is exactly Proposition \ref{prop:lipschitz_continuity}. For the lower bound, choose a pivot-coordinate pair attaining $s_{\mathcal{P}}(v_i,v_j)$. Since $\|\Phi(v_i)-\Phi(v_j)\|_2$ is the Euclidean norm of all normalized coordinate differences, it is at least the magnitude of this single normalized coordinate:
\begin{equation}
    \|\Phi(v_i)-\Phi(v_j)\|_2
    \ge
    \frac{s_{\mathcal{P}}(v_i,v_j)}{\sqrt{2M}}.
\end{equation}
If $s_{\mathcal{P}}(v_i,v_j) \ge \eta\, D_{\mathrm{sym}}(v_i,v_j)$ for every distinct pair, then
\begin{equation}
    \frac{\eta}{\sqrt{2M}} D_{\mathrm{sym}}(v_i,v_j)
    \le
    \|\Phi(v_i)-\Phi(v_j)\|_2
    \le
    D_{\mathrm{sym}}(v_i,v_j),
\end{equation}
which is the bi-Lipschitz condition with $\alpha=\eta/\sqrt{2M}$ and $\beta=1$. The distortion is therefore at most $\beta/\alpha=\sqrt{2M}/\eta$.
\end{proof}

\begin{proposition}[Coverage-Induced Separation]
\label{prop:coverage_induced_separation}
Let $\mathcal{P}\subseteq\mathcal{V}$ be a pivot set with $M=|\mathcal{P}|$ and covering radius
\begin{equation}
    \rho_{\mathrm{sym}}(\mathcal{P})
    =
    \max_{v\in\mathcal{V}}\min_{p\in\mathcal{P}}D_{\mathrm{sym}}(v,p)
\end{equation}
under the symmetrized metric $D_{\mathrm{sym}}$. For any two nodes $v_i,v_j\in\mathcal{V}$, the pivot separation score in Proposition \ref{prop:pivot_separation_lower_bound} satisfies
\begin{equation}
    s_{\mathcal{P}}(v_i,v_j)
    \ge
    \left(D_{\mathrm{sym}}(v_i,v_j)-2\rho_{\mathrm{sym}}(\mathcal{P})\right)_{+},
\end{equation}
where $(a)_{+}=\max(a,0)$. Consequently, the normalized BFR satisfies
\begin{equation}
    \|\Phi(v_i)-\Phi(v_j)\|_2
    \ge
    \frac{\left(D_{\mathrm{sym}}(v_i,v_j)-2\rho_{\mathrm{sym}}(\mathcal{P})\right)_{+}}{\sqrt{2M}}.
\end{equation}
\end{proposition}

\begin{proof}
The claim is trivial when $D_{\mathrm{sym}}(v_i,v_j)\le 2\rho_{\mathrm{sym}}(\mathcal{P})$, since $s_{\mathcal{P}}(v_i,v_j)\ge 0$. We therefore consider the case $D_{\mathrm{sym}}(v_i,v_j)>2\rho_{\mathrm{sym}}(\mathcal{P})$.

By the definition of the covering radius, there exists a pivot $p\in\mathcal{P}$ such that $D_{\mathrm{sym}}(v_j,p)\le \rho_{\mathrm{sym}}(\mathcal{P})$. Hence both $d(v_j,p)\le \rho_{\mathrm{sym}}(\mathcal{P})$ and $d(p,v_j)\le \rho_{\mathrm{sym}}(\mathcal{P})$.

If $D_{\mathrm{sym}}(v_i,v_j)=d(v_i,v_j)$, then the quasimetric triangle inequality gives
\begin{equation}
    d(v_i,v_j)
    \le
    d(v_i,p)+d(p,v_j),
\end{equation}
and therefore $d(v_i,p)\ge D_{\mathrm{sym}}(v_i,v_j)-\rho_{\mathrm{sym}}(\mathcal{P})$. Combining this with $d(v_j,p)\le \rho_{\mathrm{sym}}(\mathcal{P})$ yields
\begin{equation}
    |d(v_i,p)-d(v_j,p)|
    \ge
    D_{\mathrm{sym}}(v_i,v_j)-2\rho_{\mathrm{sym}}(\mathcal{P}).
\end{equation}

If $D_{\mathrm{sym}}(v_i,v_j)=d(v_j,v_i)$, then
\begin{equation}
    d(v_j,v_i)
    \le
    d(v_j,p)+d(p,v_i),
\end{equation}
so $d(p,v_i)\ge D_{\mathrm{sym}}(v_i,v_j)-\rho_{\mathrm{sym}}(\mathcal{P})$. Since $d(p,v_j)\le \rho_{\mathrm{sym}}(\mathcal{P})$, we have
\begin{equation}
    |d(p,v_i)-d(p,v_j)|
    \ge
    D_{\mathrm{sym}}(v_i,v_j)-2\rho_{\mathrm{sym}}(\mathcal{P}).
\end{equation}

In both cases, at least one outgoing or incoming pivot coordinate separates $v_i$ and $v_j$ by the stated amount, proving the lower bound on $s_{\mathcal{P}}(v_i,v_j)$. The lower bound on $\|\Phi(v_i)-\Phi(v_j)\|_2$ then follows directly from Proposition \ref{prop:pivot_separation_lower_bound}.
\end{proof}

\begin{remark}[Connection to Furthest Pivot Sampling]
\label{rem:fps_coverage_connection}
Proposition \ref{prop:coverage_induced_separation} explains why reducing the pivot covering radius improves the non-contraction behavior of BFR: node pairs whose symmetrized distance is larger than $2\rho_{\mathrm{sym}}(\mathcal{P})$ must be separated by at least one pivot-distance coordinate. In the implementation, FPS is performed with $d_{\mathrm{fps}}(u,v)=\frac{1}{2}(d(u,v)+d(v,u))$ for smoother coverage. Since $d_{\mathrm{fps}}(u,v)\le D_{\mathrm{sym}}(u,v)\le 2d_{\mathrm{fps}}(u,v)$, a small covering radius under $d_{\mathrm{fps}}$ also controls $\rho_{\mathrm{sym}}(\mathcal{P})$ up to a constant factor, while the final BFR still retains the directed coordinates $d(v,p)$ and $d(p,v)$.
\end{remark}

\begin{remark}[Handling Asymmetry via Symmetric Embeddings]
A fundamental challenge in embedding asymmetric VRPs is that standard vector spaces (like $\ell_2$) are symmetric norms. One might worry that mapping a directed graph to $\ell_2$ loses information. However, Proposition \ref{prop:lipschitz_continuity} only asserts \textit{stability} (closeness preservation), not \textit{isometry} or unconditional low distortion. The asymmetry is explicitly encoded in the \textit{values} of the coordinates (i.e., $\Phi(v)$ contains both $d(v, p)$ and $d(p, v)$). Thus, the neural network can learn a non-symmetric decoding function $f(\Phi(v_i), \Phi(v_j)) \approx d(v_i, v_j)$ even though the metric $\|\Phi(v_i) - \Phi(v_j)\|_2$ is symmetric.
\end{remark}

\begin{corollary}[Necessity of Bidirectionality for Directional Identifiability]
\label{cor:bidirectionality}
Let $\mathcal{P} = \{p_m\}_{m=1}^M \subset \mathcal{V}$ be a fixed pivot set. A normalized unidirectional embedding $\Phi_{out}(v) = \frac{1}{\sqrt{M}}\bigoplus_{m=1}^M [d(v, p_m)]$ can collapse nodes that are indistinguishable by their outgoing distances to $\mathcal{P}$ but differ in their incoming geometry. More precisely, there exists a finite quasimetric space and distinct nodes $v_i, v_j$ such that
\begin{equation}
    \Phi_{out}(v_i) = \Phi_{out}(v_j)
    \qquad \text{while} \qquad
    \Phi(v_i) \neq \Phi(v_j),
\end{equation}
where $\Phi(v) = \frac{1}{\sqrt{2M}}\bigoplus_{m=1}^M [d(v, p_m), d(p_m, v)]$ is the normalized bidirectional Fr\'echet representation. Therefore, bidirectionality is necessary to retain directional information relative to the chosen pivots. This statement does not claim that $\Phi$ is globally injective for an arbitrary pivot subset; such a guarantee would require additional assumptions on $\mathcal{P}$.
\end{corollary}

\begin{proof}
It suffices to construct one asymmetric example. Let $\mathcal{V} = \{p, v_i, v_j\}$ and let $d$ be the shortest-path distance induced by the directed graph with edge lengths
\begin{equation*}
    p \to v_i : 2,\qquad
    p \to v_j : 1,\qquad
    v_i \to p : 1,\qquad
    v_j \to p : 1,\qquad
    v_i \to v_j : 1,\qquad
    v_j \to v_i : 2.
\end{equation*}
Shortest-path distances define a valid quasimetric, so the triangle inequality holds automatically. In particular,
\begin{equation*}
    d(v_i, p) = d(v_j, p) = 1,\qquad
    d(p, v_i) = 2,\qquad
    d(p, v_j) = 1.
\end{equation*}

Choose the pivot set $\mathcal{P} = \{p\}$. Then the outgoing-only representation collapses the two nodes:
\begin{equation}
    \Phi_{out}(v_i) = [d(v_i, p)] = [1]
    \quad \text{and} \quad
    \Phi_{out}(v_j) = [d(v_j, p)] = [1].
\end{equation}
Hence $\Phi_{out}(v_i) = \Phi_{out}(v_j)$ even though $v_i \neq v_j$.

\textbf{Bidirectionality:}
The normalized Bidirectional Fréchet Embedding $\Phi(v_i) = \frac{1}{\sqrt{2}}[d(v_i, p), d(p, v_i)]$ incorporates the ingress distances:
\begin{equation}
    \Phi(v_i) = \frac{1}{\sqrt{2}}[d(v_i, p), d(p, v_i)] = \frac{1}{\sqrt{2}}[1, 2],
    \qquad
    \Phi(v_j) = \frac{1}{\sqrt{2}}[d(v_j, p), d(p, v_j)] = \frac{1}{\sqrt{2}}[1, 1].
\end{equation}
Therefore $\Phi(v_i) \neq \Phi(v_j)$, and in fact
\begin{equation}
    \|\Phi(v_i) - \Phi(v_j)\|_2 = \frac{1}{\sqrt{2}}.
\end{equation}

This proves the strict informational advantage claimed here: a unidirectional Fr\'echet representation can lose directional distinctions that the bidirectional one preserves. What follows is the necessity of the extra coordinates for asymmetric identifiability relative to the chosen pivots, not a global injectivity guarantee for arbitrary pivot subsets.
\end{proof}

\clearpage

\section{Implementation Details of Furthest Pivot Sampling}
\label{app:fps_details}

The representational power of the SPACE framework depends critically on the spatial distribution of the pivot set $\mathcal{P}$. While the main text outlines the theoretical motivation for Furthest Pivot Sampling (FPS), this appendix details the specific implementation strategies for metric symmetrization, training stochasticity, and inference-time augmentation.

\subsection{Metric Symmetrization for Pivot Selection}
\label{app:fps_metric}

In asymmetric VRPs, the raw distance matrix $\bm{D}$ implies a directed graph topology where $d(v_i, v_j) \neq d(v_j, v_i)$. Standard FPS assumes a symmetric metric space to define a unique "distance" between nodes. To resolve this ambiguity without altering the underlying problem structure, we compute the pivot set $\mathcal{P}$ based on a symmetrized reference metric $d_{\mathrm{fps}}$. We define this metric as the mean of the bidirectional costs:
\begin{equation}
    d_{\mathrm{fps}}(v_i, v_j) = \frac{d(v_i, v_j) + d(v_j, v_i)}{2}.
\end{equation}

\begin{proposition}[Metricity and Coverage of FPS]
\label{prop:fps_metric_coverage}
If $d$ is a quasimetric, then $d_{\mathrm{fps}}$ is a metric. Let $\mathcal{P}_{\mathrm{init}}\subseteq\mathcal{V}$ be fixed, and let $\mathcal{P}_M$ be the pivot set obtained by applying the greedy rule in Eq.~\eqref{eq:fps_greedy} until $|\mathcal{P}_M|=M$. Define the covering radius
\begin{equation}
    R(\mathcal{P})=\max_{v\in\mathcal{V}}\min_{p\in\mathcal{P}}d_{\mathrm{fps}}(v,p).
\end{equation}
If $\mathcal{P}^{\star}$ minimizes $R(\mathcal{P})$ among all $M$-point pivot sets that contain the same initial seeds $\mathcal{P}_{\mathrm{init}}$, then
\begin{equation}
    R(\mathcal{P}_M)\le 2R(\mathcal{P}^{\star}).
\end{equation}
\end{proposition}

\begin{proof}
Symmetry and identity of indiscernibles for $d_{\mathrm{fps}}$ follow directly from the definition and the corresponding quasimetric assumptions on $d$. For the triangle inequality, for any $x,y,z\in\mathcal{V}$,
\begin{align}
    d_{\mathrm{fps}}(x,z)
    &=\frac{d(x,z)+d(z,x)}{2} \\
    &\le \frac{d(x,y)+d(y,z)+d(z,y)+d(y,x)}{2} \\
    &= d_{\mathrm{fps}}(x,y)+d_{\mathrm{fps}}(y,z).
\end{align}
Thus $d_{\mathrm{fps}}$ is a metric.

Let $r=R(\mathcal{P}_M)$ and $r^\star=R(\mathcal{P}^{\star})$. Suppose, for contradiction, that $r>2r^\star$. Let $x$ be a point whose distance to $\mathcal{P}_M$ is $r$, and let $q_1,\ldots,q_{M-|\mathcal{P}_{\mathrm{init}}|}$ be the pivots added greedily after initialization. Because covering radii only decrease during furthest-first traversal, each $q_a$ was at a distance of at least $r$ from all previously selected pivots when it was chosen; also, $x$ is at a distance of at least $r$ from every pivot in $\mathcal{P}_M$. Hence the points $\{x,q_1,\ldots,q_{M-|\mathcal{P}_{\mathrm{init}}|}\}$ are pairwise more than $2r^\star$ apart under $d_{\mathrm{fps}}$.

Every point of this set must be within distance $r^\star$ of some center in $\mathcal{P}^{\star}$. None can be covered by an initial seed, since each is more than $r^\star$ away from all seeds. Moreover, two of them cannot be covered by the same non-seed center in $\mathcal{P}^{\star}$; otherwise, the triangle inequality would put them within distance at most $2r^\star$ of each other. This would require $M-|\mathcal{P}_{\mathrm{init}}|+1$ distinct non-seed centers in $\mathcal{P}^{\star}$, but only $M-|\mathcal{P}_{\mathrm{init}}|$ are available. This contradiction proves $r\le 2r^\star$.
\end{proof}

\begin{remark}
    While the theoretical Lipschitz bound (Proposition \ref{prop:lipschitz_continuity}) utilizes the maximum distance $D_{\mathrm{sym}}$ to guarantee a strict upper bound, we employ the mean distance $d_{\mathrm{fps}}$ for pivot sampling. This is because the mean provides smoother density estimates of the underlying topology, thereby avoiding outlier effects caused by extreme one-way penalties often encountered in asymmetric VRPs. Using $d_{\mathrm{fps}}$ ensures that the selected pivots cover the topology uniformly, regardless of edge directionality. Crucially, this symmetrization applies \textit{only} to the selection indices; the final embedding $\Phi(v)$ retains the asymmetric distance coordinates $d(v_i, p_m)$ and $d(p_m, v_i)$ up to the global normalization factor, preserving the directed topology required for routing.
\end{remark}

\subsection{Seed Initialization and Training Strategy}
\label{app:fps_training}

To prevent the model from overfitting to a fixed geometric constellation, we introduce stochasticity into the pivot generation process during training. We initialize the pivot set $\mathcal{P}_{init}$ not just with the depot node, but with a stochastic combination of the depot and a randomly sampled customer node.

Formally, let $\mathcal{V}_{depot}$ denote the set of depot nodes (typically $|\mathcal{V}_{depot}|=1$ for standard CVRP and $|\mathcal{V}_{depot}|=3$ for MDCVRP). For each training instance, we construct the initial pivot set $\mathcal{P}_{init}$ as:
\begin{equation}
    \mathcal{P}_{init} = \mathcal{V}_{depot} \cup \{ v_{rand} \}, \quad \text{where } v_{rand} \sim \text{Uniform}(\mathcal{V} \setminus \mathcal{V}_{depot}).
\end{equation}
The FPS algorithm then iteratively selects the remaining $M - |\mathcal{P}_{init}|$ pivots relative to this set. This strategy forces the encoder to learn robust representations that are invariant to the choice of reference points, effectively serving as a geometric regularization technique.

\subsection{Pivot-Induced Multi-View Augmentation (Inference)}
\label{app:fps_inference}

Standard test-time augmentation techniques, such as coordinate rotation or flipping, are not applicable to coordinate-free embeddings. Instead, we leverage the pivot-based architecture to generate diverse "views" of the same problem instance. By varying the pivot set $\mathcal{P}$, we alter the embedding $\Phi(v)$, thereby probing the solution space from different geometric perspectives.

We implement a deterministic yet diverse augmentation schedule based on the augmentation factor:

\begin{enumerate}
    \item \textbf{Phase 1 (Customer Traversal):} For augmentations $\alpha \in \{1, \dots, \min(N_{aug}, |\mathcal{V}_{customer}|)\}$, we systematically seed the FPS with the depot set and the $\alpha$-th customer from a fixed random permutation of the customer set $\mathcal{V}_{customer}$. This ensures that for small $N_{aug}$, the model views the graph from maximally distinct vantage points.
    
    \item \textbf{Phase 2 (Dual-Customer Seeding):} If $N_{aug} > |\mathcal{V}_{customer}|$, we exhaust single-customer seeds and transition to dual-customer seeding. We initialize $\mathcal{P}_{init}$ with the depot set and two distinct customers selected at random. This increases the entropy of the initial geometric frame, generating further diverse embeddings for deep search.
\end{enumerate}

Given the resulting batch of $N_{aug}$ solutions, we select the final solution $\pi^*$ that minimizes the objective function: $\pi^* = \arg\min_{m} \mathcal{J}(\pi_T)$.

\subsection{Computational Overhead}

We implement FPS using a vectorized approach to minimize computational overhead. Rather than recomputing pairwise distances at each step, we maintain a distance cache vector $\bm{\delta} \in \mathbb{R}^{|\mathcal{V}|}$, where $\delta_i$ stores the minimum distance from node $v_i$ to the current set of selected pivots.

Algorithm \ref{alg:vectorized_fps} summarizes the procedure. By efficiently utilizing the `scatter` and `gather` operations on the GPU, the computational cost of pivot selection becomes negligible compared to that of the Transformer encoder-decoder operations.

\begin{algorithm}[h]
   \caption{Vectorized Furthest Pivot Sampling with Seeds}
   \label{alg:vectorized_fps}
\begin{algorithmic}
   \STATE {\bfseries Input:} Symmetrized Distance Matrix $\bm{D}_{\mathrm{fps}}$, Number of pivots $M$, Initial pivot set $\mathcal{P}_{init} \subset \mathcal{V}$
   \STATE {\bfseries Output:} Pivot Indices $\mathcal{P}$
   \STATE Initialize pivot list $\mathcal{P} \leftarrow \mathcal{P}_{init}$
   \STATE Initialize distance cache $\bm{\delta} \leftarrow \infty$
   \FOR{each pivot $p \in \mathcal{P}_{init}$}
       \STATE Update cache: $\bm{\delta} \leftarrow \min(\bm{\delta}, \bm{D}_{\mathrm{fps}}[:, p])$
   \ENDFOR
   \WHILE{$|\mathcal{P}| < M$}
       \STATE Select next pivot: $p_{next} \leftarrow \arg\max (\bm{\delta})$
       \STATE Add to list: $\mathcal{P} \leftarrow \mathcal{P} \cup \{p_{next}\}$
       \STATE Update cache: $\bm{\delta} \leftarrow \min(\bm{\delta}, \bm{D}_{\mathrm{fps}}[:, p_{next}])$
       \STATE Mask selected: $\bm{\delta}[p_{next}] \leftarrow -\infty$
   \ENDWHILE
\end{algorithmic}
\end{algorithm}

\clearpage

\section{Implementation Details of Weight-Decomposed Adaptive Decoding}
\label{app:dora_details}
In this section, we elaborate on the implementation of the weight-decomposed adaptive decoding strategy described in Section \ref{subsec:dora_adaptation}. We employ a specialized variant of DoRA \citep{liu2024dora} designed for multi-attribute routing tasks, where constraint-specific logic is learned via additive weight updates conditioned on the multi-hot problem representation $\bm{\lambda}$.

\paragraph{Task Attributes and Architecture}

The adaptation mechanism is conditioned on a binary task attribute vector $\bm{\lambda} \in \{0, 1\}^{10}$, which encodes the presence of specific routing constraints. For each attention projection matrix in the decoder (query, key, value, and constraint embeddings), the final effective weight $W_{\mathrm{eff}}^{(\bm{\lambda})}$ is computed as the sum of a shared base weight $W_{0}$ and a task-specific update $\Delta W$:
\begin{equation}
    W_{\mathrm{eff}}^{(\bm{\lambda})} = W_0 + \Delta W^{(\bm{\lambda})}.
\end{equation}
This ensures that $W_{0}$ captures universal geometric reasoning while $\Delta W^{(\bm{\lambda})}$ specializes in constraint logic.

\paragraph{Multi-Head DoRA Layer}

We implement a \textbf{TaskDoRALayer} that computes $\Delta W^{(\bm{\lambda})}$ by aggregating updates from active constraints. To enhance the representational capacity for complex constraints, each task attribute uses $H$ independent heads, each a low-rank DoRA adapter.

For each attribute $\bm{\lambda}_j$ and each head $\eta \in \{1, \dots, H\}$, we maintain three learnable parameters:
\begin{itemize}
    \item Directional matrices $W_{down,\eta}^{(j)} \in \mathbb{R}^{r \times d_{\mathrm{in}}}$ and $W_{up,\eta}^{(j)} \in \mathbb{R}^{d_{\mathrm{out}} \times r}$.
    \item A magnitude vector $g_{\eta}^{(j)} \in \mathbb{R}^{d_{\mathrm{out}}}$.
\end{itemize}
The update contribution $\delta_j$ for attribute $\bm{\lambda}_j$ is the summation over its heads:
\begin{equation}
    \delta_j = \sum_{\eta=1}^{H} g_\eta^{(j)} \odot \frac{W_{up,\eta}^{(j)} W_{down,\eta}^{(j)}}{\|W_{up,\eta}^{(j)} W_{down,\eta}^{(j)}\|_{\mathrm{row}} + \epsilon},
\end{equation}
where $\odot$ denotes row-wise broadcasting multiplication and $\|\cdot\|_{\mathrm{row}}$ is the $\ell_2$ norm along the input dimension of each output row. 

The final task-specific update $\Delta W^{(\bm{\lambda})}$ is the mean aggregation of all active attribute updates:
\begin{equation}
    \Delta W^{(\bm{\lambda})}
= \frac{1}{Z(\bm{\lambda})}\sum_{j:\lambda_j=1}\delta_j,
\quad
Z(\bm{\lambda})=\max(1,\|\bm{\lambda}\|_1).
\end{equation}
Using the mean (rather than sum) prevents magnitude explosion when multiple constraints are active simultaneously, preserving the scale of the effective weights. The normalization by $Z(\bm{\lambda})$ averages active-attribute updates when constraints are present and yields $\Delta W^{(\bm{\lambda})}=0$ when no attribute is active, so the decoder falls back to the shared base weight $W_0$.

\paragraph{Initialization Strategy}
Proper initialization is critical for the additive formulation. We initialize parameters to ensure that the model starts with $\Delta W \approx \mathbf{0}$, effectively behaving as the base model initially:
\begin{itemize}
    \item The down-projection tensor $W_{down}^{(j)}$ is initialized with Kaiming uniform initializer using $a=\sqrt{5}$~\citep{he2015delving}. Under the fan-in convention for the stacked tensor $W_{down}^{(j)}\in\mathbb{R}^{H\times r\times d_{\mathrm{in}}}$, this is equivalent to $\mathcal{U}(-1/\sqrt{r d_{\mathrm{in}}}, 1/\sqrt{r d_{\mathrm{in}}})$.
    \item $W_{up,\eta}^{(j)} \sim \mathcal{N}(0, 0.01^2)$, matching the small random initialization used to avoid zero direction vectors.
    \item $g_\eta^{(j)} = \mathbf{0}$.
\end{itemize}
The initialization of $g_\eta^{(j)}$ to zero guarantees that the initial contribution of any constraint adapter is null, allowing the optimization trajectory to gradually introduce constraint-specific logic as needed.

\clearpage

\section{Multi-hot Problem Representation}
\label{append:multi_hot}

Based on UDR, we design a 10-dimensional multi-hot problem representation $\bm{\lambda} \in \{0, 1\}^{10}$ to characterize the underlying constraints of each problem. To ensure a concise yet comprehensive encoding, we integrate specific attributes into unified categories. Thus, compared to the URS~\citep{zhou2025urs}, we make the following changes: (1) Since our spatial pivot-aligned representation exists for any problem, we do not use it for problem representation; (2) For time-related constraints(Earliest Arrival Time, Latest Arrival Time, and Service Time), we consolidate them into a single Time indicator; and (3) To ensure a compact problem representation, we unify delivery and linehaul nodes under a single delivery attribute, the distinction between backhaul and PD constraints is subsequently established through the co-occurrence of this attribute with either a backhaul or a pickup indicator. 

Rather than capturing numerical values, this multi-hot vector $\bm{\lambda}$ serves as a high-level descriptor of the problem's properties. By encoding whether an instance satisfies specific constraints, this representation enables the adaptive adjustment of model parameters based on the problem variant at hand.

\subsection{Problem Representations of Seen VRPs}
\label{append:training_problems_selection}
As illustrated in Table \ref{tab:problem_attributes_seen}, we train SPACE across 12 diverse VRP variants to ensure that every feature within the UDR is activated at least once during the learning phase. Following \citet{zhou2025urs}, the selection of the 12 training problems is guided by a "minimum redundancy" principle: we ensure that most attributes in UDR appear in at least two distinct problem types during training, which helps prevent the model from overfitting a specific attribute to a single problem type. The only two exceptions are the Penalty and Pickup attributes, as preliminary experiments have shown that exposure in a single problem type is sufficient for generalization to variants such as SPCTSP and PDCVRP.

\subsection{Problem Representations of Unseen VRPs}
We construct 98 unseen variants to further highlight the broad applicability of SPACE, each incorporating at least one of nine constraints to reflect diverse real-world scenarios. Detailed representations for all unseen symmetric and asymmetric VRP variants are provided in Table \ref{tab:problem_attributes_unseen_S} and Table \ref{tab:problem_attributes_unseen_A}, respectively.

\paragraph{Oracle Solvers}
We use classical solvers to obtain the oracle solutions for computing optimality gaps. Specifically, the oracle solvers include HGS-PyVRP~\citep{wouda2024pyvrp}, LKH3~\citep{LKH3}, and OR-Tools~\citep{ortools}. We run LKH3~\citep{LKH3} with $10000$ trials and 1 run~\citep{kool2019attention}. For HGS-PyVRP and OR-Tools, we run them on a single CPU core with a time limit of $20$s~\citep{li2025cada}. The oracle solver used for each VRP variant is listed in the Oracle column of the following tables.

\begin{table}[h!]
  \centering
  \caption{Problem representations for 12 seen routing problems in our experiments. }
  \label{tab:problem_attributes_seen}
  \resizebox{0.99\textwidth}{!}{
  \begin{tabular}{l|l|cccc|cccccc}
    \toprule[0.5mm]
     Problem & Oracle 
     & Demand & Prize & Penalty & Time
     & Depot & Pickup & Backhaul & Delivery & Sub-routes & Open Route \\
    \midrule
    ATSP        & LKH-3     
        &  &  &  &  
        &  &  &  &  &  &  \\
    TSP         & LKH-3     
        &  &  &  &  
        &  &  &  &  &  &  \\
    CVRP        & PyVRP     
        & $\checkmark$ &  &  &  
        & $\checkmark$ &  &  & $\checkmark$ & $\checkmark$ &  \\
    ACVRP       & PyVRP     
        & $\checkmark$ &  &  &  
        & $\checkmark$ &  &  & $\checkmark$ & $\checkmark$ &  \\
    OP          & Compass   
        &  & $\checkmark$ &  &  
        & $\checkmark$ &  &  &  &  &  \\
    PCTSP       & ILS       
        &  & $\checkmark$ & $\checkmark$ &  
        & $\checkmark$ &  &  &  &  &  \\
    PDTSP       & LKH-3     
        &  &  &  &  
        & $\checkmark$ & $\checkmark$ &  & $\checkmark$ &  &  \\
    CVRPTW      & PyVRP     
        & $\checkmark$ &  &  & $\checkmark$
        & $\checkmark$ &  &  & $\checkmark$ & $\checkmark$ &  \\
    OCVRP       & LKH-3     
        & $\checkmark$ &  &  &  
        & $\checkmark$ &  &  & $\checkmark$ & $\checkmark$ & $\checkmark$ \\
    CVRPB       & OR-Tools  
        & $\checkmark$ &  &  &  
        & $\checkmark$ &  & $\checkmark$ & $\checkmark$ & $\checkmark$ &  \\
    OCVRPTW     & OR-Tools  
        & $\checkmark$ &  &  & $\checkmark$
        & $\checkmark$ &  &  & $\checkmark$ & $\checkmark$ & $\checkmark$ \\
    ACVRPBTW    & OR-Tools  
        & $\checkmark$ &  &  & $\checkmark$
        & $\checkmark$ &  & $\checkmark$ & $\checkmark$ & $\checkmark$ &  \\
    \bottomrule[0.5mm]
  \end{tabular}
  }
\end{table}

\begin{table}[h!]
  \centering
  \caption{Problem representations of 46 unseen symmetric VRP variants for zero-shot generalization.}
  \label{tab:problem_attributes_unseen_S}
  \resizebox{0.99\textwidth}{!}{
  \begin{tabular}{l|l|cccc|cccccc}
    \toprule[0.5mm]
    Problem & Oracle
    & Demand & Prize & Penalty & Time
    & Depot & Pickup & Backhaul & Delivery & Sub-routes & Open Route \\
    \midrule
    CVRPL        & LKH-3     & $\checkmark$ &  &  &  & $\checkmark$ &  &  & $\checkmark$ & $\checkmark$ &  \\
    OCVRPB       & OR-Tools  & $\checkmark$ &  &  &  & $\checkmark$ &  & $\checkmark$ & $\checkmark$ & $\checkmark$ & $\checkmark$ \\
    CVRPBL       & OR-Tools  & $\checkmark$ &  &  &  & $\checkmark$ &  & $\checkmark$ & $\checkmark$ & $\checkmark$ &  \\
    CVRPLTW      & OR-Tools  & $\checkmark$ &  &  & $\checkmark$ & $\checkmark$ &  &  & $\checkmark$ & $\checkmark$ &  \\
    OCVRPBTW     & OR-Tools  & $\checkmark$ &  &  & $\checkmark$ & $\checkmark$ &  & $\checkmark$ & $\checkmark$ & $\checkmark$ & $\checkmark$ \\
    CVRPBLTW     & OR-Tools  & $\checkmark$ &  &  & $\checkmark$ & $\checkmark$ &  & $\checkmark$ & $\checkmark$ & $\checkmark$ &  \\
    OCVRPL       & OR-Tools  & $\checkmark$ &  &  &  & $\checkmark$ &  &  & $\checkmark$ & $\checkmark$ & $\checkmark$ \\
    CVRPBTW      & OR-Tools  & $\checkmark$ &  &  & $\checkmark$ & $\checkmark$ &  & $\checkmark$ & $\checkmark$ & $\checkmark$ &  \\
    OCVRPBL      & OR-Tools  & $\checkmark$ &  &  &  & $\checkmark$ &  & $\checkmark$ & $\checkmark$ & $\checkmark$ & $\checkmark$ \\
    OCVRPLTW     & OR-Tools  & $\checkmark$ &  &  & $\checkmark$ & $\checkmark$ &  &  & $\checkmark$ & $\checkmark$ & $\checkmark$ \\
    OCVRPBLTW    & OR-Tools  & $\checkmark$ &  &  & $\checkmark$ & $\checkmark$ &  & $\checkmark$ & $\checkmark$ & $\checkmark$ & $\checkmark$ \\
    CVRPBP       & OR-Tools  & $\checkmark$ &  &  &  & $\checkmark$ &  & $\checkmark$ & $\checkmark$ & $\checkmark$ &  \\
    OCVRPBP      & OR-Tools  & $\checkmark$ &  &  &  & $\checkmark$ &  & $\checkmark$ & $\checkmark$ & $\checkmark$ & $\checkmark$ \\
    CVRPBPL      & OR-Tools  & $\checkmark$ &  &  &  & $\checkmark$ &  & $\checkmark$ & $\checkmark$ & $\checkmark$ &  \\
    OCVRPBPTW    & OR-Tools  & $\checkmark$ &  &  & $\checkmark$ & $\checkmark$ &  & $\checkmark$ & $\checkmark$ & $\checkmark$ & $\checkmark$ \\
    CVRPBPLTW & OR-Tools
& $\checkmark$ & & & $\checkmark$ & $\checkmark$ & & $\checkmark$ & $\checkmark$ & $\checkmark$ & \\
    CVRPBPTW     & OR-Tools  & $\checkmark$ &  &  & $\checkmark$ & $\checkmark$ &  & $\checkmark$ & $\checkmark$ & $\checkmark$ &  \\
    OCVRPBPL     & OR-Tools  & $\checkmark$ &  &  &  & $\checkmark$ &  & $\checkmark$ & $\checkmark$ & $\checkmark$ & $\checkmark$ \\
    OCVRPBPLTW    & OR-Tools  & $\checkmark$ &  &  & $\checkmark$ & $\checkmark$ &  & $\checkmark$ & $\checkmark$ & $\checkmark$ & $\checkmark$ \\
    \midrule
MDCVRP        & PyVRP
  & $\checkmark$ &  &  & 
  & $\checkmark$ &  &  & $\checkmark$ & $\checkmark$ &  \\
MDCVRPTW      & PyVRP
  & $\checkmark$ &  &  & $\checkmark$
  & $\checkmark$ &  &  & $\checkmark$ & $\checkmark$ &  \\
MDOCVRP       & PyVRP
  & $\checkmark$ &  &  & 
  & $\checkmark$ &  &  & $\checkmark$ & $\checkmark$ & $\checkmark$ \\
MDCVRPL       & PyVRP
  & $\checkmark$ &  &  & 
  & $\checkmark$ &  &  & $\checkmark$ & $\checkmark$ &  \\
MDCVRPB       & PyVRP
  & $\checkmark$ &  &  & 
  & $\checkmark$ &  & $\checkmark$ & $\checkmark$ & $\checkmark$ &  \\
MDOCVRPTW     & PyVRP
  & $\checkmark$ &  &  & $\checkmark$
  & $\checkmark$ &  &  & $\checkmark$ & $\checkmark$ & $\checkmark$ \\
MDOCVRPB      & PyVRP
  & $\checkmark$ &  &  & 
  & $\checkmark$ &  & $\checkmark$ & $\checkmark$ & $\checkmark$ & $\checkmark$ \\
MDCVRPBL      & PyVRP
  & $\checkmark$ &  &  & 
  & $\checkmark$ &  & $\checkmark$ & $\checkmark$ & $\checkmark$ &  \\
MDCVRPLTW     & PyVRP
  & $\checkmark$ &  &  & $\checkmark$
  & $\checkmark$ &  &  & $\checkmark$ & $\checkmark$ &  \\
MDOCVRPBTW    & PyVRP
  & $\checkmark$ &  &  & $\checkmark$
  & $\checkmark$ &  & $\checkmark$ & $\checkmark$ & $\checkmark$ & $\checkmark$ \\
MDCVRPBLTW    & PyVRP
  & $\checkmark$ &  &  & $\checkmark$
  & $\checkmark$ &  & $\checkmark$ & $\checkmark$ & $\checkmark$ &  \\
MDOCVRPL      & PyVRP
  & $\checkmark$ &  &  & 
  & $\checkmark$ &  &  & $\checkmark$ & $\checkmark$ & $\checkmark$ \\
MDCVRPBTW     & PyVRP
  & $\checkmark$ &  &  & $\checkmark$
  & $\checkmark$ &  & $\checkmark$ & $\checkmark$ & $\checkmark$ &  \\
MDOCVRPBL     & PyVRP
  & $\checkmark$ &  &  & 
  & $\checkmark$ &  & $\checkmark$ & $\checkmark$ & $\checkmark$ & $\checkmark$ \\
MDOCVRPLTW    & PyVRP
  & $\checkmark$ &  &  & $\checkmark$
  & $\checkmark$ &  &  & $\checkmark$ & $\checkmark$ & $\checkmark$ \\
MDOCVRPBLTW   & PyVRP
  & $\checkmark$ &  &  & $\checkmark$
  & $\checkmark$ &  & $\checkmark$ & $\checkmark$ & $\checkmark$ & $\checkmark$ \\
MDCVRPBP      & PyVRP
  & $\checkmark$ &  &  & 
  & $\checkmark$ &  & $\checkmark$ & $\checkmark$ & $\checkmark$ &  \\
MDOCVRPBP     & PyVRP
  & $\checkmark$ &  &  & 
  & $\checkmark$ &  & $\checkmark$ & $\checkmark$ & $\checkmark$ & $\checkmark$ \\
MDCVRPBPL     & PyVRP
  & $\checkmark$ &  &  & 
  & $\checkmark$ &  & $\checkmark$ & $\checkmark$ & $\checkmark$ &  \\
MDOCVRPBPTW   & PyVRP
  & $\checkmark$ &  &  & $\checkmark$
  & $\checkmark$ &  & $\checkmark$ & $\checkmark$ & $\checkmark$ & $\checkmark$ \\
MDCVRPBPLTW    & PyVRP
  & $\checkmark$ &  &  & $\checkmark$
  & $\checkmark$ &  & $\checkmark$ & $\checkmark$ & $\checkmark$ &  \\
MDCVRPBPTW     & PyVRP
  & $\checkmark$ &  &  & $\checkmark$
  & $\checkmark$ &  & $\checkmark$ & $\checkmark$ & $\checkmark$ &  \\
MDOCVRPBPL    & PyVRP
  & $\checkmark$ &  &  & 
  & $\checkmark$ &  & $\checkmark$ & $\checkmark$ & $\checkmark$ & $\checkmark$ \\
MDOCVRPBPLTW   & PyVRP
  & $\checkmark$ &  &  & $\checkmark$
  & $\checkmark$ &  & $\checkmark$ & $\checkmark$ & $\checkmark$ & $\checkmark$ \\

    \midrule
    SPCTSP       & ILS       &  & $\checkmark$ & $\checkmark$ &  & $\checkmark$ &  &  &  &  &  \\
    PDCVRP & OR-Tools
& $\checkmark$ & & & & $\checkmark$ & $\checkmark$ & & $\checkmark$ & $\checkmark$ & \\

OPDCVRP & OR-Tools
& $\checkmark$ & & & & $\checkmark$ & $\checkmark$ & & $\checkmark$ & $\checkmark$ & $\checkmark$ \\
    \bottomrule[0.5mm]
  \end{tabular}
  }
\end{table}

\begin{table}[h!]
  \centering
  \caption{Problem representations of 52 unseen asymmetric VRP variants for zero-shot generalization.}
  \label{tab:problem_attributes_unseen_A}
  \resizebox{0.99\textwidth}{!}{
  \begin{tabular}{l|l|cccc|cccccc}
    \toprule[0.5mm]
    Problem & Oracle
    & Demand & Prize & Penalty & Time
    & Depot & Pickup & Backhaul & Delivery
    & Sub-routes & Open Route \\
    \midrule
    ACVRPTW        & OR-Tools
      & $\checkmark$ &  &  & $\checkmark$
      & $\checkmark$ &  &  & $\checkmark$
      & $\checkmark$ &  \\
    AOCVRP         & OR-Tools
      & $\checkmark$ &  &  &
      & $\checkmark$ &  &  & $\checkmark$
      & $\checkmark$ & $\checkmark$ \\
    ACVRPL         & OR-Tools
      & $\checkmark$ &  &  &
      & $\checkmark$ &  &  & $\checkmark$
      & $\checkmark$ &  \\
    ACVRPB         & OR-Tools
      & $\checkmark$ &  &  &
      & $\checkmark$ &  & $\checkmark$ & $\checkmark$
      & $\checkmark$ &  \\
    AOCVRPTW       & OR-Tools
      & $\checkmark$ &  &  & $\checkmark$
      & $\checkmark$ &  &  & $\checkmark$
      & $\checkmark$ & $\checkmark$ \\
    AOCVRPB        & OR-Tools
      & $\checkmark$ &  &  &
      & $\checkmark$ &  & $\checkmark$ & $\checkmark$
      & $\checkmark$ & $\checkmark$ \\
    ACVRPBL        & OR-Tools
      & $\checkmark$ &  &  &
      & $\checkmark$ &  & $\checkmark$ & $\checkmark$
      & $\checkmark$ &  \\
    ACVRPLTW       & OR-Tools
      & $\checkmark$ &  &  & $\checkmark$
      & $\checkmark$ &  &  & $\checkmark$
      & $\checkmark$ &  \\
    AOCVRPBTW      & OR-Tools
      & $\checkmark$ &  &  & $\checkmark$
      & $\checkmark$ &  & $\checkmark$ & $\checkmark$
      & $\checkmark$ & $\checkmark$ \\
    ACVRPBLTW      & OR-Tools
      & $\checkmark$ &  &  & $\checkmark$
      & $\checkmark$ &  & $\checkmark$ & $\checkmark$
      & $\checkmark$ &  \\
    AOCVRPL        & OR-Tools
      & $\checkmark$ &  &  &
      & $\checkmark$ &  &  & $\checkmark$
      & $\checkmark$ & $\checkmark$ \\
    ACVRPBTW       & OR-Tools
      & $\checkmark$ &  &  & $\checkmark$
      & $\checkmark$ &  & $\checkmark$ & $\checkmark$
      & $\checkmark$ &  \\
    AOCVRPBL       & OR-Tools
      & $\checkmark$ &  &  &
      & $\checkmark$ &  & $\checkmark$ & $\checkmark$
      & $\checkmark$ & $\checkmark$ \\
    AOCVRPLTW      & OR-Tools
      & $\checkmark$ &  &  & $\checkmark$
      & $\checkmark$ &  &  & $\checkmark$
      & $\checkmark$ & $\checkmark$ \\
    AOCVRPBLTW     & OR-Tools
      & $\checkmark$ &  &  & $\checkmark$
      & $\checkmark$ &  & $\checkmark$ & $\checkmark$
      & $\checkmark$ & $\checkmark$ \\
    ACVRPBP        & OR-Tools
      & $\checkmark$ &  &  &
      & $\checkmark$ &  & $\checkmark$ & $\checkmark$
      & $\checkmark$ &  \\
    AOCVRPBP       & OR-Tools
      & $\checkmark$ &  &  &
      & $\checkmark$ &  & $\checkmark$ & $\checkmark$
      & $\checkmark$ & $\checkmark$ \\
    ACVRPBPL       & OR-Tools
      & $\checkmark$ &  &  &
      & $\checkmark$ &  & $\checkmark$ & $\checkmark$
      & $\checkmark$ &  \\
    AOCVRPBPTW     & OR-Tools
      & $\checkmark$ &  &  & $\checkmark$
      & $\checkmark$ &  & $\checkmark$ & $\checkmark$
      & $\checkmark$ & $\checkmark$ \\
    ACVRPBPLTW      & OR-Tools
      & $\checkmark$ &  &  & $\checkmark$
      & $\checkmark$ &  & $\checkmark$ & $\checkmark$
      & $\checkmark$ &  \\
    AOCVRPBPL      & OR-Tools
      & $\checkmark$ &  &  &
      & $\checkmark$ &  & $\checkmark$ & $\checkmark$
      & $\checkmark$ & $\checkmark$ \\
    AOCVRPBPLTW    & OR-Tools
      & $\checkmark$ &  &  & $\checkmark$
      & $\checkmark$ &  & $\checkmark$ & $\checkmark$
      & $\checkmark$ & $\checkmark$ \\
    \midrule
AMDCVRP & OR-Tools
& $\checkmark$ & & & & $\checkmark$ & & & $\checkmark$ & $\checkmark$ & \\

AMDCVRPTW & OR-Tools
& $\checkmark$ & & & $\checkmark$ & $\checkmark$ & & & $\checkmark$ & $\checkmark$ & \\

AMDOCVRP & OR-Tools
& $\checkmark$ & & & & $\checkmark$ & & & $\checkmark$ & $\checkmark$ & $\checkmark$ \\

AMDCVRPL & OR-Tools
& $\checkmark$ & & & & $\checkmark$ & & & $\checkmark$ & $\checkmark$ & \\

AMDCVRPB & OR-Tools
& $\checkmark$ & & & & $\checkmark$ & & $\checkmark$ & $\checkmark$ & $\checkmark$ & \\

AMDOCVRPTW & OR-Tools
& $\checkmark$ & & & $\checkmark$ & $\checkmark$ & & & $\checkmark$ & $\checkmark$ & $\checkmark$ \\

AMDOCVRPB & OR-Tools
& $\checkmark$ & & & & $\checkmark$ & & $\checkmark$ & $\checkmark$ & $\checkmark$ & $\checkmark$ \\

AMDCVRPBL & OR-Tools
& $\checkmark$ & & & & $\checkmark$ & & $\checkmark$ & $\checkmark$ & $\checkmark$ & \\

AMDCVRPLTW & OR-Tools
& $\checkmark$ & & & $\checkmark$ & $\checkmark$ & & & $\checkmark$ & $\checkmark$ & \\

AMDOCVRPBTW & OR-Tools
& $\checkmark$ & & & $\checkmark$ & $\checkmark$ & & $\checkmark$ & $\checkmark$ & $\checkmark$ & $\checkmark$ \\

AMDCVRPBLTW & OR-Tools
& $\checkmark$ & & & $\checkmark$ & $\checkmark$ & & $\checkmark$ & $\checkmark$ & $\checkmark$ & \\

AMDOCVRPL & OR-Tools
& $\checkmark$ & & & & $\checkmark$ & & & $\checkmark$ & $\checkmark$ & $\checkmark$ \\

AMDCVRPBTW & OR-Tools
& $\checkmark$ & & & $\checkmark$ & $\checkmark$ & & $\checkmark$ & $\checkmark$ & $\checkmark$ & \\

AMDOCVRPBL & OR-Tools
& $\checkmark$ & & & & $\checkmark$ & & $\checkmark$ & $\checkmark$ & $\checkmark$ & $\checkmark$ \\

AMDOCVRPLTW & OR-Tools
& $\checkmark$ & & & $\checkmark$ & $\checkmark$ & & & $\checkmark$ & $\checkmark$ & $\checkmark$ \\

AMDOCVRPBLTW & OR-Tools
& $\checkmark$ & & & $\checkmark$ & $\checkmark$ & & $\checkmark$ & $\checkmark$ & $\checkmark$ & $\checkmark$ \\

AMDCVRPBP & OR-Tools
& $\checkmark$ & & & & $\checkmark$ & & $\checkmark$ & $\checkmark$ & $\checkmark$ & \\

AMDOCVRPBP & OR-Tools
& $\checkmark$ & & & & $\checkmark$ & & $\checkmark$ & $\checkmark$ & $\checkmark$ & $\checkmark$ \\

AMDCVRPBPL & OR-Tools
& $\checkmark$ & & & & $\checkmark$ & & $\checkmark$ & $\checkmark$ & $\checkmark$ & \\

AMDOCVRPBPTW & OR-Tools
& $\checkmark$ & & & $\checkmark$ & $\checkmark$ & & $\checkmark$ & $\checkmark$ & $\checkmark$ & $\checkmark$ \\

AMDCVRPBPLTW & OR-Tools
& $\checkmark$ & & & $\checkmark$ & $\checkmark$ & & $\checkmark$ & $\checkmark$ & $\checkmark$ & \\

AMDCVRPBPTW & OR-Tools
& $\checkmark$ & & & $\checkmark$ & $\checkmark$ & & $\checkmark$ & $\checkmark$ & $\checkmark$ & \\

AMDOCVRPBPL & OR-Tools
& $\checkmark$ & & & & $\checkmark$ & & $\checkmark$ & $\checkmark$ & $\checkmark$ & $\checkmark$ \\

AMDOCVRPBLTW & OR-Tools
& $\checkmark$ & & & $\checkmark$ & $\checkmark$ & & $\checkmark$ & $\checkmark$ & $\checkmark$ & $\checkmark$ \\
\midrule
APDTSP & OR-Tools
& & & & & $\checkmark$ & $\checkmark$ & & $\checkmark$ & & \\

APDCVRP & OR-Tools
& $\checkmark$ & & & & $\checkmark$ & $\checkmark$ & & $\checkmark$ & $\checkmark$ & \\

AOPDCVRP & OR-Tools
& $\checkmark$ & & & & $\checkmark$ & $\checkmark$ & & $\checkmark$ & $\checkmark$ & $\checkmark$ \\
AOP & OR-Tools
& & $\checkmark$ & & & $\checkmark$ & & & & & \\
APCTSP & OR-Tools
& & $\checkmark$ & $\checkmark$ & & $\checkmark$ & & & & & \\
ASPCTSP & ILS
& & $\checkmark$ & $\checkmark$ & & $\checkmark$ & & & & & \\

    \bottomrule[0.5mm]
  \end{tabular}}
\end{table}

\clearpage

\section{Unified Data Representation}
\label{append:udr}
We introduce the UDR $\bm{U}=\{\mathbf{u}_{i}\}_{i=0}^{n}$ to encode each node $i$ into a high-dimensional vector $\mathbf{u}_i = [\Phi(v_i); \bm{\omega}_i; \bm{\xi}_i]$. This representation integrates three essential components: a bidirectional Fréchet representation (BFR) $\Phi(v_i)$ to capture directional geometric relationships, a unified attribute set $\bm{\omega}_i$ to represent local constraints, and a node-type indicator $\bm{\xi}_i$ to define functional roles.

Building upon \citet{zhou2025urs}, we refine the UDR through two key enhancements. First, we replace the position identifier composed of random identifiers and node coordinates with the proposed unified bidirectional Fréchet representation. Second, we decouple the original Pickup/Backhaul-Deliver/Linehaul structure into two distinct categories: Pickup-Delivery (PD) and Backhaul-Linehaul (BL). This separation is motivated by the fundamental structural differences between the two constraints. Specifically, in PD constraints, pickup and delivery nodes are strictly paired with an explicit precedence relationship. In contrast, the BL relationship primarily serves to distinguish node functionalities: linehaul nodes receive goods from the depot, while backhaul nodes return them, with no direct one-to-one matching between them. Coupling these disparate logic flows into a single representation would impede the model's convergence and optimization efficiency.

% \rs{表格补一个pivot}

\begin{table}[h!]
\centering
\caption{Detailed description of each component in UDR.}
\label{tab:udr_attributes}
\resizebox{0.99\textwidth}{!}{
\begin{tabular}{l|l|c|l|l}
\toprule[0.5mm]
Attribute Name & Symbol  & Range & Description & Examples \\
\midrule

% \multicolumn{5}{l}{\textbf{Unified Attribute Set }($\bm{\omega}_{i} \in [-1,1]^6$)} \\
% \midrule
BFR & $\Phi(v)$ & $\mathbb{R}$ & Normalized bidirectional Fréchet distances to $M$ pivots. & All Variants \\
\midrule
Demand & $\delta$  & $[-1,1]$ & Node demand ($+$ for delivery/linehaul, $-$ for pickup/backhaul). & CVRP, PDCVRP \\
Prize & $\epsilon$  & $[0,1]$ & Prize collected for each node. & OP, PCTSP \\
Penalty & $\mu$  & $[0,1]$ & Penalty incurred for each node. & PCTSP \\
Earliest Arrival Time & $e$ & $[0,1]$ & Earliest permissible arrival time. & CVRPTW \\
Latest Arrival Time & $l$  & $[0,1]$ & Latest permissible arrival time (deadline). & CVRPTW \\
Service Time & $s$  & $[0,1]$ & Time required for service at the node. & CVRPTW \\
\midrule
Depot & $-$  & $\{0, 1\}$ & Indicates if the node is a depot. & CVRP \\
Pickup Node & $-$  & $\{0, 1\}$ & Indicates if the node is a pickup location. & PDTSP, APDTSP \\
Delivery Node & $-$  & $\{0, 1\}$ & Indicates if the node is a delivery location. & PDTSP, APDTSP \\
Backhaul Node & $-$  & $\{0, 1\}$ & Indicates if the node is a backhaul location. & CVRPB, ACVRPB \\
Linehaul Node & $-$  & $\{0, 1\}$ & Indicates if the node is a linehaul location. & CVRPB, ACVRPB \\
Node in Sub-routes & $-$ & $\{0, 1\}$ & The solution $\pi$ can have sub-routes. & CVRP \\
Node in Open Route & $-$  & $\{0, 1\}$ & Vehicles need not return to the depot. & OCVRP \\
\bottomrule[0.5mm]
\end{tabular}
}
\end{table}

\section{Model Architecture and Training}
\label{append:model_training}

\subsection{Embedding Layer}
Given an instance $\mathcal{G}$, we project each component of $\mathbf{u}_i = [\Phi(v_i); \bm{\omega}_i; \bm{\xi}_i]$ into a $d$-dimensional latent space using three independent linear transformations. The initial node embedding $\mathbf{h}_{i}^{(0)}$ is computed as:
\begin{equation}
\mathbf{h}_{i}^{(0)}
= \Phi(v_i) W_{\Phi}
+ \bm{\omega}_{i}W_{\bm{\omega}}
+ \bm{\xi}_{i}W_{\bm{\xi}},
\label{eq:input_linear}
\end{equation}
where $W_{\Phi}\in \mathbb{R}^{2M\times d_h}$, $W_{\bm{\omega}}\in \mathbb{R}^{6\times d_h}$, and $W_{\bm{\xi}}\in \mathbb{R}^{7\times d_h}$ are learnable weight matrices. This process yields a set of initial embeddings $H^{(0)}=\{\mathbf{h}_i^{(0)}\}_{i=0}^n$ for all nodes in $\mathcal{G}$.

\subsection{Encoder}
The encoder $\bm{\theta}_{enc}$ consists of $L$ stacked attention layers. It maps initial embeddings $H^{(0)}$ to high-level node representations $H^{(L)}=\{\mathbf{h}_i^{(L)}\}_{i=0}^n$. Each layer contains two primary sub-layers: a Mixed Bias Module (MBM)\citep{zhou2025urs} and a Feed-Forward (FF) network. We integrate both sub-layers using skip-connections and Instance Normalization (IN) \citep{he2016skip, ulyanov2016instanceNorm}. For each layer $\ell$, let ${H}^{(\ell-1)} = \{\mathbf{h}_{i}^{(\ell-1)}\}_{i=0}^{n}$ be the input. We compute the intermediate and final outputs for node $i$ as follows:\begin{equation}\widetilde{\mathbf{h}}_i^{(\ell)} = \mathrm{IN} \left( \mathbf{h}_i^{(\ell-1)} + \mathrm{MBM}\left(\mathbf{h}_i^{(\ell-1)}, H^{(\ell-1)}, \bm{D}, \bm{D}^{\mathrm{T}}, \bm{R}\right) \right),\end{equation}\begin{equation}\mathbf{h}_i^{(\ell)} = \mathrm{IN} \left( \widetilde{\mathbf{h}}_i^{(\ell)} + \mathrm{FF}\left(\widetilde{\mathbf{h}}_i^{(\ell)}\right) \right).\end{equation}Here, $\mathrm{MBM}(\cdot)$ denotes the mixed bias module \citep{zhou2025urs}, detailed in \cref{eq:mixed_bias_module}. $\mathrm{FF}(\cdot)$ represents a position-wise feed-forward network with ReLU activation. In addition, we also apply z-score normalization to $\bm{D}$ and $\bm{D}^{\mathrm{T}}$ throughout the entire process at SPACE to facilitate gradient flow and accelerate convergence~\citep{lecun2002efficient}.

\paragraph{Mixed Bias Module}
MBM \citep{zhou2025urs} integrates three structural bias matrices through decoupled attention mechanisms: (1) an outgoing distance matrix $\bm{D}$; (2) an incoming distance matrix $\bm{D}^{\mathrm{T}}$; and (3) an optional relation matrix $\bm{R}$. In this module, the relation elements $r_{ij} \in \bm{R}$ are set to 0 if a predefined relationship exists (e.g., pickup–delivery pairing) and 1 otherwise. This encoding ensures that smaller values represent higher structural affinities. For the $\ell$-th layer, the MBM computes three parallel attention outputs:
\begin{equation}
\bar{\mathbf{h}}_{i}^{(0)} = \mathrm{Attention}(\mathbf{h}_i^{(\ell-1)}, H^{(\ell-1)}, f\left(\alpha,N,\bm{D}_{i}\right)),
\end{equation}
\begin{equation}
\bar{\mathbf{h}}_{i}^{(1)} = \mathrm{Attention}(\mathbf{h}_i^{(\ell-1)}, H^{(\ell-1)}, f\left(\alpha,N,\bm{D}_{i}^{\mathrm{T}}\right)), \\
\end{equation}
\begin{equation}
\bar{\mathbf{h}}_{i}^{(2)} = 
    \begin{cases}
    \mathrm{Attention}(\mathbf{h}_i^{(\ell-1)}, H^{(\ell-1)}, f\left(\alpha,\bm{R}_{i}\right))   & \text{if} \ \bm{R} \neq \emptyset \\
    \bm{0}  &\text{otherwise}, 
    \end{cases},
\label{eq:mdm_attn}
\end{equation}

The intermediate results are then aggregated via concatenation and a linear projection:
\begin{equation}
\hat{\mathbf{h}}_i^{(\ell)} = \left[\bar{\mathbf{h}}_{i}^{(0)},\bar{\mathbf{h}}_{i}^{(1)},\bar{\mathbf{h}}_{i}^{(2)}\right]W^{O},
    \label{eq:mixed_bias_module}
\end{equation}
where $W^{O}\in \mathbb{R}^{(3\times d_h)\times d_h}$ is a learnable weight matrix.

For $\mathrm{Attention}(\cdot)$, MBM employs the Adaptation Attention Free Module (AAFM) \citep{zhou2024icam}. This module enhances the model's sensitivity to diverse geometric patterns via an adaptive function $f(\alpha, N, d_{ij})$ conditioned on a bias weight $\alpha$. The bias weight $\alpha$ is generated by a lightweight bias network, denoted as $\mathrm{BIAS}(\bm{\lambda})$, which takes the multi-hot representation $\bm{\lambda}$ of a sample as input. Formally, the computation is defined as:
\begin{equation}
\mathrm{BIAS}(\bm{\lambda}) = \max\left(0, \left(\bm{\lambda} W_1+\mathbf{b}_1\right)W_2 +\mathbf{b}_2\right),
\label{eq:bias}
\end{equation}
where $W_1 \in \mathbb{R}^{|\bm{\lambda}|\times d_h}$, $W_2 \in \mathbb{R}^{d_h\times 1}$, $\mathbf{b}_1\in \mathbb{R}^{d_h}$, and $\mathbf{b}_2\in \mathbb{R}^{1}$ represent the learnable parameters. 
Consistent with \citet{zhou2025urs}, we omit the scale factor $N$ in the adaptation function for $\bm{R}$, as the relations $r_{ij}$ are inherently scale-independent. In cases where $\bm{R}$ is absent, we initialize $\bar{\mathbf{h}}_{i}^{(2)}$ as a zero vector to maintain structural consistency.

\paragraph{Adaptation Attention Free Module}
\label{app:AAFM}
Following \citet{zhou2024icam}, we implement $\mathrm{Attention}(\cdot)$ using an adaptation attention free module (AAFM) to enhance geometric pattern recognition for routing problems. Given the input $X$, AAFM first transforms it into $Q$, $K$, and $V$ through corresponding linear projection operations:
\begin{equation}
Q = XW^{Q}, \quad K = XW^{K}, \quad V = XW^{V},
\label{eq:qkv}
\end{equation}
where $W^{Q}$, $W^{K}$, and $W^{V}$ are learnable matrices. The AAFM computation is then expressed as:
\begin{equation}
\mathrm{Attention}(Q,K,V,A) = \sigma (Q) \odot \frac{\exp (A) (\exp (K) \odot V)}{\exp (A)\exp (K)},
\label{eq:AAFM}
\end{equation}
where $\sigma$ denotes the sigmoid function, $\odot$ represents the element-wise product, and $A = \{a_{ij}\}$ denotes the pair-wise adaptation bias. For distance matrices, the corresponding adaptation bias $f\left(\alpha,N,d_{ij}\right)=-\alpha\cdot\log_2{N}\cdot d_{i,j}$. Here, $N=|\mathcal{V}|$ denotes the total number of nodes (i.e., problem size), and $\alpha$ is generated via a lightweight network $\mathrm{BIAS}(\bm{\lambda})$ in our paper (see \cref{eq:bias}). Notably, for the relation matrix, we remove the scale $N$ in the calculation of adaptation bias because the relation is scale-independent. 

Compared to multi-head attention (MHA)~\citep{vaswani2017attention}, AAFM enables the model to explicitly capture instance-specific knowledge by updating pair-wise adaptation biases while exhibiting lower computational overhead.

\subsection{Decoder}
\label{app:dora_implementation}
The decoder constructs a distribution over the next node at each autoregressive step. Since WDAD and AAFM have been defined in Section~\ref{subsec:dora_adaptation}, Appendix~\ref{app:dora_details}, and \cref{eq:AAFM}, we only specify how these modules are instantiated inside the decoder. Before decoding, WDAD generates five effective projection matrices conditioned on the multi-hot problem representation $\bm{\lambda}$: two query projections for the first and last selected nodes, $W_{\mathrm{first}}^{Q,(\bm{\lambda})}$ and $W_{\mathrm{last}}^{Q,(\bm{\lambda})}$, the key and value projections $W^{K,(\bm{\lambda})}$ and $W^{V,(\bm{\lambda})}$, and the state projection $W^{C,(\bm{\lambda})}$. Given the encoded node embeddings $H^{(L)}=\{\mathbf{h}_i^{(L)}\}_{i=0}^{n}$, the candidate keys and values are cached before decoding:
\begin{equation}
    K = H^{(L)} W^{K,(\bm{\lambda})}, \qquad
    V = H^{(L)} W^{V,(\bm{\lambda})}.
\end{equation}

At step $t$, the context query is formed from the first selected node, the most recently selected node, and the optional scalar problem state $C_t$:
\begin{equation}
    \mathbf{q}_{t}
    = \mathbf{h}_{\pi_1}^{(L)} W_{\mathrm{first}}^{Q,(\bm{\lambda})}
    + \mathbf{h}_{\pi_{t-1}}^{(L)} W_{\mathrm{last}}^{Q,(\bm{\lambda})}
    + C_t W^{C,(\bm{\lambda})},
\end{equation}
where $C_t=0$ for variants that do not require an explicit state signal. The state variables used in SPACE are summarized in the next subsection. Let $\tilde{\bm{d}}_t$ denote the z-score normalized outgoing distance vector from the current node to all candidate nodes, and let $\mathcal{M}_t\in\{0,-\infty\}^{n+1}$ denote the feasibility mask introduced in Section~\ref{sec:preliminary}. Similar to the encoder, the decoder uses two independent instances of the bias network in \cref{eq:bias} to produce distance-bias coefficients $\alpha_{\mathrm{attn}}$ and $\alpha_{\mathrm{com}}$. The first coefficient defines the AAFM bias for decoder readout:
\begin{equation}
    \mathbf{a}^{\mathrm{dec}}_t
    = -\alpha_{\mathrm{attn}}\log_2 N \cdot \tilde{\bm{d}}_t + \mathcal{M}_t,
    \qquad
    \mathbf{h}'_t
    =
    \mathrm{Attention}(\mathbf{q}_t,K,V,\mathbf{a}^{\mathrm{dec}}_t),
\end{equation}
so infeasible candidates are suppressed within the same four-argument AAFM form used in \cref{eq:AAFM}. The updated context $\mathbf{h}'_t$ is then matched against the original encoded node embeddings through a single-head compatibility layer. A second distance bias is added before clipping:
\begin{equation}
    \mathbf{s}_t
    =
    \frac{\mathbf{h}'_t (H^{(L)})^{\top}}{\sqrt{d_h}}
    - \alpha_{\mathrm{com}}\log_2 N \cdot \tilde{\bm{d}}_t.
\end{equation}
Finally, the transition probability is obtained by applying tanh clipping and the feasibility mask:
\begin{equation}
    p_{\bm{\theta}}(\pi_t=i \mid \pi_{1:t-1},\mathcal{G})
    =
    \mathrm{Softmax}\left(
    \zeta \cdot \tanh(\mathbf{s}_t) + \mathcal{M}_t
    \right)_i,
\end{equation}
where $\zeta$ is the logit clipping constant. This decoder preserves the standard autoregressive selection pipeline while allowing the effective projections and distance biases to adapt to the active routing attributes.

\subsection{Problem State \texorpdfstring{$C_{t}$}{Ct}}
\label{append:problem_state}

In this section, we detail the problems that require additional state signals. For the problem state $C_{t}$, we follow the settings in the URS~\citep{zhou2025urs}, and the specific operations are as follows:

\begin{itemize}
    \item For all problems with the capacity constraint, we explicitly impose \textbf{only remaining load}, while diverse additional constraints (e.g., time windows, duration limit, and backhaul) are enforced implicitly as $\mathcal{M}_t$ masks infeasible candidate nodes to guarantee valid solutions, instead of being fully enumerated in the input of the decoder.
    \item For (A)OP, we keep track of the remaining maximum length at time $t$, and $C_t$ represents the remaining routing length that can be moved, and we normalize it to $[0,1]$ by dividing by the maximum length (i.e., 4.0 in OP100 and 1.0 in AOP100). The setting is the same as AM~\citep{kool2019attention}.
    \item For (A)PCTSP and (A)SPCTSP, $C_t$ is the remaining prize to collect, and we do not provide any information about the penalties, as this is irrelevant for the remaining decisions, following~\citet{kool2019attention}.
\end{itemize}

\subsection{Training}
\label{append:training}
To enable the model to better capture features shared across different routing problems, we adopt a per-iteration random problem-type sampling strategy: at each training step, we randomly select one problem from the set of 12 training problems, and randomly generate a mini-batch (128 in our experiment) of instances of that problem for training.

Following~\citet{kwon2020pomo}, we use $n$ trajectories with distinct starting nodes in training. SPACE is trained by the REINFORCE~\citep{williams1992reinforce} algorithm with a shared baseline~\citep{kwon2020pomo}:
\begin{equation}
    \nabla_\theta \mathcal{L} (\bm{\theta}) = \mathbb{E}_{p(\bm{\pi}| \mathcal{G},{\bm{\theta}})}[ 
    (f(\bm{\pi} | \mathcal{G})-\frac{1}{n} \sum_{i=1}^n f(\bm{\pi}^{i} | \mathcal{G}))
    \nabla_{\bm{\theta}} \log p_{\bm{\theta}}\left(\bm{\pi} | \mathcal{G}\right) ],
    \label{eq:mean_loss}
\end{equation}
\begin{equation}
    p_{\bm{\theta}}(\bm{\pi}\mid \mathcal{G})=\prod_{t=2}^{n}
    p_{\bm{\theta}}(\pi_t\mid \mathcal{G},\pi_{1:t-1}),
    \label{eq:construct_solution}
\end{equation}
where the objective function $f(\bm{\pi} | \mathcal{G})$ represents the total reward (e.g., the negative value of tour length) of instance $\mathcal{G}$ given a specific solution $\bm{\pi}$.

\subsection{Model and Training Hyperparameters}
\label{append:model_setting}
Detailed information on the hyperparameter settings can be found in \cref{table:hyperparameters}. The training of SPACE is conducted on a single NVIDIA GeForce RTX 4090 GPU (24GB of memory). The training procedure requires only about $110$ hours ($\approx 5$ days) in total (500 epochs at approximately $ 13$ minutes each) to obtain a unified parameter set capable of addressing a broad class of routing problems. 
\begin{table}[h!]
\centering
\caption{The hyperparameter settings of SPACE.}
\label{table:hyperparameters}
\resizebox{0.5\textwidth}{!}{
\begin{tabular}{l c }
\toprule[0.5mm]
Hyperparameter& Value \\
\midrule
 Optimizer              & AdamW       \\
 Clipping parameter $\zeta$            & 50 \\
 Initial learning rate             & $10^{-4}$ \\
 Epochs of learning rate decay              & [451] \\
 Factor of learning rate decay             & 0.1 \\

 Weight decay                &$10^{-6}$      \\

 The number of encoder layers &$12$           \\
 Embedding dimension $d_h$       & $128$        \\
 Feed forward dimension    & $512$        \\
 Number of Pivots $M$      & $8$          \\
 WDAD Rank $r$             & $32$         \\
 WDAD Heads                & $3$          \\

 Training scale       &$100$     \\
 Batches of each epoch      & $2,000$        \\
 Batch size      & $128$           \\
 Training Epochs         &$500$       \\

\bottomrule[0.5mm]
\end{tabular}
}
\end{table}

\clearpage

\section{Detailed Experimental Results}
\label{sec:detailed_experimental_results}

\subsection{Consistent Problem Settings}

This section reports the complete per-problem comparison under the same problem settings used in the main experiments. We group the evaluated variants into single-depot and multi-depot families in \cref{tab:detailed_singledepot_results,tab:detailed_multidepot_results}. For each variant, we report the optimality gap and inference time under both symmetric and asymmetric distance matrices, together with the average gap across the two settings. A dash indicates that the method does not support the corresponding setting. For readability, we list the seen (a)symmetric variants corresponding to unseen variants together if necessary.

For single-depot variants, \cref{tab:detailed_singledepot_results} shows that SPACE achieves the lowest average gap on most evaluated problems. Compared with URS, SPACE improves 27 of 31 average gaps, with gains mainly from the asymmetric side while maintaining symmetric performance. The improvement is especially clear in compositional variants that include open routes, backhauls, time windows, and duration limits. URS remains slightly better on a few simpler or prize-collecting cases, indicating that the main benefit of SPACE is not isolated symmetric specialization but robust transfer across symmetric and asymmetric realizations of the same constraint structure.

For multi-depot variants, \cref{tab:detailed_multidepot_results} shows a stronger asymmetry-related trend. SPACE obtains lower asymmetric gaps than URS on all 24 multi-depot variants and the lowest average gap on 22 of them. SPACE consistently reduces the degradation caused by asymmetric distances. These results support the role of the proposed unified representation in maintaining stable generalization when both the depot structure and the underlying distance geometry change.

\begin{table*}[htbp]
  \centering
  \caption{Detailed results on single-depot VRPs.}
  \resizebox{\textwidth}{!}{
    \begin{tabular}{l|ccc|ccc|ccc}
    \toprule[0.5mm]
          & Symmetric & Asymmetric & Average & Symmetric & Asymmetric & Average & Symmetric & Asymmetric & Average \\
    Method & Gap (Time) & Gap (Time) & Gap   & Gap (Time) & Gap (Time) & Gap   & Gap (Time) & Gap (Time) & Gap \\
    \midrule
    \midrule
    Problem & \multicolumn{3}{c|}{CVRPBL} & \multicolumn{3}{c|}{OCVRPBL} & \multicolumn{3}{c}{OCVRPBLTW} \\
    \midrule
    Oracle & 0.00\% (3.5h) & 0.00\% (3.5h) & 0.00\% & 0.00\% (3.5h) & 0.00\% (3.5h) & 0.00\% & 0.00\% (3.5h) & 0.00\% (3.5h) & 0.00\% \\
    MVMoE & 1.35\% (12s) & $-$ & $-$ & 7.12\% (12s) & $-$ & $-$ & 10.01\% (15s) & $-$ & $-$ \\
    MTPOMO & 1.79\% (6s) & $-$ & $-$ & 7.34\% (7s) & $-$ & $-$ & 10.50\% (7s) & $-$ & $-$ \\
    ReLD-MTL & 1.01\% (7s) & $-$ & $-$ & 5.41\% (7s) & $-$ & $-$ & 9.22\% (8s) & $-$ & $-$ \\
    URS   & 1.65\% (6s) & 6.18\% (2.4m) & 3.92\% & 9.47\% (7s) & 17.23\% (2.7m) & 13.35\% & 13.72\% (8s) & 14.32\% (2.5m) & 14.02\% \\
    SPACE & 1.67\% (6s) & -14.27\% (1.8m) & \greybg{-6.30\%} & 7.36\% (7s) & -1.94\% (1.8m) & \greybg{2.71\%} & 11.63\% (8s) & 5.52\% (2.3m) & \greybg{8.57\%} \\
    \midrule
    \midrule
    Problem & \multicolumn{3}{c|}{CVRPBLTW} & \multicolumn{3}{c|}{CVRPBP} & \multicolumn{3}{c}{CVRPBPL} \\
    \midrule
    Oracle & 0.00\% (3.5h) & 0.00\% (3.5h) & 0.00\% & 0.00\% (6.6m) & 0.00\% (6.6m) & 0.00\% & 0.00\% (6.6m) & 0.00\% (6.6m) & 0.00\% \\
    MVMoE & 7.33\% (16s) & $-$ & $-$ & 13.95\% (14s) & $-$ & $-$ & 13.38\% (14s) & $-$ & $-$ \\
    MTPOMO & 7.75\% (8s) & $-$ & $-$ & 13.71\% (7s) & $-$ & $-$ & 13.44\% (8s) & $-$ & $-$ \\
    ReLD-MTL & 6.94\% (9s) & $-$ & $-$ & 13.57\% (9s) & $-$ & $-$ & 12.96\% (10s) & $-$ & $-$ \\
    URS   & 9.18\% (8s) & 10.18\% (2.4m) & 9.68\% & 12.95\% (7s) & 24.20\% (2.4m) & 18.58\% & 12.65\% (8s) & 24.03\% (2.6m) & 18.34\% \\
    SPACE & 8.29\% (8s) & -0.37\% (2.3m) & \greybg{3.96\%} & 13.49\% (8s) & 7.93\% (2.0m) & \greybg{10.71\%} & 13.13\% (8s) & 7.82\% (2.2m) & \greybg{10.47\%} \\
    \midrule
    \midrule
    Problem & \multicolumn{3}{c|}{CVRPBPLTW} & \multicolumn{3}{c|}{CVRPB} & \multicolumn{3}{c}{OCVRP} \\
    \midrule
    Oracle & 0.00\% (6.6m) & 0.00\% (6.6m) & 0.00\% & 0.00\% (20.8m) & 0.00\% (6.6m) & 0.00\% & 0.00\% (5.3m) & 0.00\% (6.6m) & 0.00\% \\
    MVMoE & 16.66\% (16s) & $-$ & $-$ & 1.28\% (11s) & $-$ & $-$ & 3.14\% (13s) & $-$ & $-$ \\
    MTPOMO & 17.15\% (10s) & $-$ & $-$ & 1.67\% (5s) & $-$ & $-$ & 3.46\% (6s) & $-$ & $-$ \\
    ReLD-MTL & 15.95\% (12s) & $-$ & $-$ & 0.90\% (6s) & $-$ & $-$ & 2.32\% (7s) & $-$ & $-$ \\
    URS   & 18.33\% (10s) & 13.82\% (2.8m) & 16.08\% & 1.46\% (6s) & 7.10\% (2.1m) & 4.28\% & 3.24\% (6s) & 37.63\% (2.9m) & 20.44\% \\
    SPACE & 17.54\% (11s) & 4.16\% (2.9m) & \greybg{10.85\%} & 1.45\% (6s) & -14.05\% (1.7m) & \greybg{-6.30\%} & 3.31\% (7s) & 11.18\% (1.8m) & \greybg{7.24\%} \\
    \midrule
    \midrule
    Problem & \multicolumn{3}{c|}{CVRPBPTW} & \multicolumn{3}{c|}{CVRPBTW} & \multicolumn{3}{c}{CVRPL} \\
    \midrule
    Oracle & 0.00\% (6.6m) & 0.00\% (6.6m) & 0.00\% & 0.00\% (3.5h) & 0.00\% (3.5h) & 0.00\% & 0.00\% (16m) & 0.00\% (16m) & 0.00\% \\
    MVMoE & 16.86\% (15s) & $-$ & $-$ & 7.08\% (15s) & $-$ & $-$ & 0.26\% (12s) & $-$ & $-$ \\
    MTPOMO & 17.22\% (9s) & $-$ & $-$ & 7.41\% (7s) & $-$ & $-$ & 0.48\% (6s) & $-$ & $-$ \\
    ReLD-MTL & 16.09\% (11s) & $-$ & $-$ & 6.74\% (8s) & $-$ & $-$ & 0.02\% (8s) & $-$ & $-$ \\
    URS   & 18.15\% (9s) & 14.52\% (2.6m) & 16.33\% & 8.94\% (8s) & 9.95\% (2.2m) & 9.45\% & 0.43\% (7s) & -2.24\% (1.9m) & \greybg{-0.91\%} \\
    SPACE & 17.57\% (10s) & 4.73\% (2.7m) & \greybg{11.15\%} & 7.91\% (8s) & -0.52\% (2.1m) & \greybg{3.69\%} & 0.45\% (7s) & -1.88\% (2.0m) & -0.72\% \\
    \midrule
    \midrule
    Problem & \multicolumn{3}{c|}{CVRPLTW} & \multicolumn{3}{c|}{OCVRPB} & \multicolumn{3}{c}{OCVRPBPTW} \\
    \midrule
    Oracle & 0.00\% (3.5h) & 0.00\% (3.5h) & 0.00\% & 0.00\% (3.5h) & 0.00\% (3.5h) & 0.00\% & 0.00\% (6.6m) & 0.00\% (6.6m) & 0.00\% \\
    MVMoE & 1.47\% (16s) & $-$ & $-$ & 7.09\% (12s) & $-$ & $-$ & 9.42\% (15s) & $-$ & $-$ \\
    MTPOMO & 1.92\% (9s) & $-$ & $-$ & 7.34\% (5s) & $-$ & $-$ & 9.94\% (9s) & $-$ & $-$ \\
    ReLD-MTL & 1.17\% (11s) & $-$ & $-$ & 5.36\% (6s) & $-$ & $-$ & 8.43\% (11s) & $-$ & $-$ \\
    URS   & 2.67\% (9s) & 12.22\% (2.7m) & 7.45\% & 9.35\% (6s) & 18.05\% (2.5m) & 13.70\% & 10.05\% (8s) & 17.45\% (2.5m) & 13.75\% \\
    SPACE & 1.82\% (9s) & 2.26\% (2.6m) & \greybg{2.04\%} & 7.33\% (6s) & -1.34\% (1.7m) & \greybg{2.99\%} & 9.62\% (9s) & 9.65\% (2.7m) & \greybg{9.64\%} \\
    \midrule
    \midrule
    Problem & \multicolumn{3}{c|}{OCVRPBP} & \multicolumn{3}{c|}{OCVRPBPL} & \multicolumn{3}{c}{OCVRPBPLTW} \\
    \midrule
    Oracle & 0.00\% (6.6m) & 0.00\% (6.6m) & 0.00\% & 0.00\% (6.6m) & 0.00\% (6.6m) & 0.00\% & 0.00\% (6.6m) & 0.00\% (6.6m) & 0.00\% \\
    MVMoE & 18.27\% (14s) & $-$ & $-$ & 17.75\% (14s) & $-$ & $-$ & 8.87\% (16s) & $-$ & $-$ \\
    MTPOMO & 17.90\% (8s) & $-$ & $-$ & 17.31\% (8s) & $-$ & $-$ & 9.47\% (10s) & $-$ & $-$ \\
    ReLD-MTL & 15.81\% (9s) & $-$ & $-$ & 15.21\% (10s) & $-$ & $-$ & 8.06\% (11s) & $-$ & $-$ \\
    URS   & 16.96\% (7s) & 32.58\% (2.5m) & 24.77\% & 16.37\% (8s) & 32.50\% (2.7m) & 24.44\% & 9.51\% (9s) & 13.82\% (2.8m) & 11.66\% \\
    SPACE & 16.85\% (8s) & 17.40\% (2.1m) & \greybg{17.12\%} & 16.31\% (8s) & 17.00\% (2.3m) & \greybg{16.66\%} & 9.23\% (10s) & 8.57\% (2.9m) & \greybg{8.90\%} \\
    \midrule
    \midrule
    Problem & \multicolumn{3}{c|}{OCVRPBTW} & \multicolumn{3}{c|}{OCVRPL} & \multicolumn{3}{c}{OCVRPLTW} \\
    \midrule
    Oracle & 0.00\% (3.5h) & 0.00\% (6.6m) & 0.00\% & 0.00\% (3.5h) & 0.00\% (3.5h) & 0.00\% & 0.00\% (3.5h) & 0.00\% (3.5h) & 0.00\% \\
    MVMoE & 9.95\% (15s) & $-$ & $-$ & 3.15\% (13s) & $-$ & $-$ & 3.90\% (15s) & $-$ & $-$ \\
    MTPOMO & 10.45\% (6s) & $-$ & $-$ & 3.44\% (6s) & $-$ & $-$ & 4.37\% (8s) & $-$ & $-$ \\
    ReLD-MTL & 9.29\% (8s) & $-$ & $-$ & 2.31\% (8s) & $-$ & $-$ & 3.16\% (9s) & $-$ & $-$ \\
    URS   & 13.77\% (7s) & 14.45\% (2.3m) & 14.11\% & 3.22\% (7s) & 37.36\% (3.1m) & 20.29\% & 5.12\% (9s) & 16.35\% (2.8m) & 10.74\% \\
    SPACE & 11.60\% (7s) & 5.48\% (2.2m) & \greybg{8.54\%} & 3.29\% (7s) & 11.02\% (2.0m) & \greybg{7.15\%} & 4.26\% (9s) & 6.99\% (2.6m) & \greybg{5.63\%} \\
    \midrule
    \midrule
    Problem & \multicolumn{3}{c|}{TSP} & \multicolumn{3}{c|}{OP} & \multicolumn{3}{c}{PCTSP} \\
    \midrule
    Oracle & 0.00\% (6m) & 0.00\% (2m) & 0.00\% & 0.00\% (1.5m) & 0.00\% (6.6m) & 0.00\% & 0.00\% (1.2h) & 0.00\% (6.6m) & 0.00\% \\
    GOAL-MTL & $-$ & 1.77\% (1m) & $-$ & 1.20\% (38s) & $-$ & $-$ & $-$ & $-$ & $-$ \\
    URS   & 0.57\% (6s) & 2.26\% (1.1m) & \greybg{1.42\%} & 0.45\% (4s) & -5.47\% (1.2m) & -2.51\% & 1.06\% (5s) & 43.79\% (1.2m) & \greybg{22.43\%} \\
    SPACE & 0.64\% (5s) & 2.89\% (1.4m) & 1.77\% & 0.44\% (5s) & -12.13\% (1.4m) & \greybg{-5.85\%} & 0.88\% (5s) & 73.95\% (1.5m) & 37.42\% \\
    \midrule
    \midrule
    Problem & \multicolumn{3}{c|}{SPCTSP} & \multicolumn{3}{c|}{PDCVRP} & \multicolumn{3}{c}{OPDCVRP} \\
    \midrule
    Oracle & $-$ & $-$ & $-$ & 0.00\% (6.6m) & 0.00\% (6.6m) & 0.00\% & 0.00\% (6.6m) & 0.00\% (6.6m) & 0.00\% \\
    URS   & -2.37\% (5s) & 0.00\% (1.2m) & \greybg{-1.19\%} & -1.47\% (4.3s) & 7.03\% (1.2m) & 2.78\% & 4.93\% (4.3s) & 11.13\% (1.2m) & 8.03\% \\
    SPACE & -2.60\% (5s) & 20.05\% (1.5m) & 8.73\% & 0.38\% (5.1s) & 1.26\% (1.3m) & \greybg{0.82\%} & 7.90\% (5.4s) & 4.60\% (1.4m) & \greybg{6.25\%} \\
    \midrule
    \midrule
    Problem & \multicolumn{3}{c|}{CVRP} & \multicolumn{3}{c|}{CVRPTW} & \multicolumn{3}{c}{OCVRPTW} \\
    \midrule
    Oracle & 0.00\% (9.1m) & 0.00\% (6.6m) & 0.00\% & 0.00\% (19.6m) & 0.00\% (6.6m) & 0.00\% & 0.00\% (20.8m) & 0.00\% (6.6m) & 0.00\% \\
    MVMoE & 1.65\% (12s) & $-$ & $-$ & 4.90\% (15s) & $-$ & $-$ & 3.85\% (15s) & $-$ & $-$ \\
    MTPOMO & 1.85\% (6s) & $-$ & $-$ & 5.31\% (8s) & $-$ & $-$ & 4.41\% (7s) & $-$ & $-$ \\
    ReLD-MTL & 1.42\% (8s) & $-$ & $-$ & 4.56\% ( 10s) & $-$ & $-$ & 3.10\% (9s) & $-$ & $-$ \\
    GOAL-MTL & 4.22\% (48s) & $-$ & $-$ & 4.66\% (42s) & $-$ & $-$ & $-$ & $-$ & $-$ \\
    URS   & 1.81\% (6s) & 3.06\% (1.4m) & 2.44\% & 6.13\% (8s) & 11.28\% (2.5m) & 8.71\% & 5.07\% (8s) & 17.47\% (2.6m) & 11.27\% \\
    SPACE & 1.83\% (7s) & 2.67\% (1.8m) & \greybg{2.25\%} & 5.36\% (9s) & 1.54\% (2.4m) & \greybg{3.45\%} & 4.20\% (8s) & 7.97\% (2.4m) & \greybg{6.08\%} \\
    \midrule
    \midrule
    Problem & \multicolumn{3}{c|}{PDTSP} & \multicolumn{3}{c|}{$-$} & \multicolumn{3}{c}{$-$} \\
    \midrule
    Oracle & 0.00\% (9.8m) & 0.00\% (6.6m) & 0.00\% & $-$ & $-$ & $-$ & $-$ & $-$ & $-$ \\
    URS   & 4.98\% (4s) & 6.21\% (1.0m) & 5.60\% & $-$ & $-$ & $-$ & $-$ & $-$ & $-$ \\
    SPACE & 4.71\% (5s) & -7.46\% (1.2m) & \greybg{-1.38\%} & $-$ & $-$ & $-$ & $-$ & $-$ & $-$ \\
    \bottomrule[0.5mm]
    \end{tabular}%
    }
  \label{tab:detailed_singledepot_results}%
\end{table*}%

\begin{table*}[htbp]
  \centering
  \caption{Detailed results on multi-depot VRPs.}
  \resizebox{\textwidth}{!}{
    \begin{tabular}{l|ccc|ccc|ccc}
    \toprule[0.5mm]
          & Symmetric & Asymmetric & Average & Symmetric & Asymmetric & Average & Symmetric & Asymmetric & Average \\
    Method & Gap (Time) & Gap (Time) & Gap   & Gap (Time) & Gap (Time) & Gap   & Gap (Time) & Gap (Time) & Gap \\
    \midrule
    \midrule
    Problem & \multicolumn{3}{c|}{MDCVRP} & \multicolumn{3}{c|}{MDCVRPB} & \multicolumn{3}{c}{MDCVRPBL} \\
    \midrule
    Oracle & 0.00\% (6.6m) & 0.00\% (6.6m) & 0.00\% & 0.00\% (6.6m) & 0.00\% (6.6m) & 0.00\% & 0.00\% (6.6m) & 0.00\% (6.6m) & 0.00\% \\
    MVMoE & 32.39\% (1.8m) & $-$ & $-$ & 16.70\% (1.5m) & $-$ & $-$ & 16.50\% (1.5m) & $-$ & $-$ \\
    MTPOMO & 27.25\% (1.3m) & $-$ & $-$ & 13.41\% (1.0m) & $-$ & $-$ & 11.92\% (1.1m) & $-$ & $-$ \\
    ReLD-MTL & 18.79\% (1.5m) & $-$ & $-$ & 7.71\% (1.1m) & $-$ & $-$ & 6.87\% (1.2m) & $-$ & $-$ \\
    URS   & 12.88\% (56s) & 4.67\% (15m) & \greybg{8.77\%} & 3.94\% (45s) & 8.52\% (15m) & 6.23\% & 3.12\% (48s) & 9.66\% (15m) & 6.39\% \\
    SPACE & 16.67\% (1.0m) & 3.93\% (15m) & 10.30\% & 5.49\% (50s) & -9.40\% (12m) & \greybg{-1.95\%} & 5.00\% (53s) & -8.50\% (12m) & \greybg{-1.75\%} \\
    \midrule
    \midrule
    Problem & \multicolumn{3}{c|}{MDCVRPBLTW} & \multicolumn{3}{c|}{MDCVRPBP} & \multicolumn{3}{c}{MDCVRPLTW} \\
    \midrule
    Oracle & 0.00\% (6.6m) & 0.00\% (6.6m) & 0.00\% & 0.00\% (6.6m) & 0.00\% (6.6m) & 0.00\% & 0.00\% (6.6m) & 0.00\% (6.6m) & 0.00\% \\
    MVMoE & 50.63\% (1.8m) & $-$ & $-$ & 62.05\% (2.0m) & $-$ & $-$ & 41.51\% (2.1m) & $-$ & $-$ \\
    MTPOMO & 42.54\% (1.3m) & $-$ & $-$ & 52.27\% (1.4m) & $-$ & $-$ & 34.21\% (1.5m) & $-$ & $-$ \\
    ReLD-MTL & 35.33\% (1.5m) & $-$ & $-$ & 38.98\% (1.6m) & $-$ & $-$ & 26.44\% (1.8m) & $-$ & $-$ \\
    URS   & 36.76\% (1.0m) & 13.18\% (18m) & 24.97\% & 31.19\% (1.0m) & 27.10\% (20m) & 29.14\% & 32.92\% (1.3m) & 15.26\% (22m) & 24.09\% \\
    SPACE & 35.97\% (1.1m) & 2.70\% (18m) & \greybg{19.33\%} & 34.16\% (1.2m) & 9.69\% (17m) & \greybg{21.92\%} & 27.29\% (1.3m) & 4.64\% (21m) & \greybg{15.97\%} \\
    \midrule
    \midrule
    Problem & \multicolumn{3}{c|}{MDOCVRPBL} & \multicolumn{3}{c|}{MDCVRPBPL} & \multicolumn{3}{c}{MDCVRPBPLTW} \\
    \midrule
    Oracle & 0.00\% (6.6m) & 0.00\% (6.6m) & 0.00\% & 0.00\% (6.6m) & 0.00\% (6.6m) & 0.00\% & 0.00\% (6.6m) & 0.00\% (6.6m) & 0.00\% \\
    MVMoE & 29.70\% (1.6m) & $-$ & $-$ & 62.41\% (2.1m) & $-$ & $-$ & 56.55\% (2.3m) & $-$ & $-$ \\
    MTPOMO & 24.25\% (1.1m) & $-$ & $-$ & 52.77\% (1.4m) & $-$ & $-$ & 47.00\% (1.7m) & $-$ & $-$ \\
    ReLD-MTL & 17.33\% (1.2m) & $-$ & $-$ & 38.79\% (1.6m) & $-$ & $-$ & 40.18\% (2.0m) & $-$ & $-$ \\
    URS   & 13.29\% (49s) & 25.18\% (17m) & 19.24\% & 30.76\% (1.1m) & 25.81\% (20m) & 28.28\% & 37.99\% (1.4m) & 17.16\% (23m) & 27.58\% \\
    SPACE & 12.95\% (49s) & 10.31\% (13m) & \greybg{11.63\%} & 34.20\% (1.2m) & 8.66\% (17m) & \greybg{21.43\%} & 39.38\% (1.5m) & 7.51\% (23m) & \greybg{23.44\%} \\
    \midrule
    \midrule
    Problem & \multicolumn{3}{c|}{MDCVRPBPTW} & \multicolumn{3}{c|}{MDCVRPBTW} & \multicolumn{3}{c}{MDCVRPL} \\
    \midrule
    Oracle & 0.00\% (6.6m) & 0.00\% (6.6m) & 0.00\% & 0.00\% (6.6m) & 0.00\% (6.6m) & 0.00\% & 0.00\% (6.6m) & 0.00\% (6.6m) & 0.00\% \\
    MVMoE & 56.76\% (2.3m) & $-$ & $-$ & 51.68\% (1.7m) & $-$ & $-$ & 32.23\% (1.9m) & $-$ & $-$ \\
    MTPOMO & 46.92\% (1.6m) & $-$ & $-$ & 43.01\% (1.2m) & $-$ & $-$ & 27.94\% (1.3m) & $-$ & $-$ \\
    ReLD-MTL & 40.02\% (2.0m) & $-$ & $-$ & 35.21\% (1.5m) & $-$ & $-$ & 19.29\% (1.5m) & $-$ & $-$ \\
    URS   & 38.11\% (1.3m) & 17.94\% (23m) & 28.03\% & 36.51\% (1.0m) & 13.29\% (18m) & 24.90\% & 13.05\% (58s) & 3.94\% (15m) & \greybg{8.49\%} \\
    SPACE & 39.88\% (1.4m) & 8.22\% (21m) & \greybg{24.05\%} & 36.79\% (1.1m) & 2.69\% (17m) & \greybg{19.74\%} & 16.76\% (1.1m) & 3.46\% (15m) & 10.11\% \\
    \midrule
    \midrule
    Problem & \multicolumn{3}{c|}{MDOCVRPBLTW} & \multicolumn{3}{c|}{MDOCVRPBP} & \multicolumn{3}{c}{MDOCVRPBPL} \\
    \midrule
    Oracle & 0.00\% (6.6m) & 0.00\% (6.6m) & 0.00\% & 0.00\% (6.6m) & 0.00\% (6.6m) & 0.00\% & 0.00\% (6.6m) & 0.00\% (6.6m) & 0.00\% \\
    MVMoE & 46.31\% (1.8m) & $-$ & $-$ & 60.72\% (2.1m) & $-$ & $-$ & 60.87\% (2.1m) & $-$ & $-$ \\
    MTPOMO & 39.03\% (1.3m) & $-$ & $-$ & 48.33\% (1.5m) & $-$ & $-$ & 47.56\% (1.5m) & $-$ & $-$ \\
    ReLD-MTL & 33.71\% (1.6m) & $-$ & $-$ & 36.09\% (1.7m) & $-$ & $-$ & 36.15\% (1.7m) & $-$ & $-$ \\
    URS   & 29.49\% (1.0m) & 21.24\% (17m) & 25.36\% & 24.97\% (1.1m) & 37.37\% (22m) & 31.17\% & 24.99\% (1.1m) & 36.71\% (22m) & 30.85\% \\
    SPACE & 26.37\% (1.0m) & 14.34\% (17m) & \greybg{20.36\%} & 27.64\% (1.1m) & 23.35\% (17m) & \greybg{25.49\%} & 27.79\% (1.1m) & 22.81\% (17m) & \greybg{25.30\%} \\
    \midrule
    \midrule
    Problem & \multicolumn{3}{c|}{MDCVRPTW} & \multicolumn{3}{c|}{MDOCVRP} & \multicolumn{3}{c}{MDOCVRPB} \\
    \midrule
    Oracle & 0.00\% (6.6m) & 0.00\% (6.6m) & 0.00\% & 0.00\% (6.6m) & 0.00\% (6.6m) & 0.00\% & 0.00\% (6.6m) & 0.00\% (6.6m) & 0.00\% \\
    MVMoE & 40.61\% (2.1m) & $-$ & $-$ & 37.43\% (1.9m) & $-$ & $-$ & 29.77\% (1.5m) & $-$ & $-$ \\
    MTPOMO & 33.05\% (1.5m) & $-$ & $-$ & 29.65\% (1.3m) & $-$ & $-$ & 24.80\% (1.1m) & $-$ & $-$ \\
    ReLD-MTL & 25.48\% (1.8m) & $-$ & $-$ & 22.86\% (1.5m) & $-$ & $-$ & 17.66\% (1.2m) & $-$ & $-$ \\
    URS   & 31.51\% (1.3m) & 15.38\% (21m) & 23.45\% & 17.63\% (57s) & 42.04\% (22m) & 29.84\% & 13.46\% (47s) & 25.95\% (16m) & 19.70\% \\
    SPACE & 25.97\% (1.2m) & 4.73\% (19m) & \greybg{15.35\%} & 15.01\% (1.0m) & 22.11\% (15m) & \greybg{18.56\%} & 13.07\% (47s) & 11.03\% (12m) & \greybg{12.05\%} \\
    \midrule
    \midrule
    Problem & \multicolumn{3}{c|}{MDOCVRPBPLTW} & \multicolumn{3}{c|}{MDOCVRPBPTW} & \multicolumn{3}{c}{MDOCVRPTW} \\
    \midrule
    Oracle & 0.00\% (6.6m) & 0.00\% (6.6m) & 0.00\% & 0.00\% (6.6m) & 0.00\% (6.6m) & 0.00\% & 0.00\% (6.6m) & 0.00\% (6.6m) & 0.00\% \\
    MVMoE & 48.64\% (2.3m) & $-$ & $-$ & 49.11\% (2.3m) & $-$ & $-$ & 35.99\% (2.1m) & $-$ & $-$ \\
    MTPOMO & 40.66\% (1.6m) & $-$ & $-$ & 41.05\% (1.6m) & $-$ & $-$ & 30.29\% (1.4m) & $-$ & $-$ \\
    ReLD-MTL & 36.80\% (2.1m) & $-$ & $-$ & 36.87\% (2.0m) & $-$ & $-$ & 24.87\% (1.8m) & $-$ & $-$ \\
    URS   & 26.84\% (1.3m) & 22.06\% (22m) & 24.45\% & 26.88\% (1.3m) & 23.22\% (21m) & 25.05\% & 22.05\% (1.1m) & 25.30\% (20m) & 23.67\% \\
    SPACE & 25.46\% (1.3m) & 16.08\% (21m) & \greybg{20.77\%} & 25.19\% (1.3m) & 16.98\% (21m) & \greybg{21.09\%} & 18.21\% (1.1m) & 16.29\% (19m) & \greybg{17.25\%} \\
    \midrule
    \midrule
    Problem & \multicolumn{3}{c|}{MDOCVRPBTW} & \multicolumn{3}{c|}{MDOCVRPL} & \multicolumn{3}{c}{MDOCVRPLTW} \\
    \midrule
    Oracle & 0.00\% (6.6m) & 0.00\% (6.6m) & 0.00\% & 0.00\% (6.6m) & 0.00\% (6.6m) & 0.00\% & 0.00\% (6.6m) & 0.00\% (6.6m) & 0.00\% \\
    MVMoE & 46.75\% (1.8m) & $-$ & $-$ & 37.48\% (2.0m) & $-$ & $-$ & 35.82\% (2.1m) & $-$ & $-$ \\
    MTPOMO & 39.07\% (1.2m) & $-$ & $-$ & 29.37\% (1.4m) & $-$ & $-$ & 30.08\% (1.5m) & $-$ & $-$ \\
    ReLD-MTL & 33.68\% (1.5m) & $-$ & $-$ & 22.87\% (1.6m) & $-$ & $-$ & 24.65\% (1.8m) & $-$ & $-$ \\
    URS   & 29.54\% (58s) & 21.97\% (17m) & 25.75\% & 17.64\% (1.0m) & 42.56\% (23m) & 30.10\% & 21.88\% (1.1m) & 24.30\% (20m) & 23.09\% \\
    SPACE & 26.52\% (59s) & 15.16\% (18m) & \greybg{20.84\%} & 15.08\% (1.0m) & 22.46\% (16m) & \greybg{18.77\%} & 18.16\% (1.1m) & 15.35\% (20m) & \greybg{16.75\%} \\
    \bottomrule[0.5mm]
    \end{tabular}%
    }
  \label{tab:detailed_multidepot_results}%
\end{table*}%

\subsection{Inconsistent Problem Settings}
\label{append:compare_rf_cada}

For our cross-problem experiment in the main text, RF~\citep{berto2024routefinder} and CaDA~\citep{li2025cada} are excluded due to discrepancies in the problems, including the number of training CVRP variants (16 in their work vs. 5 in ours) and constraint settings (i.e., Backhauls and Time Windows). To ensure a fair comparison with RF~\citep{berto2024routefinder} and CaDA~\citep{li2025cada}, we evaluate SPACE against the official best-performing version of RouteFinder(i.e., RF-TE) and CaDA models on a subset of problems that all three methods have both seen and unseen, strictly filtering for variants where the constraint definitions align perfectly across all methods. As shown in \cref{tab:RF_CaDA_official}, when evaluated under strictly matched conditions, SPACE still performs best overall among comparable methods. Note that SPACE can provide the strongest zero-shot generalization across all shared unseen problems.

\begin{table}[htbp]
  \centering
  \caption{Comparison with RF and CaDA on shared seen and unseen VRP variants.}
  \resizebox{\textwidth}{!}{
    \begin{tabular}{c|c|c|c|c|c|c|c|c}
    \toprule[0.5mm]
    \multirow{2}[2]{*}{Method} & \multicolumn{2}{c|}{Seen} & \multicolumn{4}{c|}{Unseen}   & \multirow{2}[2]{*}{Avg.gap} & \multirow{2}[2]{*}{Best Sol.} \\
          & CVRP  & OCVRP & MDCVRP & MDOCVRP & MDCVRPL & MDOCVRPL &       &  \\
    \midrule
    RF-TE & 2.06\% & 2.86\% & 63.60\% & 33.52\% & 66.75\% & 34.40\% & 33.86\% & 0/6 \\
    CaDA  & 2.17\% & \cellcolor[HTML]{D0CECE}\textbf{2.80\%} & 40.96\% & 53.21\% & 42.40\% & 53.34\% & 32.48\% & 1/6 \\
    SPACE   & \cellcolor[HTML]{D0CECE}\textbf{1.83\%} & 3.31\% & \cellcolor[HTML]{D0CECE}\textbf{16.67\%} & \cellcolor[HTML]{D0CECE}\textbf{15.01\%} & \cellcolor[HTML]{D0CECE}\textbf{16.76\%} & \cellcolor[HTML]{D0CECE}\textbf{15.08\%} & \cellcolor[HTML]{D0CECE}\textbf{11.44\%} & \cellcolor[HTML]{D0CECE}\textbf{5/6} \\
    \bottomrule[0.5mm]
    \end{tabular}%
    }
  \label{tab:RF_CaDA_official}%
\end{table}%

\clearpage

\section{Results On Benchmark Dataset}

We provide detailed results on public CVRPLIB benchmark instances in this section. CVRPLIB Set-X~\citep{uchoa2017cvrplib_setx} and Set-XXL~\citep{arnold2019cvrplib_xxl} contain fixed large-scale real-world instances with hundreds to thousands of nodes. Since all neural solvers are trained only on generated instances of size $N=100$, this benchmark evaluates whether the learned routing policy generalizes to substantially larger instances.

Table \ref{tb:exp_benchmark_large_scale} reports the detailed Set-X results with $N \in [500,1000]$. SPACE achieves the lowest average gap among the compared neural solvers. These results indicate that the proposed SPACE not only improves performance on synthetic training-scale distributions but also strengthens the transfer of a generalist neural solver to established large-scale benchmark instances.

\begin{table*}[htbp]
\caption{Results on large-scale CVRPLIB instances (Set-X) \citep{uchoa2017cvrplib_setx} ($N \in [500,1000]$).}
  \label{tb:exp_benchmark_large_scale}
  \begin{center}
  \begin{small}
  \renewcommand\arraystretch{1.5}
  \resizebox{\textwidth}{!}{ 
  \begin{tabular}{cc|cccccccccccccccccc}
    \toprule[0.5mm]
    \multicolumn{2}{c|}{Set-X} & \multicolumn{2}{c}{POMO-MTL} & \multicolumn{2}{c}{MVMoE/4E} & \multicolumn{2}{c}{MVMOE/4E-L} & \multicolumn{2}{c}{RF-MVMOE} & \multicolumn{2}{c}{RF-TE} & \multicolumn{2}{c}{CaDA} & \multicolumn{2}{c}{ReLD-MTL} & \multicolumn{2}{c}{URS} & \multicolumn{2}{c}{SPACE} \\
    Instance & Opt.  & Obj.  & Gap   & Obj.  & Gap   & Obj.  & Gap   & Obj.  & Gap   & Obj.  & Gap   & Obj.  & Gap   & Obj.  & Gap   & Obj.  & Gap   & Obj.  & Gap \\
    \midrule
    X-n502-k39 & 69226 & 77284 & 11.640\% & 73533 & 6.222\% & 74429 & 7.516\% & 76338 & 10.274\% & 71791 & 3.705\% & 189318 & 173.478\% & 71954 & 3.941\% & \greybg{71281} & \greybg{2.969\%} & 74139 & 7.097\% \\
    X-n513-k21 & 24201 & 28510 & 17.805\% & 32102 & 32.647\% & 31231 & 29.048\% & 32639 & 34.866\% & 28465 & 17.619\% & 236328 & 876.522\% & 27554 & 13.855\% & \greybg{26166} & \greybg{8.119\%} & 26311 & 8.719\% \\
    X-n524-k153 & 154593 & 192249 & 24.358\% & 186540 & 20.665\% & 182392 & 17.982\% & \greybg{170999} & \greybg{10.612\%} & 174381 & 12.800\% & 226478 & 46.500\% & 171465 & 10.914\% & 175250 & 13.362\% & 176153 & 13.946\% \\
    X-n536-k96 & 94846 & 106514 & 12.302\% & 109581 & 15.536\% & 108543 & 14.441\% & 105847 & 11.599\% & 103272 & 8.884\% & 203165 & 114.205\% & \greybg{101259} & \greybg{6.761\%} & 102969 & 8.564\% & 103115 & 8.718\% \\
    X-n548-k50 & 86700 & 94562 & 9.068\% & 95894 & 10.604\% & 95917 & 10.631\% & 104289 & 20.287\% & 100956 & 16.443\% & 220653 & 154.502\% & 91802 & 5.885\% & \greybg{89768} & \greybg{3.539\%} & 90421 & 4.292\% \\
    X-n561-k42 & 42717 & 47846 & 12.007\% & 56008 & 31.114\% & 51810 & 21.287\% & 53383 & 24.969\% & 49454 & 15.771\% & 288552 & 575.497\% & 47390 & 10.939\% & 45964 & 7.601\% & \greybg{45681} & \greybg{6.939\%} \\
    X-n573-k30 & 50673 & 60913 & 20.208\% & 59473 & 17.366\% & 57042 & 12.569\% & 61524 & 21.414\% & 55952 & 10.418\% & 144633 & 185.424\% & \greybg{53893} & \greybg{6.354\%} & 54361 & 7.278\% & 54880 & 8.302\% \\
    X-n586-k159 & 190316 & 208893 & 9.761\% & 215668 & 13.321\% & 214577 & 12.748\% & 212151 & 11.473\% & 205575 & 8.018\% & 319620 & 67.942\% & 202742 & 6.529\% & 202645 & 6.478\% & \greybg{201504} & \greybg{5.879\%} \\
    X-n599-k92 & 108451 & 120333 & 10.956\% & 128949 & 18.901\% & 125279 & 15.517\% & 126578 & 16.714\% & 116560 & 7.477\% & 224439 & 106.950\% & 116539 & 7.458\% & \greybg{114423} & \greybg{5.507\%} & 114959 & 6.001\% \\
    X-n613-k62 & 59535 & 67984 & 14.192\% & 82586 & 38.718\% & 74945 & 25.884\% & 73456 & 23.383\% & 67267 & 12.987\% & 348589 & 485.519\% & 65896 & 10.684\% & 65901 & 10.693\% & \greybg{63680} & \greybg{6.962\%} \\
    X-n627-k43 & 62164 & 73060 & 17.528\% & 70987 & 14.193\% & 70905 & 14.061\% & 70414 & 13.271\% & 67572 & 8.700\% & 148241 & 138.468\% & 67329 & 8.309\% & \greybg{66499} & \greybg{6.973\%} & 66889 & 7.601\% \\
    X-n641-k35 & 63682 & 72643 & 14.071\% & 75329 & 18.289\% & 72655 & 14.090\% & 71975 & 13.023\% & 70831 & 11.226\% & 423415 & 564.890\% & 69382 & 8.951\% & \greybg{67005} & \greybg{5.218\%} & 67369 & 5.790\% \\
    X-n655-k131 & 106780 & 116988 & 9.560\% & 117678 & 10.206\% & 118475 & 10.952\% & 119057 & 11.497\% & 112202 & 5.078\% & 161519 & 51.263\% & 110656 & 3.630\% & 110237 & 3.237\% & \greybg{109203} & \greybg{2.269\%} \\
    X-n670-k130 & 146332 & 190118 & 29.922\% & 197695 & 35.100\% & 183447 & 25.364\% & 168226 & 14.962\% & 168999 & 15.490\% & 274207 & 87.387\% & \greybg{167064} & \greybg{14.168\%} & 184010 & 25.748\% & 170270 & 16.359\% \\
    X-n685-k75 & 68205 & 80892 & 18.601\% & 97388 & 42.787\% & 89441 & 31.136\% & 82269 & 20.620\% & 77847 & 14.137\% & 342170 & 401.679\% & 75501 & 10.697\% & 75942 & 11.344\% & \greybg{74776} & \greybg{9.634\%} \\
    X-n701-k44 & 81923 & 92075 & 12.392\% & 98469 & 20.197\% & 94924 & 15.870\% & 90189 & 10.090\% & 89932 & 9.776\% & 397296 & 384.963\% & 88989 & 8.625\% & \greybg{86038} & \greybg{5.023\%} & 86556 & 5.655\% \\
    X-n716-k35 & 43373 & 52709 & 21.525\% & 56773 & 30.895\% & 52305 & 20.593\% & 52250 & 20.467\% & 49669 & 14.516\% & 229530 & 429.200\% & 48602 & 12.056\% & 46496 & 7.200\% & \greybg{46320} & \greybg{6.795\%} \\
    X-n733-k159 & 136187 & 161961 & 18.925\% & 178322 & 30.939\% & 167477 & 22.976\% & 156387 & 14.833\% & 148463 & 9.014\% & 289794 & 112.791\% & 147574 & 8.361\% & 147743 & 8.485\% & \greybg{145816} & \greybg{7.070\%} \\
    X-n749-k98 & 77269 & 90582 & 17.229\% & 100438 & 29.985\% & 94497 & 22.296\% & 92147 & 19.255\% & 85171 & 10.227\% & 191170 & 147.408\% & 84643 & 9.543\% & 83759 & 8.399\% & \greybg{83278} & \greybg{7.777\%} \\
    X-n766-k71 & 114417 & 144041 & 25.891\% & 152352 & 33.155\% & 136255 & 19.086\% & 130505 & 14.061\% & 129935 & 13.563\% & 297390 & 159.918\% & \greybg{126834} & \greybg{10.852\%} & 139371 & 21.810\% & 127186 & 11.160\% \\
    X-n783-k48 & 72386 & 83169 & 14.897\% & 100383 & 38.677\% & 92960 & 28.423\% & 96336 & 33.087\% & 83185 & 14.919\% & 532247 & 635.290\% & 81305 & 12.321\% & \greybg{77437} & \greybg{6.978\%} & 78570 & 8.543\% \\
    X-n801-k40 & 73305 & 85077 & 16.059\% & 91560 & 24.903\% & 87662 & 19.585\% & 87118 & 18.843\% & 86164 & 17.542\% & 595083 & 711.790\% & 81063 & 10.583\% & \greybg{77369} & \greybg{5.544\%} & 77933 & 6.313\% \\
    X-n819-k171 & 158121 & 177157 & 12.039\% & 183599 & 16.113\% & 185832 & 17.525\% & 179596 & 13.581\% & 174441 & 10.321\% & 385566 & 143.842\% & \greybg{169248} & \greybg{7.037\%} & 171024 & 8.160\% & 170592 & 7.887\% \\
    X-n837-k142 & 193737 & 214207 & 10.566\% & 229526 & 18.473\% & 221286 & 14.220\% & 230362 & 18.904\% & 208528 & 7.635\% & 393693 & 103.210\% & 206963 & 6.827\% & 203457 & 5.017\% & \greybg{203379} & \greybg{4.977\%} \\
    X-n856-k95 & 88965 & 101774 & 14.398\% & 99129 & 11.425\% & 106816 & 20.065\% & 105801 & 18.924\% & 98291 & 10.483\% & 480248 & 439.817\% & 96215 & 8.149\% & \greybg{94547} & \greybg{6.274\%} & 96288 & 8.231\% \\
    X-n876-k59 & 99299 & 116617 & 17.440\% & 119619 & 20.463\% & 114333 & 15.140\% & 114016 & 14.821\% & 107416 & 8.174\% & 501455 & 404.995\% & 105900 & 6.648\% & 105417 & 6.161\% & \greybg{104368} & \greybg{5.105\%} \\
    X-n895-k37 & 53860 & 65587 & 21.773\% & 79018 & 46.710\% & 64310 & 19.402\% & 69099 & 28.294\% & 64871 & 20.444\% & 597118 & 1008.648\% & 62970 & 16.914\% & \greybg{58137} & \greybg{7.941\%} & 58511 & 8.635\% \\
    X-n916-k207 & 329179 & 361719 & 9.885\% & 383681 & 16.557\% & 374016 & 13.621\% & 373600 & 13.494\% & 352998 & 7.236\% & 619583 & 88.221\% & 350310 & 6.419\% & 346556 & 5.279\% & \greybg{346431} & \greybg{5.241\%} \\
    X-n936-k151 & 132715 & 186262 & 40.347\% & 220926 & 66.466\% & 190407 & 43.471\% & 161343 & 21.571\% & 163162 & 22.942\% & 388590 & 192.800\% & \greybg{155144} & \greybg{16.900\%} & 172675 & 30.110\% & 157340 & 18.555\% \\
    X-n957-k87 & 85465 & 98198 & 14.898\% & 113882 & 33.250\% & 105629 & 23.593\% & 123633 & 44.659\% & 102689 & 20.153\% & 414723 & 385.255\% & 93156 & 8.999\% & \greybg{90485} & \greybg{5.874\%} & 90710 & 6.137\% \\
    X-n979-k58 & 118976 & 138092 & 16.067\% & 146347 & 23.005\% & 139682 & 17.404\% & 131754 & 10.740\% & 129952 & 9.225\% & 542553 & 356.019\% & 127584 & 7.235\% & \greybg{125353} & \greybg{5.360\%} & 125495 & 5.479\% \\
    X-n1001-k43 & 72355 & 87660 & 21.153\% & 114448 & 58.176\% & 94734 & 30.929\% & 88969 & 22.962\% & 85929 & 18.760\% & 781780 & 980.478\% & 83572 & 15.503\% & \greybg{77739} & \greybg{7.441\%} & 79256 & 9.538\% \\
    \midrule
    \multicolumn{2}{c|}{Avg. Gap} & \multicolumn{2}{c}{16.796\%} & \multicolumn{2}{c}{26.408\%} & \multicolumn{2}{c}{19.607\%} & \multicolumn{2}{c}{18.795\%} & \multicolumn{2}{c}{12.303\%} & \multicolumn{2}{c}{334.840\%} & \multicolumn{2}{c}{9.439\%} & \multicolumn{2}{c}{8.678\%} & \multicolumn{2}{c}{\greybg{7.863\%}} \\
    \bottomrule[0.5mm]
    \end{tabular}
    }
  \end{small}
  \end{center}
\end{table*}

Table \ref{tab:setxxl_detailed} further evaluates the more challenging Set-XXL instances with $N \in [3000,7000]$. Most neural baselines cannot handle large instances due to memory limitations. Combined with the Set-X results, SPACE achieves the lowest benchmark average across all reported size ranges, supporting its robustness under out-of-distribution scale transfer.

\begin{table}[htbp]
  \centering
  \caption{Results on large-scale CVRPLIB instances (Set-XXL) \citep{arnold2019cvrplib_xxl} ($N \in [3000,7000]$).}
  \resizebox{0.90\textwidth}{!}{
    \begin{tabular}{ll|c|c|c|c|cc}
    \toprule[0.5mm]
   \multicolumn{2}{l|}{\multirow{2}[2]{*}{Method}} & Leuven1 & Leuven2 & Antwerp1 & Antwerp2 & \multicolumn{2}{c}{\multirow{2}[2]{*}{Avg.gap}} \\
   \multicolumn{2}{l|}{}& (N=3000) & (N=4000) & (N=6000) & (N=7000) &\multicolumn{2}{c}{} \\
    \midrule
    \multicolumn{2}{l|}{MTPOMO }  & 67.72\% & 87.31\% & OOM   & OOM   & \multicolumn{2}{c}{$-$} \\
    \multicolumn{2}{l|}{MVMoE/4E  }  & 299.10\% & 170.10\% & OOM   & OOM   & \multicolumn{2}{c}{$-$} \\
    \multicolumn{2}{l|}{MVMoE/4E-L  }  & 182.28\% & 127.07\% & OOM   & OOM   & \multicolumn{2}{c}{$-$} \\
    \multicolumn{2}{l|}{RF-MVMoE  }  & 57.30\% & 160.27\% & OOM   & OOM   & \multicolumn{2}{c}{$-$} \\
    \multicolumn{2}{l|}{RF-TE  }  & 26.90\% & 45.56\% & OOM   & OOM   & \multicolumn{2}{c}{$-$} \\
    \multicolumn{2}{l|}{CaDA  }  & 1035.51\% & OOM   & OOM   & OOM   & \multicolumn{2}{c}{$-$} \\
    \multicolumn{2}{l|}{ReLD-MTL  }  & 17.00\% & 30.59\% & OOM   & OOM   & \multicolumn{2}{c}{$-$} \\
    \midrule
    \multicolumn{2}{l|}{URS  }  
    & 11.57\% & 17.80\% & 9.23\% & \greybg{14.95\%} 
    & \multicolumn{2}{c}{13.39\%} \\
    \multicolumn{2}{l|}{SPACE  }  
    & \greybg{8.79\%} & \greybg{16.26\%} & \greybg{7.82\%} & 15.69\%
    & \multicolumn{2}{c}{\greybg{12.14\%}} \\
    \bottomrule[0.5mm]
    \end{tabular}%
    }
  \label{tab:setxxl_detailed}%
\end{table}%

\clearpage

\section{Training Problem Selection}
\label{append:problem_selection}

\begin{figure*}[h!]
    \centering
    \includegraphics[width=1\linewidth]{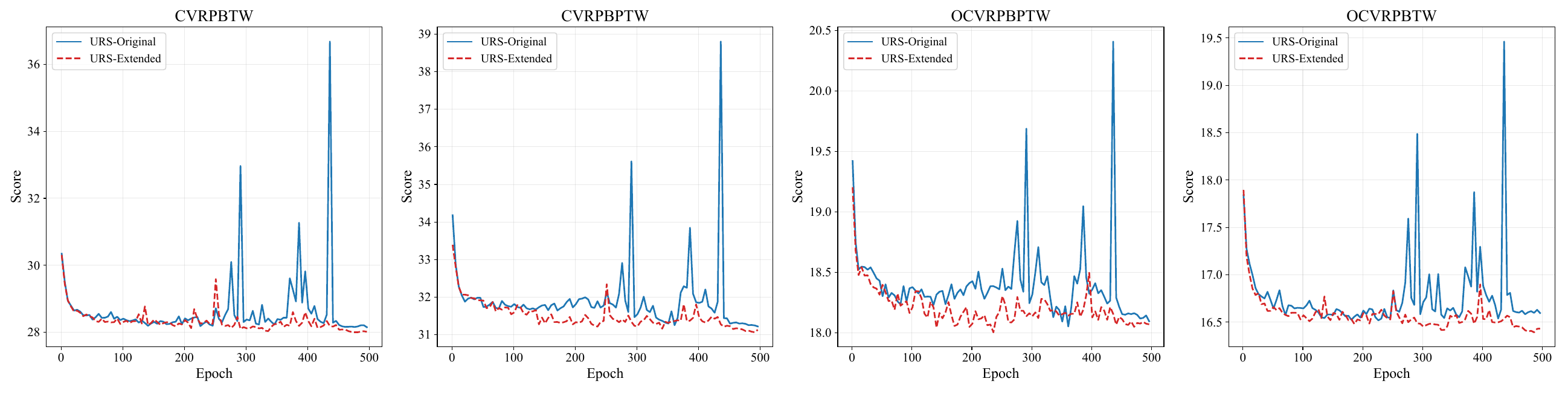}
    \caption{The validation curves of URS trained w/o ACVRPBTW (denoted as URS-Original) and w. ACVRPBTW (denoted as URS-Extended).}
    \label{fig:urs12}
\end{figure*}
As detailed in Section \ref{sec:experiments}, our training set encompasses 12 VRP variants. Unlike the URS configuration, we include ACVRPBTW in our training set. This design choice is driven by two primary considerations: (1) As discussed in Appendix \ref{append:udr} and \ref{append:multi_hot}, SPACE enhances the original UDR by replacing the random identifier and node coordinate indicators with our bidirectional Fréchet representation. This architectural update requires the model to learn unified routing knowledge from symmetric and asymmetric variants. Given that the original URS incorporates only two asymmetric problems (i.e., ATSP and ACVRP), we deliberately introduce an additional asymmetric variant to enhance the model's representational learning capacity. (2) More importantly, since SPACE is built from the architecture of URS, as empirically observed in Figure \ref{fig:urs12}, the original URS exhibits severe performance instability on variants containing both B and TW constraints. We find that including this specific problem effectively stabilizes training on these complex variants that integrate B and TW constraints.

The main paper compares against URS under its published training protocol, because this corresponds to the official baseline and keeps the primary comparison aligned with prior work. To ensure a fair comparison, we conduct a controlled comparison by retraining a URS variant (denoted URS-Extended) on the same 12 training problems as SPACE. This setting directly tests whether improvements in SPACE can be explained by differences in training problems. To have a comprehensive experimental comparison, we evaluate SPACE against URS-Extended on 96 VRPs (48 CVRP variants and their asymmetric counterparts). As shown in Tables~\ref{tab:urs12_space_comparison_sym} and~\ref{tab:urs12_space_comparison_asym}, SPACE still achieves lower average gaps under the matched training set compared to URS-Extended. Notably, our SPACE achieves better results across substantially more variants. These results indicate that the gains mainly come from the proposed unified representation and adaptive decoding, rather than from increased training problems.

\begin{table*}[htbp]
  \centering
  \caption{Controlled comparison on symmetric variants with URS retrained on the same training problems as SPACE.}
  \resizebox{\textwidth}{!}{
    \begin{tabular}{l|c|c|c|c|c}
    \toprule[0.5mm]
    Method & CVRPTW & OCVRP & CVRPL & CVRPB & OCVRPTW \\
    \midrule
    URS-Extended & 5.78\% & \greybg{3.28\%} & 0.45\% & 1.47\% & 4.79\% \\
    SPACE & \greybg{5.36\%} & 3.31\% & \greybg{0.45\%} & \greybg{1.45\%} & \greybg{4.20\%} \\
    \midrule
    \midrule
    Method & OCVRPB & CVRPBL & CVRPLTW & OCVRPBTW & CVRPBLTW \\
    \midrule
    URS-Extended & 8.42\% & 1.68\% & 2.30\% & 12.45\% & 8.63\% \\
    SPACE & \greybg{7.33\%} & \greybg{1.67\%} & \greybg{1.82\%} & \greybg{11.60\%} & \greybg{8.29\%} \\
    \midrule
    \midrule
    Method & OCVRPL & CVRPBTW & OCVRPBL & OCVRPLTW & OCVRPBLTW \\
    \midrule
    URS-Extended & \greybg{3.25\%} & 8.27\% & 8.54\% & 4.84\% & 12.52\% \\
    SPACE & 3.29\% & \greybg{7.91\%} & \greybg{7.36\%} & \greybg{4.26\%} & \greybg{11.63\%} \\
    \midrule
    \midrule
    Method & CVRPBP & OCVRPBP & CVRPBPL & OCVRPBPTW & CVRPBPLTW \\
    \midrule
    URS-Extended & 13.74\% & 17.41\% & 13.20\% & 9.72\% & 17.62\% \\
    SPACE & \greybg{13.49\%} & \greybg{16.85\%} & \greybg{13.13\%} & \greybg{9.62\%} & \greybg{17.54\%} \\
    \midrule
    \midrule
    Method & CVRPBPTW & OCVRPBPL & OCVRPBPLTW & CVRP  & MDCVRPTW \\
    \midrule
    URS-Extended & 17.61\% & 16.60\% & \greybg{9.22\%} & 1.83\% & 30.86\% \\
    SPACE & \greybg{17.57\%} & \greybg{16.31\%} & 9.23\% & \greybg{1.83\%} & \greybg{25.97\%} \\
    \midrule
    \midrule
    Method & MDOCVRP & MDCVRPL & MDCVRPB & MDOCVRPTW & MDOCVRPB \\
    \midrule
    URS-Extended & 16.80\% & 18.31\% & \greybg{3.97\%} & 22.22\% & 13.53\% \\
    SPACE & \greybg{15.01\%} & \greybg{16.76\%} & 5.49\% & \greybg{18.21\%} & \greybg{13.07\%} \\
    \midrule
    \midrule
    Method & MDCVRPBL & MDCVRPLTW & MDOCVRPBTW & MDCVRPBLTW & MDOCVRPL \\
    \midrule
    URS-Extended & \greybg{3.32\%} & 32.77\% & 29.22\% & 37.92\% & 16.61\% \\
    SPACE & 5.00\% & \greybg{27.29\%} & \greybg{26.52\%} & \greybg{35.97\%} & \greybg{15.08\%} \\
    \midrule
    \midrule
    Method & MDCVRPBTW & MDOCVRPBL & MDOCVRPLTW & MDOCVRPBLTW & MDCVRPBP \\
    \midrule
    URS-Extended & 38.17\% & 13.36\% & 22.50\% & 29.30\% & \greybg{31.31\%} \\
    SPACE & \greybg{36.79\%} & \greybg{12.95\%} & \greybg{18.16\%} & \greybg{26.37\%} & 34.16\% \\
    \midrule
    \midrule
    Method & MDOCVRPBP & MDCVRPBPL & MDOCVRPBPTW & MDCVRPBPLTW & MDCVRPBPTW \\
    \midrule
    URS-Extended & \greybg{26.88\%} & \greybg{31.00\%} & 27.03\% & 39.52\% & \greybg{39.77\%} \\
    SPACE & 27.63\% & 34.20\% & \greybg{25.19\%} & \greybg{39.38\%} & 39.88\% \\
    \midrule
    \midrule
    Method & MDOCVRPBPL & MDOCVRPBPLTW & MDCVRP & Avg.gap & Best Sol. \\
    \midrule
    URS-Extended & \greybg{26.80\%} & 27.58\% & 17.95\% & 16.67\% & 10/48 \\
    SPACE & 27.79\% & \greybg{25.45\%} & \greybg{16.67\%} & \greybg{15.93\%} & \greybg{38/48} \\
    \bottomrule[0.5mm]
    \end{tabular}%
    }   
  \label{tab:urs12_space_comparison_sym}%
\end{table*}%

\begin{table*}[h!]
  \centering
  \caption{Controlled comparison on asymmetric variants with URS retrained on the same training problems as SPACE.}
  \resizebox{\textwidth}{!}{
    \begin{tabular}{l|c|c|c|c|c}
    \toprule[0.5mm]
    Method & ACVRPTW & AOCVRP & ACVRPL & ACVRPB & AOCVRPTW \\
    \midrule
    URS-Extended & 1.84\% & 26.82\% & \greybg{-2.23\%} & \greybg{-14.45\%} & 14.06\% \\
    SPACE & \greybg{1.54\%} & \greybg{11.18\%} & -1.88\% & -14.05\% & \greybg{7.97\%} \\
    \midrule
    \midrule
    Method & AOCVRPB & ACVRPBL & ACVRPLTW & AOCVRPBTW & ACVRPBLTW \\
    \midrule
    URS-Extended & 10.97\% & \greybg{-14.77\%} & 2.62\% & 7.88\% & 0.20\% \\
    SPACE & \greybg{-1.34\%} & -14.27\% & \greybg{2.26\%} & \greybg{5.48\%} & \greybg{-0.37\%} \\
    \midrule
    \midrule
    Method & AOCVRPL & ACVRPBTW & AOCVRPBL & AOCVRPLTW & AOCVRPBLTW \\
    \midrule
    URS-Extended & 26.98\% & 0.07\% & 10.74\% & 12.95\% & 7.96\% \\
    SPACE & \greybg{11.02\%} & \greybg{-0.52\%} & \greybg{-1.94\%} & \greybg{6.99\%} & \greybg{5.52\%} \\
    \midrule
    \midrule
    Method & ACVRPBP & AOCVRPBP & ACVRPBPL & AOCVRPBPTW & ACVRPBPLTW \\
    \midrule
    URS-Extended & 8.00\% & 27.20\% & 7.82\% & 11.99\% & 4.69\% \\
    SPACE & \greybg{7.93\%} & \greybg{17.40\%} & \greybg{7.82\%} & \greybg{9.65\%} & \greybg{4.16\%} \\
    \midrule
    \midrule
    Method & ACVRPBPTW & AOCVRPBPL & AOCVRPBPLTW & ACVRP & AMDCVRPTW \\
    \midrule
    URS-Extended & 5.34\% & 26.87\% & 10.88\% & \greybg{-1.96\%} & 6.88\% \\
    SPACE & \greybg{4.73\%} & \greybg{17.00\%} & \greybg{8.57\%} & -1.62\% & \greybg{4.72\%} \\
    \midrule
    \midrule
    Method & AMDOCVRP & AMDCVRPL & AMDCVRPB & AMDOCVRPTW & AMDOCVRPB \\
    \midrule
    URS-Extended & 32.08\% & 4.23\% & \greybg{-9.89\%} & 25.00\% & 16.09\% \\
    SPACE & \greybg{22.11\%} & \greybg{3.46\%} & -9.40\% & \greybg{16.29\%} & \greybg{11.03\%} \\
    \midrule
    \midrule
    Method & AMDCVRPBL & AMDCVRPLTW & AMDOCVRPBTW & AMDCVRPBLTW & AMDOCVRPL \\
    \midrule
    URS-Extended & \greybg{-8.84\%} & 6.55\% & 19.19\% & 5.87\% & 32.23\% \\
    SPACE & -8.50\% & \greybg{4.64\%} & \greybg{15.16\%} & \greybg{2.70\%} & \greybg{22.46\%} \\
    \midrule
    \midrule
    Method & AMDCVRPBTW & AMDOCVRPBL & AMDOCVRPLTW & AMDOCVRPBLTW & AMDCVRPBP \\
    \midrule
    URS-Extended & 5.82\% & 15.59\% & 24.31\% & 18.31\% & 11.38\% \\
    SPACE & \greybg{2.69\%} & \greybg{10.31\%} & \greybg{15.35\%} & \greybg{14.34\%} & \greybg{9.69\%} \\
    \midrule
    \midrule
    Method & AMDOCVRPBP & AMDCVRPBPL & AMDOCVRPBPTW & AMDCVRPBPLTW & AMDCVRPBPTW \\
    \midrule
    URS-Extended & 28.24\% & 10.49\% & 21.18\% & 11.50\% & 11.93\% \\
    SPACE & \greybg{23.35\%} & \greybg{8.66\%} & \greybg{16.98\%} & \greybg{7.51\%} & \greybg{8.22\%} \\
    \midrule
    \midrule
    Method & AMDOCVRPBPL & AMDOCVRPBPLTW & AMDCVRP & Avg.gap & Best Sol. \\
    \midrule
    URS-Extended & 27.82\% & 20.07\% & 4.94\% & 11.20\% & 6/48 \\
    SPACE & \greybg{22.81\%} & \greybg{16.08\%} & \greybg{3.93\%} & \greybg{7.13\%} & \greybg{42/48} \\
    \bottomrule[0.5mm]
    \end{tabular}%
    }   
  \label{tab:urs12_space_comparison_asym}%
\end{table*}%
\clearpage

\section{Ablation Study}
\label{append:ablation}

\subsection{Effects of Bidirectional Fréchet Representation}

Table \ref{tab:ablation_bfr_details} reports the per-variant ablation results for the main BFR design components. \textit{Pivot} indicates whether the BFR is used; when disabled, the model falls back to a unified representation (e.g., GOAL). \textit{FPS} indicates whether pivots are selected by Furthest Pivot Sampling; when disabled, pivots are selected randomly. \textit{Depot} indicates whether the depot is used as the initial pivot; when disabled, the initial pivot is randomly selected. \textit{Bi.} indicates whether both incoming and outgoing pivot distances are retained; when disabled, they are replaced by their bidirectional average.

The complete BFR design achieves the best average gap in both symmetric and asymmetric settings. Removing the entire BFR increases the symmetric average gap, showing that BFR provides useful representation beyond other unified representations. Among the individual components, FPS is the most important: replacing FPS with random pivot selection increases the average gap to $83.10\%$ on symmetric variants and $95.72\%$ on asymmetric variants. This confirms that BFR depends not only on using pivots, but also on selecting spatially diverse pivots that cover the topology. Removing depot initialization or bidirectional coordinates leads to smaller but consistent degradation, indicating that the depot anchor and direction-aware coordinates further stabilize the representation.

\begin{table*}[htbp]
  \centering
  \caption{Detailed ablation results for Bidirectional Fréchet Representation.}
  \resizebox{\textwidth}{!}{
    \begin{tabular}{cccc|c|c|c|c|c|c}
    \toprule[0.5mm]
    Pivot & FPS   & depot & Bi.    & CVRP  & CVRPTW & OCVRP & CVRPL & CVRPB & OCVRPTW \\
    \midrule
    $\times$ & $\times$ & $\times$ & $\times$ & 2.25\% & 6.33\% & 4.00\% & 0.91\% & 2.21\% & 5.47\% \\
    $\checkmark$ & $\times$ & $\checkmark$ & $\checkmark$ & 23.23\% & 37.06\% & 121.75\% & 21.01\% & 26.24\% & 177.45\% \\
    $\checkmark$ & $\checkmark$ & $\times$ & $\checkmark$ & 2.09\% & 5.77\% & 3.51\% & 0.68\% & 1.81\% & 4.61\% \\
    $\checkmark$ & $\checkmark$ & $\checkmark$ & $\times$ & 1.92\% & 5.55\% & 3.58\% & 0.54\% & 1.66\% & 4.42\% \\
    $\checkmark$ & $\checkmark$ & $\checkmark$ & $\checkmark$ & \greybg{1.83\%} & \greybg{5.36\%} & \greybg{3.31\%} & \greybg{0.45\%} & \greybg{1.45\%} & \greybg{4.20\%} \\
    \midrule
    \midrule
    Pivot & FPS   & depot & Bi.    & OCVRPB & OCVRPL & CVRPBL & CVRPBTW & CVRPLTW & OCVRPBL \\
    \midrule
    $\times$ & $\times$ & $\times$ & $\times$ & 9.69\% & 4.02\% & 2.43\% & 9.52\% & 2.84\% & 9.77\% \\
    $\checkmark$ & $\times$ & $\checkmark$ & $\checkmark$ & 88.19\% & 120.92\% & 25.43\% & 31.31\% & 32.60\% & 88.03\% \\
    $\checkmark$ & $\checkmark$ & $\times$ & $\checkmark$ & 7.54\% & 3.48\% & 2.08\% & 9.62\% & 2.25\% & 7.69\% \\
    $\checkmark$ & $\checkmark$ & $\checkmark$ & $\times$ & 7.99\% & 3.61\% & 1.90\% & 9.42\% & 2.02\% & 8.17\% \\
    $\checkmark$ & $\checkmark$ & $\checkmark$ & $\checkmark$ & \greybg{7.33\%} & \greybg{3.29\%} & \greybg{1.67\%} & \greybg{7.91\%} & \greybg{1.82\%} & \greybg{7.36\%} \\
    \midrule
    \midrule
    Pivot & FPS   & depot & Bi.    & OCVRPBTW & OCVRPLTW & CVRPBLTW & OCVRPBLTW & Avg.gap & Best Sol. \\
    \midrule
    $\times$ & $\times$ & $\times$ & $\times$ & 23.94\% & 5.58\% & 9.76\% & 23.98\% & 7.67\% & 0/16 \\
    $\checkmark$ & $\times$ & $\checkmark$ & $\checkmark$ & 164.07\% & 177.23\% & 31.71\% & 163.39\% & 83.10\% & 0/16 \\
    $\checkmark$ & $\checkmark$ & $\times$ & $\checkmark$ & 13.81\% & 4.67\% & 10.08\% & 13.94\% & 5.85\% & 0/16 \\
    $\checkmark$ & $\checkmark$ & $\checkmark$ & $\times$ & 13.81\% & 4.49\% & 9.87\% & 13.80\% & 5.80\% & 0/16 \\
    $\checkmark$ & $\checkmark$ & $\checkmark$ & $\checkmark$ & \greybg{11.60\%} & \greybg{4.26\%} & \greybg{8.29\%} & \greybg{11.63\%} & \greybg{5.11\%} & \greybg{16/16} \\
    \toprule[0.5mm]
    Pivot & FPS   & depot & Bi.    & ACVRP & ACVRPTW & AOCVRP & ACVRPL & ACVRPB & AOCVRPTW \\
    \midrule
    $\times$ & $\times$ & $\times$ & $\times$ & 4.15\% & 3.75\% & 10.66\% & -0.48\% & -13.37\% & 10.65\% \\
    $\checkmark$ & $\times$ & $\checkmark$ & $\checkmark$ & 123.94\% & 59.05\% & 161.50\% & 111.28\% & 92.95\% & 86.93\% \\
    $\checkmark$ & $\checkmark$ & $\times$ & $\checkmark$ & 2.79\% & 2.24\% & \greybg{10.46\%} & -1.75\% & -12.77\% & 9.25\% \\
    $\checkmark$ & $\checkmark$ & $\checkmark$ & $\times$ & 3.90\% & 2.78\% & 15.43\% & -0.76\% & \greybg{-14.12\%} & 9.27\% \\
    $\checkmark$ & $\checkmark$ & $\checkmark$ & $\checkmark$ & \greybg{2.67\%} & \greybg{1.54\%} & 11.18\% & \greybg{-1.88\%} & -14.05\% & \greybg{7.97\%} \\
    \midrule
    \midrule
    Pivot & FPS   & depot & Bi.    & AOCVRPB & AOCVRPL & ACVRPBL & ACVRPBTW & ACVRPLTW & AOCVRPBL \\
    \midrule
    $\times$ & $\times$ & $\times$ & $\times$ & \greybg{-1.50\%} & 10.50\% & -13.61\% & 0.87\% & 4.53\% & \greybg{-2.09\%} \\
    $\checkmark$ & $\times$ & $\checkmark$ & $\checkmark$ & 132.02\% & 162.44\% & 82.90\% & 48.48\% & 60.47\% & 127.58\% \\
    $\checkmark$ & $\checkmark$ & $\times$ & $\checkmark$ & 1.20\% & \greybg{10.41\%} & -13.01\% & -0.36\% & 2.98\% & 0.40\% \\
    $\checkmark$ & $\checkmark$ & $\checkmark$ & $\times$ & 0.95\% & 15.30\% & \greybg{-14.42\%} & 0.04\% & 3.60\% & 0.45\% \\
    $\checkmark$ & $\checkmark$ & $\checkmark$ & $\checkmark$ & -1.34\% & 11.02\% & -14.27\% & \greybg{-0.52\%} & \greybg{2.26\%} & -1.94\% \\
    \midrule
    \midrule
    Pivot & FPS   & depot & Bi.    & AOCVRPBTW & AOCVRPLTW & ACVRPBLTW & AOCVRPBLTW & Avg.gap & Best Sol. \\
    \midrule
    $\times$ & $\times$ & $\times$ & $\times$ & 6.73\% & 9.71\% & 1.01\% & 6.63\% & 2.38\% & 2/16 \\
    $\checkmark$ & $\times$ & $\checkmark$ & $\checkmark$ & 73.04\% & 86.95\% & 49.25\% & 72.75\% & 95.72\% & 0/16 \\
    $\checkmark$ & $\checkmark$ & $\times$ & $\checkmark$ & 6.50\% & 8.26\% & -0.17\% & 6.52\% & 2.06\% & 2/16 \\
    $\checkmark$ & $\checkmark$ & $\checkmark$ & $\times$ & 5.69\% & 8.38\% & 0.16\% & 5.65\% & 2.64\% & 2/16 \\
    $\checkmark$ & $\checkmark$ & $\checkmark$ & $\checkmark$ & \greybg{5.48\%} & \greybg{6.99\%} & \greybg{-0.37\%} & \greybg{5.52\%} & \greybg{1.27\%} & \greybg{10/16} \\
    \bottomrule[0.5mm]
    \end{tabular}%
    }   
  \label{tab:ablation_bfr_details}%
\end{table*}%

\clearpage

\subsection{Effects of Pivot Set Size}

Table \ref{tab:ablation_pivot_details} studies the sensitivity of BFR to the number of pivots $M$. The default setting $M=8$ achieves the best average gap in both symmetric and asymmetric settings, respectively. Using fewer pivots ($M=5$) remains competitive on symmetric variants but provides less stable asymmetric transfer. Increasing the pivot count beyond $8$ does not improve performance: larger values, such as $M=20$ or $M=30$, can introduce redundant or noisy reference-distance coordinates and increase the input dimensionality. Overall, $M=8$ provides the best trade-off between geometric coverage and representation compactness.

\begin{table*}[htbp]
  \centering
  \caption{Detailed ablation results for the number of pivots in BFR.}
  \resizebox{\textwidth}{!}{
    \begin{tabular}{c|c|c|c|c|c|c}
    \toprule[0.5mm]
    Pivot & CVRP  & CVRPTW & OCVRP & CVRPL & CVRPB & OCVRPTW \\
    \midrule
    5     & 1.84\% & \greybg{5.29\%} & 3.32\% & 0.47\% & \greybg{1.42\%} & 4.25\% \\
    10    & 1.92\% & 5.50\% & \greybg{3.23\%} & 0.53\% & 1.62\% & 4.29\% \\
    20    & 2.09\% & 5.74\% & 3.49\% & 0.67\% & 1.80\% & 4.71\% \\
    30    & 2.16\% & 5.67\% & 3.60\% & 0.76\% & 1.89\% & 4.50\% \\
    8     & \greybg{1.83\%} & 5.36\% & 3.31\% & \greybg{0.45\%} & 1.45\% & \greybg{4.20\%} \\
    \midrule
    \midrule
    Pivot & OCVRPB & OCVRPL & CVRPBL & CVRPBTW & CVRPLTW & OCVRPBL \\
    \midrule
    5     & \greybg{7.31\%} & 3.29\% & \greybg{1.57\%} & 8.27\% & \greybg{1.74\%} & \greybg{7.35\%} \\
    10    & 7.50\% & \greybg{3.26\%} & 1.80\% & 8.96\% & 1.99\% & 7.70\% \\
    20    & 7.91\% & 3.43\% & 2.07\% & 9.46\% & 2.32\% & 7.91\% \\
    30    & 7.89\% & 3.59\% & 2.14\% & 9.08\% & 2.18\% & 8.06\% \\
    8     & 7.33\% & 3.29\% & 1.67\% & \greybg{7.91\%} & 1.82\% & 7.36\% \\
    \midrule
    \midrule
    Pivot & OCVRPBTW & OCVRPLTW & CVRPBLTW & OCVRPBLTW & Avg.gap & Best Sol. \\
    \midrule
    5     & 11.92\% & 4.35\% & 8.65\% & 11.92\% & 5.19\% & 6/16 \\
    10    & 12.94\% & 4.39\% & 9.35\% & 13.00\% & 5.50\% & 2/16 \\
    20    & 13.54\% & 4.79\% & 9.81\% & 13.51\% & 5.83\% & 0/16 \\
    30    & 12.83\% & 4.66\% & 9.39\% & 12.73\% & 5.70\% & 0/16 \\
    8     & \greybg{11.60\%} & \greybg{4.26\%} & \greybg{8.29\%} & \greybg{11.63\%} & \greybg{5.11\%} & \greybg{8/16} \\
    \toprule[0.5mm]
    Pivot & ACVRP & ACVRPTW & AOCVRP & ACVRPL & ACVRPB & AOCVRPTW \\
    \midrule
    5     & 2.99\% & 2.02\% & 12.15\% & -1.61\% & -13.05\% & 8.89\% \\
    10    & 2.73\% & 2.08\% & 13.56\% & -1.85\% & -14.02\% & 10.94\% \\
    20    & 2.97\% & 1.97\% & \greybg{10.81\%} & -1.53\% & \greybg{-15.21\%} & 8.86\% \\
    30    & 2.80\% & 1.86\% & 16.93\% & -1.73\% & -14.67\% & 9.11\% \\
    8     & \greybg{2.67\%} & \greybg{1.54\%} & 11.18\% & \greybg{-1.88\%} & -14.05\% & \greybg{7.97\%} \\
    \midrule
    \midrule
    Pivot & AOCVRPB & AOCVRPL & ACVRPBL & ACVRPBTW & ACVRPLTW & AOCVRPBL \\
    \midrule
    5     & 0.64\% & 11.89\% & -13.26\% & -0.31\% & 2.74\% & -0.06\% \\
    10    & 1.36\% & 13.47\% & -14.27\% & -0.46\% & 2.79\% & 0.74\% \\
    20    & -0.38\% & \greybg{10.66\%} & \greybg{-15.43\%} & -0.17\% & 2.76\% & -1.02\% \\
    30    & 1.91\% & 16.93\% & -15.00\% & -0.14\% & 2.64\% & 1.32\% \\
    8     & \greybg{-1.34\%} & 11.02\% & -14.27\% & \greybg{-0.52\%} & \greybg{2.26\%} & \greybg{-1.94\%} \\
    \midrule
    \midrule
    Pivot & AOCVRPBTW & AOCVRPLTW & ACVRPBLTW & AOCVRPBLTW & Avg.gap & Best Sol. \\
    \midrule
    5     & 5.88\% & 7.97\% & -0.16\% & 5.94\% & 2.04\% & 0/16 \\
    10    & 5.46\% & 10.04\% & -0.31\% & \greybg{5.37\%} & 2.35\% & 1/16 \\
    20    & 6.11\% & 8.08\% & -0.04\% & 6.13\% & 1.54\% & 4/16 \\
    30    & \greybg{5.34\%} & 8.25\% & -0.03\% & 5.38\% & 2.55\% & 1/16 \\
    8     & 5.48\% & \greybg{6.99\%} & \greybg{-0.37\%} & 5.52\% & \greybg{1.27\%} & \greybg{10/16} \\
    \bottomrule[0.5mm]
    \end{tabular}%
    }   
  \label{tab:ablation_pivot_details}%
\end{table*}%

\clearpage

\subsection{Effects of Weight-Decomposed Adaptive Decoding}

Table \ref{tab:ablation_wdad_details} reports the per-variant ablation results for WDAD. In the \textit{WDAD} column, $\times$ replaces WDAD with the original URS hypernetwork. The value $1$ denotes the single-head ablation under the same total adapter rank budget as the multi-head setting, and $3$ denotes the default $H=3$ WDAD configuration.

The comparison shows that BFR alone is not sufficient; the decoder also needs a parameterization that separates geometry-invariant routing knowledge from constraint-conditioned decision logic. When WDAD is replaced by the URS-style hypernetwork, the adaptive projections are generated from the problem representation as a whole, so routing attributes can be coupled again with metric-specific patterns observed during training. In contrast, WDAD forms each effective decoder weight as $W_{\mathrm{eff}}^{(\bm{\lambda})}=W_0+\Delta W^{(\bm{\lambda})}$, where the shared base weight $W_0$ carries reusable geometric decoding behavior and $\Delta W^{(\bm{\lambda})}$ is the average of active attribute-specific updates. This additive mean aggregation keeps the update scale stable when multiple constraints coexist and makes an unseen variant behave more like a composition of learned constraint primitives, especially when asymmetric distances alter pairwise costs and constraint-induced feasibility while the constraint primitives should remain reusable.

The single-head variant further indicates that the improvement is not solely due to the addition of a low-rank adapter. Although it uses the same total rank budget, a single normalized update direction per attribute tends to merge heterogeneous effects into one residual pathway. The multi-head decomposition instead lets each routing attribute express several complementary directions with independent row-wise magnitudes, so the decoder can adapt its query, key, value, and state projections without allowing any single constraint to dominate the effective weight scale. The consistent advantage of the default setting, therefore, supports the intended role of WDAD: separating geometry-invariant decoding from constraint-conditioned decision adjustments and enabling smoother transfer to unseen constraint combinations.

\begin{table*}[htbp]
  \centering
  \caption{Detailed ablation results for Weight-Decomposed Adaptive Decoding.}
  \resizebox{\textwidth}{!}{
    \begin{tabular}{c|c|c|c|c|c|c}
    \toprule[0.5mm]
    WDAD  & CVRP  & CVRPTW & OCVRP & CVRPL & CVRPB & OCVRPTW \\
    \midrule
    $\times$ & 2.04\% & 6.02\% & 3.70\% & 0.68\% & 2.04\% & 5.21\% \\
    1     & 1.98\% & 5.61\% & 3.61\% & 0.61\% & 1.76\% & 4.48\% \\
    3     & \greybg{1.83\%} & \greybg{5.36\%} & \greybg{3.31\%} & \greybg{0.45\%} & \greybg{1.45\%} & \greybg{4.20\%} \\
    \midrule
    \midrule
    WDAD  & OCVRPB & OCVRPL & CVRPBL & CVRPBTW & CVRPLTW & OCVRPBL \\
    \midrule
    $\times$ & 11.81\% & 3.63\% & 2.27\% & 10.67\% & 2.50\% & 11.92\% \\
    1     & 8.16\% & 3.59\% & 2.02\% & 8.55\% & 2.11\% & 8.33\% \\
    3     & \greybg{7.33\%} & \greybg{3.29\%} & \greybg{1.67\%} & \greybg{7.91\%} & \greybg{1.82\%} & \greybg{7.36\%} \\
    \midrule
    \midrule
    WDAD  & OCVRPBTW & OCVRPLTW & CVRPBLTW & OCVRPBLTW & Avg.gap & Best Sol. \\
    \midrule
    $\times$ & 17.15\% & 5.36\% & 10.85\% & 17.11\% & 7.06\% & 0/16 \\
    1     & 12.58\% & 4.53\% & 8.83\% & 12.60\% & 5.58\% & 0/16 \\
    3     & \greybg{11.60\%} & \greybg{4.26\%} & \greybg{8.29\%} & \greybg{11.63\%} & \greybg{5.11\%} & \greybg{16/16} \\
    \toprule[0.5mm]
    WDAD  & ACVRP & ACVRPTW & AOCVRP & ACVRPL & ACVRPB & AOCVRPTW \\
    \midrule
    $\times$ & 3.46\% & 2.72\% & 16.70\% & -1.09\% & -13.01\% & 12.79\% \\
    1     & 2.94\% & 1.81\% & 13.09\% & -1.67\% & \greybg{-14.65\%} & 9.80\% \\
    3     & \greybg{2.67\%} & \greybg{1.54\%} & \greybg{11.18\%} & \greybg{-1.88\%} & -14.05\% & \greybg{7.97\%} \\
    \midrule
    \midrule
    WDAD  & AOCVRPB & AOCVRPL & ACVRPBL & ACVRPBTW & ACVRPLTW & AOCVRPBL \\
    \midrule
    $\times$ & 2.51\% & 16.39\% & -13.27\% & 0.42\% & 3.50\% & 1.91\% \\
    1     & 1.01\% & 12.85\% & \greybg{-14.93\%} & -0.22\% & 2.56\% & 0.40\% \\
    3     & \greybg{-1.34\%} & \greybg{11.02\%} & -14.27\% & \greybg{-0.52\%} & \greybg{2.26\%} & \greybg{-1.94\%} \\
    \midrule
    \midrule
    WDAD  & AOCVRPBTW & AOCVRPLTW & ACVRPBLTW & AOCVRPBLTW & Avg.gap & Best Sol. \\
    \midrule
    $\times$ & 7.93\% & 11.72\% & 0.54\% & 7.91\% & 3.82\% & 0/16 \\
    1     & 6.94\% & 8.97\% & -0.08\% & 6.84\% & 2.23\% & 2/16 \\
    3     & \greybg{5.48\%} & \greybg{6.99\%} & \greybg{-0.37\%} & \greybg{5.52\%} & \greybg{1.27\%} & \greybg{14/16} \\
    \bottomrule[0.5mm]
    \end{tabular}%
    }   
  \label{tab:ablation_wdad_details}%
\end{table*}%

\clearpage

\section{Licenses For Used Resources}
\label{append:licenses}

\begin{table}[htbp]
\centering
\caption{List of licenses for the codes and datasets we used in this work.}
\label{table:Licenses}
\resizebox{0.99\textwidth}{!}{%
\begin{tabular}{l |l | l | l }
\toprule[0.5mm]
 Resource   &   Type  &  Link  & License    \\
\midrule
HGS-PyVRP \citep{wouda2024pyvrp}& Code & \url{https://github.com/PyVRP/PyVRP} & MIT License\\
LKH3 \citep{LKH3} & Code & \url{http://webhotel4.ruc.dk/~keld/research/LKH-3/} & Available for academic research use\\

OR-Tools \citep{ortools}& Code & \url{https://github.com/google/or-tools} & Apache-2.0 License\\
\midrule

MTPOMO \citep{liu2024mtpomo} & Code & \url{https://github.com/FeiLiu36/MTNCO} & MIT License  \\
MVMoE \citep{zhou2024mvmoe} & Code & \url{https://github.com/RoyalSkye/Routing-MVMoE} & MIT License  \\
RouteFinder \citep{berto2024routefinder} & Code & \url{https://github.com/ai4co/routefinder} & MIT License  \\
CaDA \citep{li2025cada} & Code & \url{https://github.com/CIAM-Group/CaDA} & MIT License  \\
ReLD \citep{huang2025reld} & Code & \url{https://github.com/ziweileonhuang/reld-nco} & MIT License  \\
GOAL \citep{drakulic2025goal} & Code & \url{https://github.com/naver/goal-co} & Available for any non-commercial use  \\
URS \citep{zhou2025urs} & Code & \url{https://github.com/CIAM-Group/URS} & MIT License  \\
\midrule

CVRPLIB Set-X \citep{uchoa2017cvrplib_setx} & Dataset & \url{http://vrp.galgos.inf.puc-rio.br/index.php/en/}  & Available for academic research use \\
CVRPLIB Set-XXL \citep{arnold2019cvrplib_xxl} & Dataset & \url{http://vrp.galgos.inf.puc-rio.br/index.php/en/}  & Available for academic research use \\

\bottomrule[0.5mm]
\end{tabular}%
}
\end{table}
We list the existing codes and datasets in \cref{table:Licenses}, all of which are open-source resources for academic use.

\section{Broader Impacts}
\label{append:broader_impacts}

This paper presents work aimed at advancing neural combinatorial optimization for vehicle routing problems. By unifying symmetric and asymmetric VRPs within a generalist neural routing solver, SPACE may reduce the need to design and retrain specialized models for each routing variant, thereby improving the adaptability of learning-based solvers in logistics, supply-chain management, and urban delivery scenarios. More efficient, transferable route-optimization methods could reduce operational costs, travel distance, and energy consumption when deployed with appropriate domain constraints. The proposed method has no potential negative societal impacts that we feel must be specifically highlighted.

\end{CJK*}

\end{document}